\documentclass{article}

\usepackage{microtype}
\usepackage{graphicx}
\usepackage{subcaption}
\usepackage{booktabs}
\usepackage{hyperref}
\usepackage{xcolor}
\usepackage{multirow}
\usepackage{array}
\usepackage{pifont}

\usepackage[accepted]{icml2026}
\usepackage{amsmath}
\usepackage{amssymb}
\usepackage{mathtools}
\usepackage{amsthm}
\usepackage{bm}

\usepackage{algorithm}
\usepackage{algorithmic}
\usepackage[capitalize,noabbrev]{cleveref}
\newcommand{\contribsep}{\unskip\nobreak\hspace{0.6em}%
  \ensuremath{\blacktriangleright}\hspace{0.6em}\ignorespaces}
\theoremstyle{plain}
\newtheorem{theorem}{Theorem}[section]
\newtheorem{proposition}[theorem]{Proposition}
\newtheorem{lemma}[theorem]{Lemma}
\newtheorem{corollary}[theorem]{Corollary}
\theoremstyle{definition}
\newtheorem{definition}[theorem]{Definition}

\theoremstyle{remark}

\ifodd 1

\newcommand{\com}[1]{\textbf{\color{red}(comment: #1)}} %comment of the text
\newcommand{\res}[1]{\textbf{\color{magenta}(RESPONSE: #1)}} %response to comments
\else

\newcommand{\com}[1]{}
\newcommand{\res}[1]{}
\fi
\newcommand{\R}{\mathbb{R}}
\newcommand{\E}{\mathbb{E}}

\newcommand{\calM}{\mathcal{M}}
\newcommand{\calV}{\mathcal{V}}
\newcommand{\calH}{\mathcal{H}}
\newcommand{\Id}{\mathbf{I}}
\newcommand{\tr}{\mathrm{tr}}
\newcommand{\vecop}{\mathrm{vec}}
\newcommand{\rank}{\mathrm{rank}}
\newcommand{\projA}{\Pi_A}
\newcommand{\projB}{\Pi_B}
\newcommand{\projperpA}{\Pi_A^{\perp}}
\newcommand{\projperpB}{\Pi_B^{\perp}}
\newcommand{\pinv}{^{\dagger}} % Moore--Penrose pseudoinverse
\newcommand{\gl}{\mathrm{GL}}
\newcommand{\Retr}{\mathrm{Retr}}
\newcommand{\cmark}{\ding{51}}
\newcommand{\xmark}{\ding{55}}

\usepackage{enumitem}
\ifodd 0

\newcommand{\com}[1]{\textbf{\color{red}(comment: #1)}} %comment of the text
\newcommand{\res}[1]{\textbf{\color{magenta}(RESPONSE: #1)}} %response to comments
\else
\fi

\icmltitlerunning{PRISM: Gauge-Invariant Tangent-Space Differentially Private LoRA}

\begin{document}

\twocolumn[
\icmltitle{PRISM: Gauge-Invariant Tangent-Space Differentially Private LoRA}

\begin{icmlauthorlist}
\icmlauthor{Shihao Wang}{osu}
\icmlauthor{Xueru Zhang}{osu}
\end{icmlauthorlist}

\icmlaffiliation{osu}{Department of Computer Science and Engineering,
The Ohio State University, Columbus, OH, USA}

\icmlcorrespondingauthor{Shihao Wang}{wang.17571@osu.edu}
\icmlcorrespondingauthor{Xueru Zhang}{zhang.12807@osu.edu}

\icmlkeywords{differential privacy, LoRA, parameter-efficient fine-tuning, invariance, matrix manifolds}

\vskip 0.3in
]

\printAffiliationsAndNotice{}  

%%%%%%%%%%%%%%%%%%%%%%%%%%%%%%%%%%%%%%%%%%%%%%%%%%%%%%%
% ABSTRACT
%%%%%%%%%%%%%%%%%%%%%%%%%%%%%%%%%%%%%%%%%%%%%%%%%%%%%%%

\begin{abstract}
Applying differential privacy (DP) via DP-SGD to Low-Rank Adaptation (LoRA) is a natural approach for privacy-preserving fine-tuning. However, LoRA's low-rank parameterization poses a fundamental challenge. In LoRA, each trainable update is represented as a low-rank matrix $Z=AB^\top$, but this factorization is inherently \emph{non-identifiable}: many factor pairs $(A,B)$ represent the same update $Z$. As a result, applying DP-SGD directly to the factors induces \emph{gauge-dependent} perturbations on $Z$, and we show that this naive DP-LoRA can lead to unbounded noise amplification. We propose \textbf{PRISM}, an intrinsic DP mechanism for LoRA that is gauge invariant by construction, avoids bilinear noise amplification, and admits an efficient low-dimensional noise sampler. Moreover, PRISM yields a closed-form characterization of the effective intrinsic noise induced on $Z$, enabling stable privacy--utility trade-offs through bounded, gauge-invariant perturbations. We establish standard  $(\varepsilon,\delta)$-DP guarantees for PRISM and introduce a DP-aware, gauge-invariant adaptive update rule that prevents adaptive optimization from amplifying injected privacy noise, improving numerical stability in practice.

\end{abstract}

%%%%%%%%%%%%%%%%%%%%%%%%%%%%%%%%%%%%%%%%%%%%%%%%%%%%%%%
% 1 INTRODUCTION
%%%%%%%%%%%%%%%%%%%%%%%%%%%%%%%%%%%%%%%%%%%%%%%%%%%%%%%

\section{Introduction}\label{sec:intro}

Foundation models are routinely adapted to domain-specific tasks using \emph{private} corpora. % (e.g., clinical notes, customer support logs). 
Because full-model fine-tuning is costly, practitioners often adopt parameter-efficient fine-tuning (PEFT) methods, which update only a small subset of parameters while keeping the backbone frozen. Common methods include adapters \cite{houlsby2019adapter,hu2023llm_adapters}, prefix/prompt tuning \cite{li2021prefix,lester2021prompt}, bias-only updates \cite{zaken2021bitfit}, and lightweight reparameterizations \cite{liu2022ia3}. Among these, Low-Rank Adaptation (LoRA) \cite{hu2022lora} has emerged as a dominant method due to its drop-in compatibility with existing linear layers, strong performance under low-precision training with a quantized backbone~\cite{dettmers2023qlora,frantar2023loftq}, and the non-identifiability induced by its low-rank parameterization. Formally, LoRA adapts a frozen pretrained weight matrix $W_0\in\R^{m\times n}$ by adding a low-rank update $Z$:
\begin{align}\label{eq:lora}
W = W_0 + Z, \qquad Z = AB^\top
\end{align}
where $A \in\R^{m\times r}$, $ B\in\R^{n\times r}$ with rank $r\ll \min\{m,n\}.$ Rather than updating $W_0$ directly, optimization proceeds over $(A,B)$, which implicitly define the intrinsic update $Z$ applied to the backbone. This low-rank factorization substantially reduces the number of trainable parameters.

Despite its efficiency, PEFT on sensitive data raises significant privacy concerns. Prior work has shown that trained models can leak information about training data through membership, inversion, or extraction attacks~\cite{fredrikson2015modelinversion,shokri2017membership,ganju2018property,carlini2019secret,carlini2021extracting}.
These risks are often exacerbated in PEFT settings, where fine-tuning datasets are small, domain-specific, and contain rare or uniquely identifying records.

%inference~\cite{shokri2017membership}, inversion or property inference~\cite{fredrikson2015modelinversion,ganju2018property}, and verbatim extraction of memorized content~\cite{carlini2019secret,carlini2021extracting}. These risks are particularly pronounced in PEFT settings, where fine-tuning datasets are often small, domain-specific, and contain rare or uniquely identifying records.%Prior work has shown that models can leak information about their training data: membership inference attacks test whether specific records were used during training \cite{shokri2017membership}, model inversion and property inference attacks reveal sensitive attributes of individuals or populations \cite{fredrikson2015modelinversion,ganju2018property}, and extraction attacks can recover memorized text verbatim \cite{carlini2019secret,carlini2021extracting}. These risks are often exacerbated in PEFT scenarios, where fine-tuning datasets tend to be small, domain-specific, and may contain rare or uniquely identifying samples.  

Various approaches have been proposed to mitigate privacy risks in large models~\cite{bourtoule2021machineunlearning}. We focus on differential privacy (DP)~\cite{dwork2006calibrating,dwork2014algorithmic}, which provides an attack-agnostic, \emph{ex ante} guarantee by bounding the influence of any individual record and remaining robust to arbitrary post-processing. In practice, DP is most commonly instantiated via DP-SGD~\cite{abadi2016deep}, which clips per-example gradients and injects calibrated Gaussian noise before each update. %with privacy loss tracked using modern accountants such as R\'enyi DP/moments \cite{mironov2017renyi,wang2019subsampled} or privacy loss random variables (PRV) \cite{gopi2021numerical}. This mechanism provides strong formal guarantees and underpins most practical private training pipelines \cite{papernot2017pate,kairouz2021advances}. 

A natural approach to obtaining DP in LoRA fine-tuning is to apply DP-SGD directly to the low-rank factors 
$(A,B)$ \cite{yu2022dpftlm,liu2023dplora,xu2025dpfedlora}. Some variants further freeze one of the two factors to improve numerical stability~\cite{sun2024improvinglora}. While straightforward, this strategy is fundamentally misaligned with the structure of LoRA and leads to ill-defined private updates.

The core issue is that DP-SGD is defined relative to a \emph{parameterization}, whereas in LoRA the factors $(A,B)$ are only an auxiliary representation of the intrinsic update applied to the frozen backbone. It is the effective update $Z=AB^\top$, rather than the factors themselves, that ultimately determines model behavior. As a consequence, naively applying DP-SGD in factor space (i) induces \emph{gauge-dependent} perturbations on the intrinsic update $Z$, meaning that different factor pairs corresponding to the same intrinsic update can induce significantly different clipping and noise effects; (ii) introduces spurious higher-order noise terms caused by independently noising the two factors, resulting in quadratic noise effects that do not arise in standard DP-SGD on linear parameters; and (iii) interacts poorly with optimization dynamics, where adaptive preconditioning can amplify stochasticity and lead to numerical instability. We discuss these issues in detail in Section~\ref{sec:problem}.

These issues motivate a key design principle: the randomized DP mechanism should operate on the space of \emph{intrinsic} model updates, rather than on a gauge-redundant factorization. Based on this, we propose the \textbf{P}rojected \textbf{R}iemannian \textbf{I}nvariant \textbf{S}ubspace \textbf{M}echanism (\textbf{PRISM}), which performs DP-SGD directly on the rank-$r$ manifold of LoRA updates. Specifically, PRISM projects per-example gradients to the tangent space $\Delta Z\in T_Z\mathcal{M}_r$, applies global Frobenius-norm clipping and isotropic tangent Gaussian noise in this intrinsic space, and retracts back to rank $r$. By aligning the DP mechanism with the intrinsic update $Z$, PRISM ensures that the \textit{effective intrinsic noise} on $Z$ is deterministic and independent of the particular factorization $(A,B)$. Moreover, operating in intrinsic coordinates keeps updates additive and avoids the spurious bilinear second-order noise terms induced by independently noising the factors. Importantly, PRISM achieves these guarantees with LoRA-scale computational cost. Table~\ref{tab:comparison} compares PRISM with existing DP-LoRA design choices.

\begin{table}[t]
  \caption{\textbf{Comparison with other DP-LoRA design choices.} %\emph{One-side} updates $B$ only (freezing $A$). $\mathcal{E}_Z$ is the effective intrinsic noise. The comparison evaluates three desiderata: 
  We evaluate each variant against three desiderata: (a) gauge-invariant randomized mechanism, (b) additive perturbations on 
$Z$ (i.e., no bilinear DP term), and (c) LoRA-scale efficiency. The naive factor-space variant violates (a) and (b), yielding unbounded effective intrinsic noise $\mathcal{E}_Z$. The \emph{one-sided} variant, which updates $B$ only  while freezing $A$, satisfies (b) and (c) but not (a). PRISM satisfies all three and admits a closed-form expression for $\mathcal{E}_Z$.
}
  \label{tab:comparison}
  \begin{center}
    \begin{small}
      \begin{sc}
        \setlength{\tabcolsep}{2.9pt}
        \renewcommand{\arraystretch}{1.12}
        \begin{tabular}{@{}lccccc@{}}
          \toprule
          Method & Params & $\mathcal{E}_Z$ & {\normalfont(a)} & {\normalfont(b)} & {\normalfont(c)} \\
          \midrule
          DP-LoRA  & $(m{+}n)r$ & unbounded & \xmark & \xmark & \cmark \\
          One-side  & $nr$ & $(\sigma C/b)\sqrt{n}\,\|A\|_F$ & \xmark & \cmark & \cmark \\
          \textbf{PRISM } & $(m{+}n)r$ & $(\sigma C/b)\sqrt{r(m{+}n{-}r)}$ & \cmark & \cmark & \cmark \\
          \bottomrule
        \end{tabular}
      \end{sc}
    \end{small}
  \end{center}
  \vskip -0.1in
\end{table}

%\paragraph{Contributions.}
Our contributions can be summarized below.
\begin{itemize}[leftmargin=*,itemsep=-0.05cm,topsep=-0.05cm]
    \item We identify issues with factor-space DP-LoRA and show that they can lead to (i) gauge-dependent clipping and noise injection, (ii) unbounded amplification of intrinsic noise, and (iii) spurious bilinear second-order noise terms arising from independently noising the two factors.
    
   % formalize \textit{effective intrinsic noise} $\mathcal{E}_Z$ (Definition~\ref{def:eff_noise}) and show that factor-space DP-LoRA can induce (i) gauge-dependent clipping/noising, (ii) unbounded intrinsic noise amplification under benign rescalings, and (iii) a bilinear second-order noise term from independently noising both factors.
    \item We propose PRISM, a gauge-invariant DP mechanism that projects per-example gradients on rank-$r$ tangent space, applies \emph{global} intrinsic clipping across all LoRA modules, injects isotropic tangent noise using an $O((m+n)r^2)$ sampler, and retracts updates back to rank $r$.
    % \item We develop a DP-aware gauge-invariant adaptive update rule that floors and condition-clamps rank-space preconditioners and clips post-preconditioned updates, mitigating optimizer-induced noise amplification while preserving LoRA-scale efficiency.
    \item We develop a DP-aware gauge-invariant adaptive update rule that floors rank-space preconditioners based on the DP noise level, mitigating optimizer-induced noise amplification while preserving LoRA-scale efficiency.
    \item We provide theoretical guarantees showing that PRISM produces gauge-invariant updates and noise distributions, satisfies $(\varepsilon,\delta)$-DP under subsampled Gaussian accounting, and injects tangent noise with gauge-invariant covariance and energy proportional to $r(m+n-r)$. The project
code is available at
\href{https://github.com/osu-srml/PRISM-DP-LoRA}
{\texttt{github.com/osu-srml/PRISM-DP-LoRA}}.

\end{itemize}
 %Appendix~\ref{app:all}.

% \paragraph{Conflict of Interest Disclosure.}
% The authors declare no financial conflicts of interest related to this work.
\iffalse
\contribsep\textbf{DP gauge issue and effective intrinsic noise.}
We formalize effective intrinsic noise $\mathcal{E}_Z$ (Definition~\ref{def:eff_noise}) and show it is gauge dependent under naive factor-space DP, with unbounded amplification under benign rescalings (Corollary~\ref{cor:unbounded}).
We also isolate the bilinear second-order noise term induced by noising both factors.
\contribsep\textbf{Method (PRISM).}
We design a gauge-invariant DP mechanism that clips/noises projected tangent updates and retracts to rank $r$, with LoRA-scale cost.
\contribsep\textbf{Theory.}
We prove: (i) PRISM's update and noise distribution are gauge invariant; (ii) naive factor-space DP can amplify intrinsic noise without bound; (iii) the projected Gaussian has gauge-invariant covariance with energy proportional to $r(m+n-r)$; and (iv) PRISM attains standard $(\varepsilon,\delta)$-DP via subsampled Gaussian accounting plus post-processing.
\fi 

\section{Problem Formulation}
\label{sec:problem}
%\label{sec:theory}

We study differentially private (DP) parameter-efficient fine-tuning using LoRA.
Given a frozen weight matrix $W_0\in\R^{m\times n}$, LoRA learns a rank-$r$ update $Z$ such that
$W=W_0+Z$ with $Z=AB^\top$ (Eq.~\eqref{eq:lora}), by minimizing the empirical risk over a private dataset
$\displaystyle \mathcal{D}=\{x_i\}_{i=1}^N$:
$$\textstyle \min F(A,B)\triangleq \frac{1}{N}\sum_{i=1}^N \ell_i(W_0+A B^\top).$$ 
Our goal is to design a randomized training procedure whose released adapter $Z$ satisfies
$(\varepsilon,\delta)$-DP \cite{dwork2014algorithmic} with respect to $\mathcal{D}$, while preserving the utility of LoRA.

Formally, an algorithm $\mathcal{M}$ satisfies
$(\varepsilon,\delta)$-DP if, for any adjacent datasets $\mathcal{D},\mathcal{D}'$ and any measurable event $\mathcal{S}$,
\begin{equation}
\Pr[\mathcal{M}(\mathcal{D})\in\mathcal{S}]
\;\le\;
e^\varepsilon \Pr[\mathcal{M}(\mathcal{D}')\in\mathcal{S}] + \delta,
\label{eq:dp}
\end{equation}
A standard approach for achieving DP in model training is DP-SGD \cite{abadi2016deep}. At each iteration, DP-SGD computes per-example gradients \(g_i=\nabla \ell_i \), clips them to a prescribed norm \(C\) and aggregates them with added Gaussian noise \(\xi \sim \mathcal{N}(0,I)\):
\begin{equation}
\tilde g_i = \frac{g_i}{\max\{1, \lVert g_i\rVert_2/C\}},\qquad
\widehat g = \frac{1}{b}\sum_{i=1}^b \tilde g_i + \frac{\sigma C}{b}\,\xi
\label{eq:dpsgd}
\end{equation}
where $b$ denotes the batch size and $\sigma$ is the noise multiplier. The noisy $\widehat g $ is then used to perform an optimizer update.

A dominant approach in DP-LoRA applies DP-SGD directly to the factor parameters $(A,B)$
\cite{yu2022dpftlm,liu2023dplora,xu2025dpfedlora}. Specifically, let \(g_{A,i}, g_{B,i} \) denote the per-example gradients with respect to $A$ and $B$, respectively. This approach applies \eqref{eq:dpsgd} to the concatenated gradient
$g_i=(g_{A,i},g_{B,i})$ with $\|g_i\|_2^2=\|g_{A,i}\|_F^2+\|g_{B,i}\|_F^2$. Some works further consider variants with one-sided training that updates only one LoRA factor while freezing the other \cite{sun2024improvinglora}. However, enforcing DP on the factors $(A,B)$ is fundamentally misaligned with the effective update $Z=AB^\top$ that governs model behavior, as we detail below.
%We isolate three pathologies showing why \emph{factor-space} DP is misaligned with the intrinsic parameter $Z$ and fragile under adaptive optimization.

\textbf{Issue I: Factor-space DP violates LoRA gauge symmetry.}
LoRA factorization is \textit{non-identifiable} \cite{hu2022lora}: for any invertible $R\in\mathrm{GL}(r)$, the factor pairs
$(A,B)$ and $(AR,BR^{-\top})$ induce the same intrinsic update $Z$. Under such a gauge transformation $(A,B)\mapsto(AR,BR^{-\top})$, the corresponding per-example gradients transform as
$$g'_{A,i}=g_{A,i}R^{-\top},~~~~~~g'_{B,i}=g_{B,i}R.$$
As a result, the clipping norm, and hence the clipping coefficient, used in DP-SGD depend on the particular factorization chosen to represent $Z$. For example, under the simple rescaling gauge  $(A,B)\mapsto(cA,c^{-1}B)$, 
\begin{equation}
\|g'_{A,i}\|_F^2+\|g'_{B,i}\|_F^2 \;=\; c^{-2}\|g_{A,i}\|_F^2 + c^{2}\|g_{B,i}\|_F^2.
\label{eq:gauge_norm}
\end{equation}
which can vary arbitrarily with $c$. Consequently, the distribution of the clipped-and-noised update produced in factor space is gauge dependent, and the induced intrinsic update is not determined by $Z$ alone (Appendix~\ref{app:general_gauge_amp}\footnote{Throughout the paper, we state the main claims in the main text and defer detailed analysis and proofs to the appendix, with explicit appendix references provided after each claim.}). 

This gauge dependence propagates from gradient clipping to the resulting update increments. Even when a per-example increment $(\Delta A_i,\Delta B_i)$ represents a fixed intrinsic direction $\Delta Z_i=\Delta A_i B^\top + A\Delta B_i^\top$, the same intrinsic direction admits gauge-related representatives $(\Delta A_i',\Delta B_i')=(\Delta A_i R,\Delta B_i R^{-\top})$ for any $R\in\gl(r)$. In general,
%A related issue arises at the level of factor-space update increments. Even when a per-example increment $(\Delta A_i,\Delta B_i)$ represents a fixed intrinsic direction $\Delta Z_i=\Delta A_i B^\top + A\Delta B_i^\top$, the same intrinsic direction admits gauge-related representatives $(\Delta A_i',\Delta B_i')=(\Delta A_i R,\Delta B_i R^{-\top})$ for any $R\in\gl(r)$. In general,
\begin{equation}
\label{eq:gauge_dep_clip}
\|\Delta A'_{i}\|_F^2+\|\Delta B'_{i}\|_F^2\neq \|\Delta A_{i}\|_F^2+\|\Delta B_{i}\|_F^2,
\end{equation}
so any mechanism that clips and perturbs based on the Euclidean norms in factor space is inherently gauge dependent.

Formally, let $\Delta Z_{\text{fac}}(A,B)$ denote the intrinsic update induced by a single factor-space DP-SGD step. A gauge-respecting mechanism would require 
\begin{equation}\label{eq:gauge_invariance_req}
\Delta Z_{\text{fac}}(A,B)\overset{d}{=}
  \Delta Z_{\text{fac}}(AR,BR^{-\top}), ~~\forall R\in\operatorname{GL}(r),
\end{equation}
a condition already violated by Eq.~\eqref{eq:gauge_norm}. Importantly, this issue is distinct from, and not resolved by, deterministic transformation-invariant optimizers for LoRA \cite{yen2025rite}, as DP requires invariance of the randomized clipping and noising procedure itself.

\textbf{Issue II: Noising both factors injects bilinear and gauge-amplified intrinsic noise.}
When DP noise is injected into both factors, the intrinsic update inevitably contains a second-order noise term. Consider a single update step $(A,B)\leftarrow(A,B)+(\Delta A,\Delta B)$, and let $\xi_A,\xi_B$ denote the
Gaussian perturbations added by DP-SGD, scaled by the step size $\eta$. The induced intrinsic update then satisfies
\begin{align}\label{eq:bilinear_noise_prob}
Z^+&=(A+\Delta A+\eta\xi_A)\,(B+\Delta B+\eta\xi_B)^{\top} \nonumber  \\
&  = Z + \Delta Z
     + \eta\bigl(\xi_A B^{\top} + A\xi_B^{\top}\bigr)
     + \eta^{2}\,\xi_A\xi_B^{\top}.
\end{align}
The final term $\eta^2\xi_A\xi_B^\top$ arises from multiplying independently noised factors and cannot be produced by any additive noise mechanism applied directly to $Z$ (Appendix~\ref{app:second_order_stats}).

Even if this second-order term is ignored, the first-order intrinsic noise remains problematic, as its magnitude depends on the norms of the factors (see Proposition~\ref{prop:amp}, Eq.~\eqref{eq:naive_energy}).
Under rescaling gauge $(A,B)\mapsto(cA,c^{-1}B)$, 
Eq.~\eqref{eq:naive_energy} becomes $\tau^2\big(m c^{-2}\|B\|_F^2+n c^{2}\|A\|_F^2\big)$, which can grow without bound as
$c\to 0$ or $c\to\infty$. One-sided variants that freeze one LoRA factor \cite{sun2024improvinglora} can suppress parts of
Eq.~\eqref{eq:bilinear_noise_prob}, but do not eliminate the dependence of intrinsic noise energy on a
representation-dependent scale, since the frozen factor still determines $\|A\|_F$ or $\|B\|_F$.
This motivates enforcing DP directly in the intrinsic space of $Z$ (or its tangent space),
rather than in factor coordinates.

To formalize the scale of randomized perturbations on $Z$, we introduce the following notion.

\begin{definition}[Effective intrinsic noise]
\label{def:eff_noise}
For a random intrinsic perturbation $\mathcal{N}_{Z}$, define $\mathcal{E}_{Z}\triangleq\sqrt{\E\|\mathcal{N}_{Z}\|_F^2}$.

\end{definition}

\begin{proposition}[Intrinsic noise energy under factor noising]
\label{prop:amp}
Let $\xi_{A}\in\R^{m\times r}$ and $\xi_{B}\in\R^{n\times r}$ be independent with i.i.d.\ entries $\mathcal{N}(0,\tau^2)$, and define
$\mathcal{N}_{Z} = \xi_{A} B^\top + A\xi_{B}^\top + \xi_{A}\xi_{B}^\top$.
Then the first-order term $\mathcal{N}_{Z}^{(1)}\triangleq \xi_{A} B^\top + A\xi_{B}^\top$ satisfies
\begin{equation}
\label{eq:naive_energy}
\E\|\mathcal{N}_{Z}^{(1)}\|_F^2
= \tau^2\Bigl(m\|B\|_F^2+n\|A\|_F^2\Bigr),
\end{equation}
while the bilinear term $\mathcal{N}_{Z}^{(2)}\triangleq \xi_{A}\xi_{B}^\top$ satisfies
\begin{equation}
\label{eq:naive_second}
\E\|\mathcal{N}_{Z}^{(2)}\|_F^2
= mnr\,\tau^4.
\end{equation}
\end{proposition}
Proposition~\ref{prop:amp} reveals that factor-space noising induces intrinsic noise whose first-order component scales with $\|A\|_F,\|B\|_F$ and whose second-order bilinear component scales as
$mnr\,\tau^4$, leading to gauge-dependent amplification and an unavoidable $\eta^2$ noise term.
%suggests that Noising $A$ and $B$ induces intrinsic noise on $Z$ with a first-order term scaling with $\|A\|_F,\|B\|_F$ and a second-order bilinear term scaling with $mnr\,\tau^4$, which drives gauge-dependent amplification and the extra $\eta^2$ factor-space term.

\begin{corollary}[Unbounded gauge amplification]
\label{cor:unbounded}
Fix $Z=AB^\top\neq 0$ and $\tau^2>0$.
Along the scalar gauge family $(A_{c},B_{c})=(cA,c^{-1}B)$, the first-order effective intrinsic noise $\mathcal{E}_{Z}$ defined in Eq.~\eqref{eq:naive_energy} 
is unbounded over $c>0$ even though %satisfies $\sup_{c>0}\mathcal{E}_{Z,c}=\infty$ even though 
$A_{c}B_{c}^\top=Z$ is constant. 
\end{corollary}
See Appendix~\ref{app:naive_noise}-\ref{app:eff_noise_range} for derivations. Corollary~\ref{cor:unbounded} shows that factor-space DP induces gauge-dependent and potentially unbounded intrinsic noise, even when the effective update $Z$ is fixed. %implies that, even when the effective update $Z$ is fixed, rescalings of the LoRA factors can make the intrinsic DP noise arbitrarily large. So factor-space DP does not define a stable noise level on $Z$ unless one also fixes a gauge.

\textbf{Issue III: Adaptive preconditioning magnifies DP noise and stresses low-rank numerics.}
Private PEFT methods typically rely on adaptive optimizers such as Adam and AdamW to maintain utility under noisy gradients
\cite{kingma2015adam,loshchilov2019adamw}, and recent LoRA-specific
invariant optimizers extend the same principle \cite{yen2025rite}.
Conceptually, these methods apply a (possibly low-dimensional) preconditioner to the DP gradient: 
\begin{equation}
\theta^+ = \theta - \eta\,\mathsf{P}^{-1/2}\widehat g,
\qquad
\widehat g = g + \xi
\label{eq:precond_dp}
\end{equation}
where $\xi \sim \mathcal{N}(0,\Sigma_\xi)$ denotes the injected DP noise. The resulting update noise is therefore $\eta\,\mathsf{P}^{-1/2}\xi$ with covariance
$\eta^{2}\,\mathsf{P}^{-1/2}\Sigma_\xi \mathsf{P}^{-1/2}$. Because the preconditioner $\mathsf{P}$ is estimated from noisy gradients, adaptive optimizers inevitably adapt to the injected DP noise. When DP noise dominates the true gradient signal, the preconditioner is largely determined by noise statistics, yielding $\mathsf{P}\approx\mathbb{E}[\xi\xi^\top]=\Sigma_\xi$. In this regime, the update noise covariance becomes $\eta^{2}\,\mathsf{P}^{-1/2}\Sigma_\xi \mathsf{P}^{-1/2}\approx \eta^2 I$,
i.e., adaptive preconditioning can \emph{normalize/reshape} the injected DP noise so that the effective update noise is no longer scaled in a simple way by the base DP-SGD noise level. This noisy-moment-driven behavior is known to undermine the benefits of black-box combinations of DP with adaptive optimizers, and has motivated several DP-aware adaptive variants (e.g., leveraging side information, delayed/stale preconditioners, or bias-corrected moment estimation) \cite{pmlr-v162-li22x,li2022differentially,tang2023dpadambc}.

These issues are further exacerbated in LoRA setting. Adaptive and invariant LoRA optimizers often involve operations on small $r\times r$ Gram matrices such as
$M=A^\top A$ and $N=B^\top B$ (e.g., via inverses, pseudoinverses, or inverse square roots). DP noise and gauge drift can drive these matrices toward ill-conditioning, where $\|\;M^{\dagger/2}\;\|_2 = 1/\sqrt{\lambda_{\min}^{+}(M)}$ becomes large. This both amplifies noise in the update and destabilizes numerical routines such as eigendecompositions. Consequently, a practical DP-LoRA mechanism must control not only intrinsic sensitivity and noise injection, but also the interaction between privacy noise and adaptive preconditioning under low-rank numerical constraints.

\noindent\textbf{Design target.} We aim to develop a DP-LoRA procedure whose randomized update of the intrinsic parameter is
(i) gauge invariant in distribution, (ii) additive in an intrinsic (tangent) representation, thereby avoiding bilinear noise,
and (iii) stable under adaptive optimization and preconditioning, without magnifying DP noise or destabilizing low-rank numerical operations.

\section{Proposed Method: PRISM}
\label{sec:method}
\label{sec:theory_opt}
To address the issues in Section~\ref{sec:problem}, we propose the Projected Riemannian
Invariant Subspace Mechanism (PRISM), a DP-LoRA procedure that applies DP-SGD in the \emph{intrinsic} geometry of low-rank updates. For simplicity, we previously described LoRA for a single weight matrix $W=W_0+AB^\top$; in practice, LoRA is applied to multiple layers of a model. We therefore consider a LoRA model with $L$ \emph{LoRA modules} (i.e., LoRA-augmented layers), indexed by $\ell$ \cite{hu2022lora}. Each module $\ell$ has factor parameters $(A_\ell,B_\ell)$ and an intrinsic update $Z_\ell=A_\ell B_\ell^\top$.
For each training example $i$, we denote the intrinsic gradient $G_{i,\ell}\triangleq\nabla_{Z_\ell}\ell_i$. By the chain rule, the corresponding factor gradients satisfy $g_{A,i,\ell}=G_{i,\ell}B_\ell$ and $g_{B,i,\ell}=G_{i,\ell}^\top A_\ell$.

\begin{algorithm}[t]
  \caption{One PRISM update across all LoRA modules}
  \label{alg:prism}
  \begin{algorithmic}[1]
    \STATE \textbf{Input:} LoRA factors $\{(A_\ell,B_\ell)\}_{\ell=1}^L$, minibatch $\{x_i\}_{i=1}^b$, clip $C$, noise multiplier $\sigma$, learning rate $\eta$.
    \STATE \textbf{Per-example intrinsic norms:}
    \FOR{$i=1$ to $b$}
      \STATE For each module $\ell$, compute lifted tangent gradient $(\Delta A_{i,\ell},\Delta B_{i,\ell})$ via Eq.~\eqref{eq:canonical}.
      \STATE Compute $\|\Delta Z_{i,\ell}\|_F^2$ via Eq.~\eqref{eq:tangent_norm}, and $s_i$ via Eq.~\eqref{eq:global_clip}. %set $s_i=\big(\sum_\ell \|\Delta Z_{i,\ell}\|_F^2\big)^{1/2}$.
      \STATE Set clipping coefficient $\alpha_i=\min\{1,C/s_i\}$.
    \ENDFOR
    \STATE \textbf{Module-wise DP tangent update:}
    \FOR{each module $\ell$}
      \STATE $\bar\Delta A_\ell=\frac{1}{b}\sum_i \alpha_i\Delta A_{i,\ell}$, $\bar\Delta B_\ell=\frac{1}{b}\sum_i \alpha_i\Delta B_{i,\ell}$.
      \STATE Sample tangent noise $(\Xi_{A,\ell},\Xi_{B,\ell})$ via Eq.~\eqref{eq:noise_form}.
      \STATE $\Delta A_\ell^{\mathrm{dp}}=\bar\Delta A_\ell + \frac{\sigma C}{b}\Xi_{A,\ell}$, $\Delta B_\ell^{\mathrm{dp}}=\bar\Delta B_\ell + \frac{\sigma C}{b}\Xi_{B,\ell}$.
      \STATE Compute DP-aware invariant adaptive direction $(U_{A,\ell},U_{B,\ell})$ using Eqs.~\eqref{eq:adam_moments}--\eqref{eq:noise_floor}.
      \STATE Retract via Eq.~\eqref{eq:retraction}:
      
       $(A_\ell,B_\ell)\leftarrow \mathrm{Retr}_r\!\big(A_\ell,B_\ell; -\eta U_{A,\ell},-\eta U_{B,\ell}\big).$
    \ENDFOR
    \STATE \textbf{Output:} updated LoRA factors.
  \end{algorithmic}
\end{algorithm}

\textbf{Overview.}
Algorithm~\ref{alg:prism} summarizes one PRISM update applied across all LoRA modules. Rather than directly perturbing the non-identifiable factors, it performs DP-SGD on the intrinsic parameters $\{Z_\ell\}$ by operating on tangent directions of the fixed-rank manifold $\mathcal{M}_r$. For each sample $i$ and module $\ell$, PRISM computes a \emph{lifted tangent gradient} $(\Delta A_{i,\ell},\Delta B_{i,\ell})$ (line~4), which is a factor-space representation of an intrinsic tangent matrix
$\Delta Z_{i,\ell}\in T_{Z_\ell}\mathcal{M}_r$ satisfying $\Delta Z_{i,\ell}=\Delta A_{i,\ell}B_\ell^\top + A_\ell\Delta B_{i,\ell}^\top$. Using these intrinsic tangent matrices, PRISM computes a per-example intrinsic norm aggregated across all modules and applies a single \emph{global} clipping coefficient $\alpha_i$ (lines~3-7), yielding a unified sensitivity bound for the entire LoRA update.

% The algorithm then proceeds module-wise. For each module $\ell$, the clipped lifted tangents are averaged across the minibatch (line~10), and isotropic Gaussian noise is sampled directly in the tangent space and added to form the DP tangent update $(\Delta A_\ell^{\mathrm{dp}},\Delta B_\ell^{\mathrm{dp}})$ (lines~11-12). A DP-aware, gauge-invariant adaptive transformation is then applied to obtain the final update direction $(U_{A,\ell},U_{B,\ell})$ (line~13). Finally, PRISM retracts the updated tangent direction back onto the rank-$r$ manifold, producing updated LoRA factors $(A_\ell,B_\ell)$ (line~14). The remainder of this section details these components and explains how they address Issues~I–III.
PRISM then proceeds module-wise: it averages the clipped lifted tangents over the minibatch (line~10) and adds isotropic tangent Gaussian noise to form the DP tangent update $(\Delta A_\ell^{\mathrm{dp}},\Delta B_\ell^{\mathrm{dp}})$ (lines~11--12). It applies a DP-aware, gauge-invariant adaptive transform to obtain $(U_{A,\ell},U_{B,\ell})$ (line~13), and retracts back to rank $r$ to update $(A_\ell,B_\ell)$ (line~14). The remainder of this section details these components and explains how they address Issues~I--III.
\subsection{Tackle Issue I: Gauge-Invariant Tangent Projection}
To eliminate the gauge dependence of factor-space updates, we treat the intrinsic update \(Z_\ell = A_\ell B_\ell^{\top}\) as a point on the fixed-rank manifold \(\mathcal{M}_r\) and operate directly in its tangent space. 
For full-column-rank \(A_\ell,B_\ell\), the tangent space at $Z_\ell$ admits the characterization (Appendix~\ref{app:tangent})
\begin{equation}
T_{Z_\ell}\mathcal{M}_r \;=\; \{\Delta Z_\ell = \Delta A_\ell B_\ell^{\top} + A_\ell\Delta B_\ell^{\top}\},
\label{eq:tangent}
\end{equation}
where \(\Delta A_\ell \in \mathbb{R}^{m\times r}\) and \(\Delta B_\ell \in \mathbb{R}^{n\times r}\). 
We refer to any pair $(\Delta A_\ell,\Delta B_\ell)$ satisfying this relation as a (factor-space) \emph{lift} of the intrinsic tangent matrix $\Delta Z_\ell$.
While such lifts are not unique, the induced matrix $\Delta Z_\ell$ depends only on $Z_\ell$ (Appendix~\ref{app:lift_min}).

To define a gauge-invariant intrinsic gradient, we introduce the orthogonal projectors onto the column spaces of factors,
\begin{equation}
\begin{aligned}
\Pi_{A_\ell} \triangleq A_\ell (A_\ell^\top A_\ell)^\dagger A_\ell^\top,~~~~ 
\Pi_{B_\ell} \triangleq B_\ell (B_\ell^\top B_\ell)^\dagger B_\ell^\top.
\end{aligned}
\label{eq:projectors}
\end{equation}
Given a per-example intrinsic gradient $G_{i,\ell}\in\R^{m\times n}$, we project it onto the tangent space via
\begin{align}\label{eq:tangent_proj}
\mathcal{P}_{A_\ell,B_\ell}(G_{i,\ell})
&\triangleq G_{i,\ell} - (I-\Pi_{A_\ell})\,G_{i,\ell}\,(I-\Pi_{B_\ell}) \\\nonumber
&= \Pi_{A_\ell} G_{i,\ell} + G_{i,\ell}\Pi_{B_\ell} - \Pi_{A_\ell} G_{i,\ell}\Pi_{B_\ell}.
\end{align}
which depends only on $\Pi_{A_\ell},\Pi_{B_\ell}$, and hence is invariant to the gauge transformation. As a result, the projected tangent direction $\mathcal{P}_{A_\ell,B_\ell}(G_{i,\ell})$ represents an intrinsic update
of $Z_\ell$ that is independent of the chosen factorization (Appendix~\ref{app:proj}).

To obtain a concrete factor-space representation, we adopt a canonical horizontal lift that maps the intrinsic tangent direction back to factor space in a gauge-consistent manner. 
Let $g_{A,i,\ell}=G_{i,\ell}B_\ell$, $g_{B,i,\ell}=G_{i,\ell}^\top A_\ell$, and define $M_\ell=A_\ell^\top A_\ell$, $N_\ell=B_\ell^\top B_\ell$. We set
\begin{equation}
\label{eq:canonical}
\begin{aligned}
\Delta A_{i,\ell} &= g_{A,i,\ell} N_\ell^\dagger - \tfrac12 \Pi_{A_\ell}\bigl(g_{A,i,\ell} N_\ell^\dagger\bigr),\\[2pt]
\Delta B_{i,\ell} &= g_{B,i,\ell} M_\ell^\dagger - \tfrac12 \Pi_{B_\ell}\bigl(g_{B,i,\ell} M_\ell^\dagger\bigr).
\end{aligned}
\end{equation}
which satisfies $\Delta A_{i,\ell} B_\ell^\top + A_\ell \Delta B_{i,\ell}^\top = \mathcal{P}_{A_\ell,B_\ell}(G_{i,\ell})$ (Appendix~\ref{app:lifts}).

\subsection{Tackle Issue II: Global Intrinsic DP Mechanism}
We next design a DP mechanism that operates directly on intrinsic tangent updates, thereby avoiding amplified noise inherent to factor-space perturbations.
 Given lifted tangent directions $(\Delta A_{i,\ell},\Delta B_{i,\ell})$ obtained from Eq.~\eqref{eq:canonical}, we measure the magnitude of tangent directions using the Frobenius norm
$\|\Delta Z_{i,\ell}\|_F^2=\|\Delta A_{i,\ell} B_{\ell}^\top + A_{\ell}\Delta B_{i,\ell}^\top\|_F^2$, which can be computed efficiently (Appendix~\ref{app:norm}):
\begin{align}\label{eq:tangent_norm}
\|\Delta Z_{i,\ell}\|_F^{2}
  &= \operatorname{tr}\!\bigl(\Delta A_{i,\ell}^{\top} \Delta A_{i,\ell}\,N_\ell\bigr)
   + \operatorname{tr}\!\bigl(\Delta B_{i,\ell}^{\top} \Delta B_{i,\ell}\,M_\ell\bigr)  \nonumber\\
  & + 2\,\operatorname{tr}\!\bigl((A_{\ell}^{\top} \Delta A_{i,\ell})(B_{\ell}^{\top} \Delta B_{i,\ell})\bigr).
\end{align}
In the common per-example gradient setting, Eq.~\eqref{eq:tangent_norm} further simplifies (Appendix~\ref{app:norm_rank1}).

To control sensitivity across all LoRA modules, we aggregate intrinsic norms and define the \emph{global intrinsic norm}
\begin{equation}\label{eq:global_clip}
\textstyle s_i      \triangleq \Bigl(\sum_{\ell=1}^{L} \|\Delta Z_{i,\ell}\|_F^{2}\Bigr)^{1/2}    
\end{equation}
We then compute per-example clipping coefficients $\alpha_i \triangleq \min\{1,\,C/s_i\}$. Each module aggregates the clipped lifts as $\bar\Delta A_\ell=\frac{1}{b}\sum_i \alpha_i\Delta A_{i,\ell}$ and $\bar\Delta B_\ell=\frac{1}{b}\sum_i \alpha_i\Delta B_{i,\ell}$.
This mirrors DP-SGD sensitivity control (Eq.~\eqref{eq:dpsgd}), but crucially operates in the intrinsic geometry of LoRA.

\underline{Isotropic tangent noise.}
For each module $\ell$, PRISM adds Gaussian noise directly in the tangent space,
\begin{equation}
(\Delta A_\ell^{\mathrm{dp}},\Delta B_\ell^{\mathrm{dp}})
=
(\bar\Delta A_\ell,\bar\Delta B_\ell)
+
\frac{\sigma C}{b}\,(\Xi_{A,\ell},\Xi_{B,\ell}),
\label{eq:dp_tangent_noise}
\end{equation}
where the random pair $(\Xi_{A,\ell},\Xi_{B,\ell})$ is constructed so that 
$\Xi_{A,\ell} B_\ell^\top + A_\ell \Xi_{B,\ell}^\top \sim \mathcal{P}_{A_\ell,B_\ell}(\Xi_\ell)$ for $\Xi_\ell\sim\mathcal{N}(0,I_{m_\ell\times n_\ell})$ (Appendix~\ref{app:lowrank}). 
PRISM uses factor-space lifts for efficiency, but the \emph{released} update is intrinsic and invariant to the factor lift; hence it admits an equivalent lift-free form:
\begin{equation}
\label{eq:dpnoise}
\widehat{\Delta Z}_\ell
  \;=\;
  \frac{1}{b}\sum_{i=1}^{b}\alpha_i\,\Delta Z_{i,\ell}
  \;+
  \frac{\sigma C}{b}\,\mathcal{P}_{A_\ell,B_\ell}(\Xi_\ell),
\end{equation}
where $\Delta Z_{i,\ell}=\mathcal{P}_{A_\ell,B_\ell}(G_{i,\ell})\in T_{Z_\ell}\mathcal{M}_r$. %and $\Xi_\ell\sim\mathcal{N}(0,I_{m_\ell\times n_\ell})$.
This intrinsic form is convenient for stating gauge invariance and privacy guarantees; in implementation, to sample $\mathcal{P}_{A_\ell,B_\ell}(\Xi_\ell)$ efficiently, we avoid drawing a full $m_\ell \times n_\ell$ Gaussian matrix and instead use a low-dimensional factor sampler (Appendix~\ref{app:noise_sampler}-\ref{app:factor_noise_gauge}):
%we sample $\mathcal{P}_{A_\ell,B_\ell}(\Xi_\ell)$ via the low-dimensional factor sampler below. Rather than sampling a full $m_\ell \times n_\ell$ Gaussian matrix, PRISM employs a low-dimensional sampler (Appendix~\ref{app:noise_sampler}) given by
\begin{align}\label{eq:noise_form}
\Xi_{A,\ell} = (I-\Pi_{A_\ell})\Omega_{A,\ell}\,N_\ell^{-\frac{1}{2}},~
\Xi_{B,\ell} = \Omega_{B,\ell}\,M_\ell^{-\frac{1}{2}}.
\end{align}
with $\Omega_{A,\ell}\sim\mathcal{N}(0,I_{m_\ell\times r})$ and $\Omega_{B,\ell}\sim\mathcal{N}(0,I_{n_\ell\times r})$. %(Theorem~\ref{thm:noise}). See Appendix~\ref{app:conditioning}-\ref{app:factor_noise_gauge}. 

\begin{theorem}[Isotropic tangent noise and closed-form intrinsic energy]
\label{thm:noise_energy}\label{thm:noise}
Let $\Xi_\ell\in\R^{m_\ell\times n_\ell}$ have i.i.d.\ $\mathcal{N}(0,1)$ entries.
Then $\mathcal{P}_{A_\ell,B_\ell}(\Xi_\ell)$ is an isotropic Gaussian supported on $T_{Z_\ell}\mathcal{M}_r$ and
\begin{equation}
\label{eq:energy_gila}
\E\|\mathcal{P}_{A_\ell,B_\ell}(\Xi_\ell)\|_F^2 = r(m_\ell+n_\ell-r).
\end{equation}
Therefore, the effective intrinsic noise of PRISM perturbation $\mathcal{N}_{Z_\ell}^{\text{PRISM}}=\frac{\sigma C}{b}\mathcal{P}_{A_\ell,B_\ell}(\Xi_\ell)$ is
\begin{equation}
\label{eq:eff_prism}
\mathcal{E}_{Z_\ell}^{\text{PRISM}}=\frac{\sigma C}{b}\sqrt{r(m_\ell+n_\ell-r)}.
\end{equation}
\end{theorem}
See Appendix~\ref{app:proof_noise_energy} for the proof, with supporting results in Appendix~\ref{app:isotropy}-\ref{app:noise_cov} and concentration bounds in Appendix~\ref{app:noise_concentration}. Theorem~\ref{thm:noise_energy} shows that projecting a dense Gaussian matrix onto the rank-$r$ tangent space yields an isotropic Gaussian supported on that subspace with expected squared norm $r(m+n-r)$. Consequently, PRISM induces intrinsic noise on $Z$ whose magnitude depends only on $(\sigma,C,b)$ and layer dimensions, and is independent of gauge-dependent quantities $\|A\|_F,\|B\|_F$.

\underline{Retraction without bilinear noise.}
Given the noisy tangent direction, PRISM updates the intrinsic parameter via a retraction onto the fixed-rank manifold $\calM_r$, 
\begin{equation}
Z_{\ell}^{+} \;=\; \mathrm{Retr}_r\!\big(
Z_\ell - \eta(\Delta A_\ell^{\mathrm{dp}}B_\ell^\top + A_\ell(\Delta B_\ell^{\mathrm{dp}})^\top)
\big),
\label{eq:retraction}
\end{equation}
where $\mathrm{Retr}_r(\cdot)$ denotes the best rank-$r$ approximation in Frobenius norm (Appendix~\ref{app:retraction}). Because \eqref{eq:retraction}  is additive in the intrinsic tangent perturbation, this step avoids the bilinear second-order noise term
$\eta^2\xi_{A,\ell}\xi_{B,\ell}^\top$ that arises when independently noising both factors (Eq.~\eqref{eq:bilinear_noise_prob}; Appendix~\ref{app:no_second_order}). 
Consequently, the effective intrinsic noise induced by PRISM admits a closed-form characterization (Eq.~\eqref{eq:eff_prism}), and retraction introduces only second-order distortion through the lifted factor-product residual.
\begin{proposition}[Retraction distortion is second order]
\label{prop:retract_error}
Let $Z=AB^\top\in\mathcal{M}_r$, with $A,B$ full column
rank, and let $\Delta Z=\Delta A B^\top + A\Delta B^\top
\in T_Z\mathcal{M}_r$ be a lifted tangent perturbation.
For the truncated-SVD retraction $\Retr_r$ and any $\eta\ge 0$,
\begin{equation}
\label{eq:retract_error1}
\|\Retr_r(Z-\eta\Delta Z)-(Z-\eta\Delta Z)\|_F
\le \eta^2\|\Delta A\Delta B^\top\|_F .
\end{equation}
Since $\|\Delta A\Delta B^\top\|_F
\le \|\Delta A\|_F\|\Delta B\|_F$, the distortion is
$O(\eta^2)$ for fixed $\Delta A,\Delta B$; equivalently,
$\Retr_r(Z-\eta\Delta Z)=Z-\eta\Delta Z+O(\eta^2)$ as
$\eta\to 0$.
\end{proposition}
Proposition~\ref{prop:retract_error} shows that retraction is
first-order exact for the lifted tangent step used by PRISM. Thus,
the DP perturbation remains additive to first order; the only
discrepancy is the second-order residual in
Eq.~\eqref{eq:retract_error1}, rather than an explicit
factor-space bilinear noise term.
\begin{theorem}[Gauge invariance of PRISM]
\label{thm:gauge_invariant}
Fix $(\sigma,C,b)$ and consider one PRISM step at intrinsic state $Z_\ell=A_\ell B_\ell^\top$.
For any $R\in\operatorname{GL}(r)$ and gauge-equivalent factors $(A'_\ell,B'_\ell)=(A_\ell R,B_\ell R^{-\top})$, the distribution of the intrinsic DP increment in Eq.~\eqref{eq:dpnoise} is invariant:
\[
\widehat{\Delta Z}_\ell(A_\ell,B_\ell)\ \overset{d}{=}\ \widehat{\Delta Z}_\ell(A'_\ell,B'_\ell).
\]
Since retraction in Eq.~\eqref{eq:retraction} is deterministic post-processing, $Z_\ell^+$ is also gauge invariant in distribution.
\end{theorem}

Theorem~\ref{thm:gauge_invariant} implies that PRISM is a well-defined randomized mechanism on the rank-$r$ manifold: the law of the clipped-and-noised increment is determined by $Z_\ell$ alone, not by the particular factor gauge $(A_\ell,B_\ell)$. See Appendix~\ref{app:gauge_invariant} for the proof.

\noindent\textbf{Privacy guarantee.}
Eq.~\eqref{eq:dpnoise} is a (subsampled) Gaussian mechanism on a linear space, and all subsequent operations---adaptive post-processing, factorization, alignment, and retraction---are DP-preserving by post-processing (Appendix~\ref{app:subspace_gaussian} and Appendix~\ref{app:procrustes}).

\begin{theorem}[DP guarantee of PRISM]
\label{thm:dp}
Assume Poisson subsampling with rate $q=b/N$ and per-example intrinsic clipping at threshold $C$ (Eq.~\eqref{eq:global_clip}).
Each PRISM iteration is a subsampled Gaussian mechanism with noise multiplier $\sigma$.
Consequently, for any target $\delta\in(0,1)$, after $T$ iterations PRISM satisfies $(\varepsilon,\delta)$-DP, where $\varepsilon$ is determined by composing $T$ subsampled Gaussian mechanisms and can be computed numerically using the privacy loss random variable (PRV) accountant~\cite{gopi2021numerical,yousefpour2021opacus,opacus_privacy_engine}.

%the privacy-loss composition of $T$ subsampled Gaussian mechanisms and can be computed numerically via the PRV accountant \cite{gopi2021numerical,yousefpour2021opacus,opacus_privacy_engine}.
\end{theorem}

Theorem~\ref{thm:dp} shows that each PRISM iteration is a subsampled Gaussian mechanism on intrinsic tangent updates; the remaining operations are DP-preserving post-processing. Hence, standard DP-SGD accounting applies; see Appendix~\ref{app:dp_proof} for the composition analysis.

\begin{table*}[t]
  \caption{
  \textbf{Utility on GLUE8 and Math-10K (higher is better).}
  ``Non-DP'' uses the same setup without DP clipping/noise; $\varepsilon\in\{6,3\}$ uses DP-SGD with $\delta=10^{-5}$ .
  \textbf{Avg} is the unweighted mean over the 12 tasks; bold is best per column.
  \emph{Takeaway:} Under DP, \textsc{PRISM} attains the best Avg and wins most tasks, especially on multi-step reasoning (GSM8K/MAWPS/SVAMP).}
  \label{tab:glue8_math10k_combined}
  \begin{center}
    \begin{small}
      \begin{sc}
        \setlength{\tabcolsep}{2.0pt}
        \renewcommand{\arraystretch}{0.95}
        \begin{tabular}{@{}ll*{13}{c}@{}}
          \toprule
          \textbf{Setting} & \textbf{Method} &
          \multicolumn{8}{c}{\textbf{GLUE8}} &
          \multicolumn{4}{c}{\textbf{Math-10K}} &
          \textbf{Avg} \\
          \cmidrule(lr){3-10}\cmidrule(lr){11-14}
          & &
          \textbf{CoLA} & \textbf{SST-2} & \textbf{MRPC} & \textbf{STS-B} &
          \textbf{QQP} & \textbf{MNLI} & \textbf{QNLI} & \textbf{RTE} &
          \textbf{GSM8K} & \textbf{AQuA} & \textbf{MAWPS} & \textbf{SVAMP} & \\
          \midrule

          \multirow{6}{*}{Non-DP} & FFA   & 0.456 & 0.935 & 0.759 & 0.821 & 0.713 & 0.736 & 0.809 & 0.809 & 0.513 & 0.476 & 0.836 & 0.678 & 0.712 \\
          & Rite  & 0.515 & 0.947 & \textbf{0.883} & \textbf{0.873} & \textbf{0.821} & \textbf{0.846} & \textbf{0.889} & \textbf{0.895} & \textbf{0.595} & \textbf{0.488} & \textbf{0.899} & \textbf{0.736} & \textbf{0.782} \\
          & AdamW  & 0.504 & \textbf{0.954} & 0.831 & 0.863 & 0.766 & 0.813 & 0.846 & 0.848 & 0.561 & 0.476 & 0.870 & 0.698 & 0.752 \\
          & LoRA+ & \textbf{0.578} & 0.950 & 0.840 & 0.862 & 0.807 & 0.845 & 0.851 & 0.838 & 0.592 & 0.465 & 0.891 & 0.712 & 0.769 \\
          & Lamb  & 0.468 & 0.939 & 0.860 & 0.872 & 0.776 & 0.842 & 0.868 & 0.856 & 0.559 & 0.449 & 0.878 & 0.708 & 0.756 \\
          & PRISM & 0.392 & 0.921 & 0.857 & 0.822 & 0.797 & 0.814 & 0.834 & 0.798 & 0.552 & 0.472 & 0.895 & 0.693 & 0.737 \\
          \midrule

          \multirow{6}{*}{$\epsilon=6$} & FFA   & 0.355 & 0.907 & 0.738 & 0.465 & 0.479 & 0.579 & 0.684 & 0.755 & 0.375 & 0.390 & 0.735 & 0.611 & 0.589 \\
          & Rite  & 0.235 & 0.787 & 0.635 & 0.268 & 0.500 & 0.482 & 0.562 & 0.657 & 0.282 & 0.366 & 0.597 & 0.503 & 0.490 \\
          & AdamW  & 0.407 & 0.915 & 0.770 & 0.659 & 0.493 & 0.651 & 0.716 & 0.798 & 0.441 & \textbf{0.465} & 0.761 & 0.615 & 0.641 \\
          & LoRA+ & 0.436 & 0.897 & 0.787 & 0.691 & 0.739 & \textbf{0.721} & 0.747 & \textbf{0.823} & 0.446 & 0.409 & 0.786 & 0.611 & 0.674 \\
          & Lamb  & 0.414 & \textbf{0.920} & 0.756 & 0.544 & 0.521 & 0.602 & 0.709 & 0.776 & 0.425 & 0.437 & 0.761 & 0.592 & 0.621 \\
          & PRISM & \textbf{0.444} & 0.919 & \textbf{0.798} & \textbf{0.718} & \textbf{0.770} & 0.707 & \textbf{0.776} & 0.791 & \textbf{0.469} & 0.445 & \textbf{0.819} & \textbf{0.626} & \textbf{0.690} \\
          \midrule

          \multirow{6}{*}{$\epsilon=3$} & FFA   & 0.337 & 0.890 & 0.730 & 0.406 & 0.466 & 0.561 & 0.662 & 0.740 & 0.350 & 0.374 & 0.718 & 0.598 & 0.569 \\
          & Rite  & 0.221 & 0.713 & 0.636 & 0.260 & 0.485 & 0.463 & 0.548 & 0.606 & 0.255 & 0.362 & 0.525 & 0.474 & 0.462 \\
          & AdamW  & 0.410 & 0.903 & 0.778 & 0.622 & 0.555 & 0.633 & 0.718 & \textbf{0.812} & 0.446 & 0.413 & 0.731 & 0.591 & 0.634 \\
          & LoRA+ & \textbf{0.434} & 0.906 & \textbf{0.798} & 0.668 & 0.730 & 0.708 & 0.740 & \textbf{0.812} & 0.419 & 0.386 & 0.765 & 0.609 & 0.665 \\
          & Lamb  & 0.396 & \textbf{0.909} & 0.759 & 0.517 & 0.486 & 0.586 & 0.708 & 0.783 & 0.393 & \textbf{0.425} & 0.744 & 0.608 & 0.609 \\
          & PRISM & 0.406 & 0.884 & 0.784 & \textbf{0.729} & \textbf{0.770} & \textbf{0.732} & \textbf{0.791} & 0.780 & \textbf{0.456} & 0.406 & \textbf{0.807} & \textbf{0.614} & \textbf{0.680} \\

          \bottomrule
        \end{tabular}
      \end{sc}
    \end{small}
  \end{center}
\end{table*}

\subsection{Tackle Issue III: DP-Aware Gauge-Invariant Adaptivity and Numerical Stability}
The mechanism in Eq.~\eqref{eq:dp_tangent_noise} produces a DP-sanitized tangent direction; by the post-processing property of DP, any subsequent transformation preserves the DP guarantee. We leverage this property to design a gauge-invariant adaptive update that is robust to privacy noise. For clarity, we describe the computation for a single LoRA module $\ell$.

\underline{Right-invariant preconditioning in rank space.}
For each module $\ell$, we track first moments $m_{A,\ell},m_{B,\ell}$ and rank-space 
second moments $V_{A,\ell},V_{B,\ell}\in\R^{r\times r}$ defined as,
\begin{equation}
\label{eq:adam_moments}
\begin{aligned}
m_{A,\ell} &\leftarrow \beta_1 m_{A,\ell} + (1-\beta_1)\,\Delta A_\ell^{\mathrm{dp}},\\[2pt]
V_{A,\ell} &\leftarrow \beta_2 V_{A,\ell}
      + (1-\beta_2)\,\frac{(\Delta A_\ell^{\mathrm{dp}})^{\top}\Delta A_\ell^{\mathrm{dp}}}{m_\ell}.
\end{aligned}
\end{equation}
with analogous updates for $m_{B,\ell},V_{B,\ell}$ (replacing $m_\ell$ by $n_\ell$). We precondition on the \emph{right} by inverse square roots and set the adaptive direction
\begin{align}\label{eq:precond_dir}
U_{A,\ell} &= m_{A,\ell}\,(V_{A,\ell} + \lambda_{A,\ell} I)^{-1/2},\nonumber\\
U_{B,\ell} &= m_{B,\ell}\,(V_{B,\ell} + \lambda_{B,\ell} I)^{-1/2}.
\end{align}
Under a gauge action $(A_\ell,B_\ell)\mapsto(A_\ell R,B_\ell R^{-\top})$, $V_{A,\ell}$ and $V_{B,\ell}$ transform by congruence and \eqref{eq:precond_dir} yields the same intrinsic update $U_{A,\ell} B_\ell^\top + A_\ell U_{B,\ell}^\top$.

\underline{DP-aware floors and conditioning control.}
Adaptive preconditioners can amplify DP noise when $V_{A,\ell}$ or $V_{B,\ell}$ has small or ill-conditioned eigenvalues. To mitigate this, 
PRISM introduces DP-aware floors $\lambda_{A,\ell},\lambda_{B,\ell}$, scaled according to the known DP noise level $\big(\frac{\sigma C}{b}\big)^2$ and the geometry of the current LoRA module. 
For isotropic tangent noise, the rank-space noise covariances satisfy $\E[\Xi_{A,\ell}^\top\Xi_{A,\ell}/m_\ell]=\frac{m_\ell-r}{m_\ell}N_\ell^{-1}$ and $\E[\Xi_{B,\ell}^\top\Xi_{B,\ell}/n_\ell]=M_\ell^{-1}$ (Appendix~\ref{app:noise_moments}). Small eigenvalues of $M_\ell$ or $N_\ell$ therefore simultaneously increase DP noise seen by the preconditioner and degrade numerical stability. Motivated by this,
we set
\begin{align}\label{eq:noise_floor}
\lambda_{A,\ell} \asymp \bigl(\tfrac{\sigma C}{b}\bigr)^{\!2}\,
  \frac{\operatorname{tr}(N_\ell^{-1})}{r},~~
\lambda_{B,\ell} \asymp
\bigl(\tfrac{\sigma C}{b}\bigr)^{\!2}\,
  \frac{\operatorname{tr}(M_\ell^{-1})}{r}.
\end{align}
%\end{equation}
These operations yield a uniform bound on DP noise amplification under the adaptive post-processing.

% \begin{theorem}[Bounding DP noise amplification under adaptive preconditioning]
% \label{thm:precond_bound}
% Let $V\succeq 0$ and $\lambda>0$, and define $\mathsf{P}=(V+\lambda I)^{-1/2}$.
% Then $\|\mathsf{P}\|_2\le \lambda^{-1/2}$ and for any matrix $X$,
% \begin{equation}
% \label{eq:precond_bound}
% \|X\mathsf{P}\|_F^2 \le \lambda^{-1}\|X\|_F^2.
% \end{equation}
% Apply $X\in\{\Delta A_\ell^{\mathrm{dp}},\Delta B_\ell^{\mathrm{dp}}\}$ and the floors in Eq.~\eqref{eq:noise_floor}, Eq.~\eqref{eq:precond_bound} implies that the amplification of DP perturbation energy induced by PRISM’s adaptive post-processing is bounded by a controlled factor; update clipping further enforces a deterministic bound on the intrinsic step. %implies that PRISM's adaptive post-processing cannot amplify the DP perturbation energy by more than a controlled factor; update clipping further enforces a deterministic bound on the intrinsic step.
% \end{theorem}
\begin{theorem}[Bounding DP noise amplification under adaptive preconditioning]
\label{thm:precond_bound}
Let $V\succeq 0$ and $\lambda>0$, and define $\mathsf{P}=V+\lambda I$.
Then $\|\mathsf{P}^{-1/2}\|_2\le \lambda^{-1/2}$ and for any matrix $X$,
\begin{equation}
\label{eq:precond_bound}
\|X\mathsf{P}^{-1/2}\|_F^2 \le \lambda^{-1}\|X\|_F^2.
\end{equation}
% Apply $X\in\{\Delta A_\ell^{\mathrm{dp}},\Delta B_\ell^{\mathrm{dp}}\}$ and the floors in Eq.~\eqref{eq:noise_floor}, Eq.~\eqref{eq:precond_bound} implies that the amplification of DP perturbation energy induced by PRISM’s adaptive post-processing is bounded by a controlled factor; update clipping further enforces a deterministic bound on the intrinsic step. %implies that PRISM's adaptive post-processing cannot amplify the DP perturbation energy by more than a controlled factor; update clipping further enforces a deterministic bound on the intrinsic step.
% Apply $X\in\{\Delta A_\ell^{\mathrm{dp}},\Delta B_\ell^{\mathrm{dp}}\}$ and the floors in Eq.~\eqref{eq:noise_floor}. Then Eq.~\eqref{eq:precond_bound} implies that the amplification of DP perturbation energy induced by PRISM’s adaptive post-processing is bounded by a controlled factor.
\end{theorem}
Theorem~\ref{thm:precond_bound} formalizes Issue III (Eq.~\eqref{eq:precond_dp}): DP-noise amplification under inverse-square-root right-preconditioning is governed by the preconditioner's spectral gain $g(\mathsf{P})\triangleq\|\mathsf{P}^{-1/2}\|_2=1/\sqrt{\lambda_{\min}^{+}(\mathsf{P})}$. A floor $\mathsf{P}=V+\lambda I$ forces $g(\mathsf{P})\le 1/\sqrt{\lambda}$, hence the Frobenius energy of any perturbation can increase by at most $1/\lambda$.

In PRISM, Eq.~\eqref{eq:dp_tangent_noise} injects DP noise into $(\Delta A_\ell^{\mathrm{dp}},\Delta B_\ell^{\mathrm{dp}})$, and Eq.~\eqref{eq:precond_dir} post-processes it by $(V_{A,\ell}+\lambda_{A,\ell}I)^{-1/2}$ and $(V_{B,\ell}+\lambda_{B,\ell}I)^{-1/2}$. The DP-aware floors in Eq.~\eqref{eq:noise_floor} keep $\lambda_{A,\ell},\lambda_{B,\ell}$ from collapsing 
 when $V_{A,\ell},V_{B,\ell}$ (or $M_\ell,N_\ell$) are ill-conditioned, so Eq.~\eqref{eq:precond_bound} yields controlled DP-noise amplification (bounded by $1/\lambda_{A,\ell}$ and $1/\lambda_{B,\ell}$), and intrinsic clipping caps the final step size (Appendix~\ref{app:precond}).

\begin{table*}[t]
\caption{Math-10K results on \texttt{Gemma-3-4B-pt}, \texttt{Gemma-2-9B}, and \texttt{Gemma-3-12B-pt} under DP ($r=16$, $\varepsilon=6$, $\delta=10^{-5}$). The “Type” column indicates the backbone family; all experiments use text-only inputs. Bold indicates the best result within each backbone.}
\label{tab:larger_backbones}
\centering
\small
\setlength{\tabcolsep}{4pt}
\begin{tabular}{lllccccc}
\toprule
Backbone & Type & Method & GSM8K & AQuA & MAWPS & SVAMP & Avg \\
\midrule
\texttt{Gemma-3-4B-pt} & Multimodal & AdamW & 0.441 & \textbf{0.465} & 0.761 & 0.615 & 0.571 \\
               &            & LoRA+ & 0.446 & 0.409 & 0.786 & 0.611 & 0.563 \\
               &            & PRISM & \textbf{0.469} & 0.445 & \textbf{0.819} & \textbf{0.626} & \textbf{0.590} \\
\midrule
\texttt{Gemma-2-9B}     & Text-only  & AdamW & 0.6473 & 0.4979 & 0.8093 & 0.7570 & 0.6779 \\
               &            & LoRA+ & 0.6293 & 0.4409 & 0.8067 & 0.6970 & 0.6435 \\
               &            & PRISM & \textbf{0.6603} & \textbf{0.5197} & \textbf{0.8487} & \textbf{0.7790} & \textbf{0.7019} \\
\midrule
\texttt{Gemma-3-12B-pt} & Multimodal & AdamW & 0.5807 & 0.4764 & 0.7311 & 0.6870 & 0.6188 \\
               &            & LoRA+ & 0.6346 & 0.5039 & \textbf{0.8193} & 0.7460 & 0.6760 \\
               &            & PRISM & \textbf{0.6535} & \textbf{0.5315} & \textbf{0.8193} & \textbf{0.7820} & \textbf{0.6966} \\
\bottomrule
\end{tabular}
\end{table*}

\section{Experiments}
\label{sec:experiments}

We benchmark \textsc{PRISM} for private LoRA fine-tuning on two multi-task instruction suites: \textbf{GLUE8} (NLU) and \textbf{Math-10K} (multi-step numerical reasoning), spanning diverse linguistic phenomena and compositional reasoning tasks to assess robustness across settings.

\textbf{Setup.}
We fine-tune the \texttt{Gemma-3-4B-pt} backbone \citep{gemma3} with LoRA \citep{hu2022lora}.
Our implementation follows LLM-Adapters \citep{hu2023llm_adapters,llm_adapters_code}:
Math-10K is used in its original form, while GLUE8 is constructed from GLUE \citep{wang2018glue} in the same instruction-format interface.
We report both \textbf{non-private} results and \textbf{$(\varepsilon,\delta)$-DP} results using DP-SGD \citep{abadi2016deep} with
$\varepsilon \in \{3,6\}$ and $\delta=10^{-5}$.
We use Opacus \citep{yousefpour2021opacus,opacus_privacy_engine} with the default PRV accountant \citep{gopi2021numerical}.
Full details are provided in Appendix~\ref{app:exp_details}.

\textbf{Datasets and metrics.}
\textbf{GLUE8} consists of eight GLUE tasks (excluding WNLI) \citep{wang2018glue}.
We evaluate on the official validation splits using standard GLUE metrics.
\textbf{Math-10K} combines GSM8K \citep{cobbe2021gsm8k}, AQuA \citep{ling2017program}, MAWPS \citep{koncel-kedziorski-etal-2016-mawps}, and SVAMP \citep{patel2021svamp} via LLM-Adapters \citep{hu2023llm_adapters}.
Performance is measured by exact answer accuracy using the LLM-Adapters protocol.

\textbf{Baselines.}
We compare \textsc{PRISM} against
\textbf{FFA} \citep{sun2024improvinglora},
\textbf{LoRA-RITE} \citep{yen2025rite},
\textbf{AdamW} \citep{kingma2015adam,loshchilov2019adamw},
\textbf{LoRA+} \citep{hayou2024loraplus},
and \textbf{LAMB} \citep{you2020lamb}.

\textbf{Main Results and Interpretation.}
Table~2 reports utility under non-private training and DP
training with $\varepsilon\in\{6,3\}$. Without DP,
LoRA-RITE achieves the best average performance, which is
expected since PRISM is designed to address
\emph{DP-specific} issues rather than to improve non-private
optimization. Under DP, PRISM achieves the best average
performance at both privacy budgets (0.690 at
$\varepsilon=6$; 0.680 at $\varepsilon=3$) and wins the
majority of tasks (8/12 and 7/12, respectively). PRISM is not best on every task (e.g., SST-2, RTE, and AQuA).
This is expected because
GLUE8 and Math-10K are trained as task suites
with shared hyperparameters, while individual tasks have
different convergence rates and sensitivities to DP noise.

The mechanism-level distinction is where the DP perturbation
is applied. Factor-space DP perturbs the factors $(A,B)$;
after multiplication, the effective perturbation on
$Z=AB^\top$ is scaled by the current factor norms and
therefore depends on the chosen gauge. Thus, under the same
nominal privacy budget, factor-space DP can induce uneven
intrinsic noise across layers and tasks, causing some
components to be over-noised. PRISM instead clips and
noises the gauge-invariant tangent update of $Z$, so the
induced intrinsic noise is bounded and independent of the
factorization. This makes the private updates more stable
across heterogeneous task suites and improves the average
DP utility, even when another optimizer is best on a few
individual tasks.

\textbf{Additional Scaling and Overhead Results.}
To assess robustness beyond the main setting, we further evaluate PRISM on larger backbones, across varying LoRA ranks, and with explicit runtime and memory profiling. For the backbone type, \texttt{Gemma-2-9B} is a text-only language model \citep{gemma2}, while \texttt{Gemma-3-4B-pt} and \texttt{Gemma-3-12B-pt} belong to the multimodal Gemma 3 family \citep{gemma3}; all benchmark inputs in our experiments are text-only.

Table~\ref{tab:larger_backbones} shows that PRISM’s advantage persists on both \texttt{Gemma-2-9B} and
\texttt{Gemma-3-12B-pt}. Table~\ref{tab:rank_sweep} shows that PRISM remains consistently the best method across ranks $r\in\{8,16,32\}$ on
\texttt{Gemma-3-4B-pt}. Finally, Table~\ref{tab:runtime_memory}
shows that PRISM’s overhead is primarily in runtime rather than memory usage, consistent with its additional geometry-aware computations. 

%To assess robustness beyond the main setting, we further evaluate PRISM on larger backbones, under varied ranks, and with explicit runtime/memory profiling. For the backbone type, Gemma-2-9B is a text-only language model \citep{gemma2}, while Gemma-3-4B-pt and Gemma-3-12B-pt are from the multimodal Gemma 3 family \citep{gemma3}; all benchmark inputs in our experiments are text-only. 

\begin{table}[h]
\caption{Rank sensitivity on \texttt{Gemma-3-4B-pt} under DP
($\varepsilon=6$, $\delta=10^{-5}$; text-only inputs).
Avg is the unweighted mean over all 12 tasks. PRISM is best
across all tested ranks.}
\label{tab:rank_sweep}
\centering
\small
\setlength{\tabcolsep}{4.5pt}
\begin{tabular}{llccc}
\toprule
Rank & Method & GLUE8 Avg & Math Avg & Avg \\
\midrule
$r=8$  & AdamW  & 0.659 & 0.522 & 0.614 \\
       & LoRA+  & 0.704 & 0.543 & 0.650 \\
       & PRISM  & \textbf{0.744} & \textbf{0.565} & \textbf{0.684} \\
\midrule
$r=16$ & AdamW  & 0.676 & 0.571 & 0.641 \\
       & LoRA+  & 0.730 & 0.563 & 0.674 \\
       & PRISM  & \textbf{0.740} & \textbf{0.590} & \textbf{0.690} \\
\midrule
$r=32$ & AdamW  & 0.682 & 0.542 & 0.636 \\
       & LoRA+  & 0.721 & 0.516 & 0.653 \\
       & PRISM  & \textbf{0.740} & \textbf{0.566} & \textbf{0.682} \\
\bottomrule
\end{tabular}
\end{table}

\begin{table}[h]
\caption{Runtime and memory profiling on \texttt{Gemma-3-4B-pt}
for Math-10K
($r=16$, $\varepsilon=6$). Measurements use a single
A100-40GB GPU with 10 warmup updates and 30 measured
updates. PRISM roughly doubles step time in the current
implementation, while peak memory is essentially unchanged.}
\label{tab:runtime_memory}
\centering
\small
\setlength{\tabcolsep}{6pt}
\begin{tabular}{lcc}
\toprule
Method & Step time (s) & Peak memory (MB) \\
\midrule
LoRA+ & 9.32  & 20961.1 \\
AdamW & 9.37  & 20961.1 \\
PRISM & 18.64 & 20964.3 \\
\bottomrule
\end{tabular}
\end{table}

\textbf{Mechanism Diagnostics (Three Issues).}
We next analyze how DP clipping and noise injection propagate to the intrinsic update $Z$ during Math-10K training (300 steps), and relate these effects to the three issues identified earlier.
\Cref{fig:issue2_amp_vs_step} illustrates systematic amplification of intrinsic noise in factor-space DP compared to \textsc{PRISM}.

\underline{\textit{Issue I: gauge sensitivity.}}
We perform DP training from gauge-rescaled initializations $(A,B)\mapsto(cA,c^{-1}B)$ with $c\in\{0.25,0.5,1,2,4\}$.
\Cref{fig:pathology1_dz_range} shows that factor-space DP remains sensitive to this benign reparameterization (Eq.~\eqref{eq:gauge_norm}), whereas \textsc{PRISM} quickly reduces this variability to near zero after warm-up (Appendix~\ref{app:pathology1_extra}).

\begin{figure}[h]
  \centering
  \includegraphics[width=\linewidth]{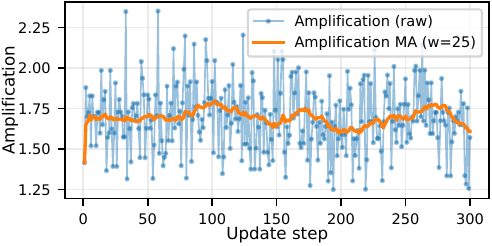}
  \caption{\textbf{Intrinsic DP-noise amplification during training (Math-10K).}
  We plot the per-step ratio $\|N_Z\|_F^{\text{fac}}/\|N_Z\|_F^{\text{\textsc{PRISM}}}$, where $\|N_Z\|_F$ is the Frobenius norm of the effective DP noise on the merged LoRA update $Z$. The blue curve reports the raw per-step ratio, and the orange
curve reports its moving average (MA) with window size $w=25$
updates.
  Values $>1$ indicate that applying DP-SGD in factor space $(A,B)$ injects a larger intrinsic perturbation into $Z$ than \textsc{PRISM} under the same privacy budget.}
  \label{fig:issue2_amp_vs_step}
\end{figure}
\begin{figure}[h]
  \centering
  \includegraphics[width=\linewidth]{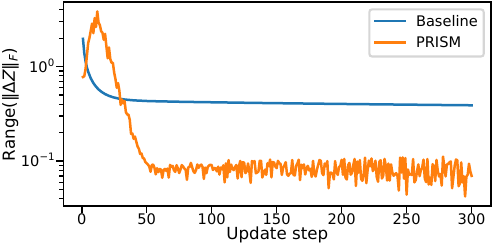}
  \caption{\textbf{Gauge sensitivity under DP (Math-10K).}
  At step $t$ we compute $\rho_t=\max_c\|\Delta Z_t\|_F-\min_c\|\Delta Z_t\|_F$ across gauge-rescaled runs; smaller is better and $\rho_t\approx 0$ indicates practical gauge invariance.
  \textsc{PRISM} drives $\rho_t$ near zero, whereas factor-space DP exhibits persistently large $\rho_t$.}
  \label{fig:pathology1_dz_range}
\end{figure}

\underline{\textit{Issue II: gauge-dependent intrinsic noise.}}
Proposition~\ref{prop:amp} predicts $\mathbb{E}\|N_Z\|_F^2=\tau^2 S_t$ for factor-space DP, where $S_t=\sum_\ell(m_\ell\|B_\ell\|_F^2+n_\ell\|A_\ell\|_F^2)$ varies under gauge rescaling.
\Cref{fig:issue2_noise_sq_vs_St} confirms this linear scaling behavior for the baseline, while \textsc{PRISM} remains substantially lower and nearly invariant to $S_t$ (Eq.~\eqref{eq:eff_prism}).
Appendix~\ref{app:issue2_gauge_sweep} further fixes $Z$ and varies $c$ to reproduce Corollary~\ref{cor:unbounded}.

\begin{figure}[t]
  \centering
  \includegraphics[width=\linewidth]{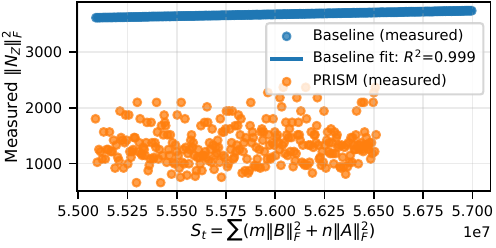}
  \caption{\textbf{Gauge-dependent intrinsic DP noise (Math-10K).}
  We plot the measured intrinsic noise energy $\|N_Z\|_F^2$ against the gauge-dependent statistic $S_t$.
  A strong linear trend indicates that the amount of DP noise injected into $Z$ depends on the factorization; \textsc{PRISM} largely removes this dependence and keeps $\|N_Z\|_F^2$ low.}
  \label{fig:issue2_noise_sq_vs_St}
\end{figure}

\underline{\textit{Issue III: DP under adaptive preconditioning.}}
\Cref{fig:path3_sigma_precond} sweeps the DP noise multiplier $\sigma$ and measures the resulting preconditioned intrinsic noise magnitude (Eq.~\eqref{eq:precond_dp}).
Factor-space DP-AdamW exhibits a ``noise-normalization'' effect (Proposition~\ref{prop:noise_normalization}), whereas \textsc{PRISM} with DP-aware floors (Eq.~\eqref{eq:noise_floor}) consistently reduces the perturbation across all $\sigma$, which aligns with Theorem~\ref{thm:precond_bound} (Appendix~\ref{app:path3}).

\begin{figure}[H]
  \centering
  \includegraphics[width=\linewidth]{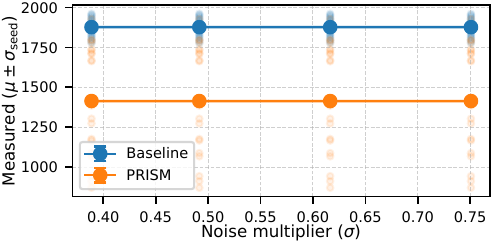}
  \caption{\textbf{Preconditioned intrinsic DP noise vs.\ $\sigma$ (Math-10K).}
  We report $\mathbb{E}\|\mathsf{P}_t^{-1/2}\boldsymbol{\xi}_{\mathrm{intr}}\|_F$; lower means the optimizer applies less stochastic perturbation after preconditioning.
  Factor-space DP becomes nearly $\sigma$-invariant, while \textsc{PRISM} keeps the preconditioned noise smaller via DP-aware floors.}
  \label{fig:path3_sigma_precond}
\end{figure}

\paragraph{Limitations.}
PRISM is tailored to LoRA-style fixed-rank updates; extending it to other PEFT methods requires deriving the corresponding intrinsic geometry. Its main practical cost is runtime overhead from geometry-aware operations, while peak memory remains nearly unchanged.
% \section{Conclusion}
% \label{sec:conclusion}

% We studied differential privacy for LoRA under the correct intrinsic parameterization.
% Naively applying DP-SGD to LoRA factors introduces gauge-dependent clipping/noising and an additional bilinear noise term, yielding an \emph{effective intrinsic noise} that can be amplified without bound by benign reparameterizations.
% PRISM resolves these issues by defining the DP mechanism on the rank-$r$ manifold: project per-example gradients to the tangent space, clip and add an isotropic tangent Gaussian, and retract to rank $r$.
% This construction is gauge invariant as a randomized mechanism and yields a closed-form, deterministic intrinsic noise level.
% Empirically, PRISM improves DP utility on both GLUE8 and Math-10K under the same privacy budget.

\section*{Acknowledgements} This work was funded in part by the National Science Foundation under award number  IIS2202699, IIS-2416895, IIS-2301599, CMMI2301601, and DMS-2529302.

\section*{Impact Statement}
This work develops a differentially private method for LoRA fine-tuning.
We expect it to reduce privacy risks when adapting models on sensitive data; we do not foresee additional negative societal impacts beyond those typical of deploying ML systems.

\bibliographystyle{icml2026}
\bibliography{prism_icml2026}

@inproceedings{dwork2006calibrating,
  title={Calibrating noise to sensitivity in private data analysis},
  author={Dwork, Cynthia and McSherry, Frank and Nissim, Kobbi and Smith, Adam},
  booktitle={Theory of cryptography conference},
  pages={265--284},
  year={2006},
  organization={Springer}
}

@article{dwork2014algorithmic,
  title={The Algorithmic Foundations of Differential Privacy},
  author={Dwork, Cynthia and Roth, Aaron},
  journal={Foundations and Trends in Theoretical Computer Science},
  volume={9},
  number={3--4},
  pages={211--407},
  year={2014},
  publisher={Now Publishers Inc.},
  doi={10.1561/0400000042},
  url={https://doi.org/10.1561/0400000042}
}

@article{gemma2,
  title={Gemma 2: Improving Open Language Models at a Practical Size},
  author={{Gemma Team} and Riviere, Morgane and Pathak, Shreya and Sessa, Pier Giuseppe and Hardin, Cassidy and Bhupatiraju, Surya and Hussenot, L{\'e}onard and Mesnard, Thomas and Shahriari, Bobak and Ram{\'e}, Alexandre and others},
  journal={arXiv preprint arXiv:2408.00118},
  year={2024},
  doi={10.48550/arXiv.2408.00118},
  url={https://arxiv.org/abs/2408.00118}
}

@article{gemma3,
  title={Gemma 3 Technical Report},
  author={{Gemma Team} and Kamath, Aishwarya and Ferret, Johan and Pathak, Shreya and Vieillard, Nino and Merhej, Ramona and Perrin, Sarah and Matejovicova, Tatiana and Ram{\'e}, Alexandre and Rivi{\`e}re, Morgane and others},
  journal={arXiv preprint arXiv:2503.19786},
  year={2025},
  doi={10.48550/arXiv.2503.19786},
  url={https://arxiv.org/abs/2503.19786}
}

@inproceedings{abadi2016deep,
  title={Deep learning with differential privacy},
  author={Abadi, Martin and Chu, Andy and Goodfellow, Ian and McMahan, H Brendan and Mironov, Ilya and Talwar, Kunal and Zhang, Li},
  booktitle={Proceedings of the 2016 ACM SIGSAC conference on computer and communications security},
  pages={308--318},
  year={2016}
}

@inproceedings{
hu2022lora,
title={Lo{RA}: Low-Rank Adaptation of Large Language Models},
author={Edward J Hu and Yelong Shen and Phillip Wallis and Zeyuan Allen-Zhu and Yuanzhi Li and Shean Wang and Lu Wang and Weizhu Chen},
booktitle={International Conference on Learning Representations},
year={2022},
url={https://openreview.net/forum?id=nZeVKeeFYf9}
}

@inproceedings{
yu2022dpftlm,
title={Differentially Private Fine-tuning of Language Models},
author={Da Yu and Saurabh Naik and Arturs Backurs and Sivakanth Gopi and Huseyin A Inan and Gautam Kamath and Janardhan Kulkarni and Yin Tat Lee and Andre Manoel and Lukas Wutschitz and Sergey Yekhanin and Huishuai Zhang},
booktitle={International Conference on Learning Representations},
year={2022},
url={https://openreview.net/forum?id=Q42f0dfjECO}
}

@article{liu2023dplora,
  title={Differentially Private Low-Rank Adaptation of Large Language Model Using Federated Learning},
  author={Liu, Xiao-Yang and Zhu, Rongyi and Zha, Daochen and Gao, Jiechao and Zhong, Shan and White, Matt and Qiu, Meikang},
  journal={ACM Transactions on Management Information Systems},
  volume={16},
  number={2},
  year={2025},
  doi={10.1145/3682068},
  url={https://doi.org/10.1145/3682068}
}

@inproceedings{xu2025dpfedlora,
  title={{DP-FedLoRA}: Privacy-Enhanced Federated Fine-Tuning for On-Device Large Language Models},
  author={Xu, Honghui and Shrestha, Shiva and Chen, Wei and Li, Zhiyuan and Cai, Zhipeng},
  booktitle={2025 IEEE International Conference on Data Mining (ICDM)},
  pages={813--822},
  year={2025},
  publisher={IEEE},
  doi={10.1109/ICDM65498.2025.00089},
  url={https://doi.org/10.1109/ICDM65498.2025.00089}
}

@inproceedings{
sun2024improvinglora,
title={Improving {LoRA} in Privacy-preserving Federated Learning},
author={Youbang Sun and Zitao Li and Yaliang Li and Bolin Ding},
booktitle={The Twelfth International Conference on Learning Representations},
year={2024},
url={https://openreview.net/forum?id=NLPzL6HWNl}
}

@inproceedings{kingma2015adam,
  title     = {Adam: A Method for Stochastic Optimization},
    author    = {Kingma, Diederik P. and Ba, Jimmy},

  booktitle = {International Conference on Learning Representations (ICLR)},
  year      = {2015},
  url       = {https://arxiv.org/abs/1412.6980}
}

@inproceedings{pmlr-v162-li22x,
  title={Private Adaptive Optimization with Side information},
  author={Li, Tian and Zaheer, Manzil and Reddi, Sashank and Smith, Virginia},
  booktitle={Proceedings of the 39th International Conference on Machine Learning},
  pages={13086--13105},
  year={2022},
  volume={162},
  series={Proceedings of Machine Learning Research},
  publisher={PMLR},
  url={https://proceedings.mlr.press/v162/li22x.html}
}

@inproceedings{
li2022differentially,
title={Differentially Private Adaptive Optimization with Delayed Preconditioners},
author={Tian Li and Manzil Zaheer and Ken Liu and Sashank J. Reddi and Hugh Brendan McMahan and Virginia Smith},
booktitle={The Eleventh International Conference on Learning Representations },
year={2023},
url={https://openreview.net/forum?id=j1zQGmQQOX1}
}

@article{tang2023dpadambc,
  title={{DP-AdamBC}: Your {DP-Adam} Is Actually {DP-SGD} (Unless You Apply Bias Correction)},
  author = {Tang, Qiaoyue and Shpilevskiy, Frederick and L\'{e}cuyer, Mathias},
  journal={Proceedings of the AAAI Conference on Artificial Intelligence},
  volume={38},
  number={14},
  pages={15276--15283},
  year={2024},
  doi={10.1609/aaai.v38i14.29451},
  url={https://ojs.aaai.org/index.php/AAAI/article/view/29451}
}

@inproceedings{
loshchilov2019adamw,
title={Decoupled Weight Decay Regularization},
author={Ilya Loshchilov and Frank Hutter},
booktitle={International Conference on Learning Representations},
year={2019},
url={https://openreview.net/forum?id=Bkg6RiCqY7},
}

@inproceedings{
yen2025rite,
title={{LoRA} {Done} {RITE}: Robust Invariant Transformation Equilibration for {LoRA} Optimization},
author={Jui-Nan Yen and Si Si and Zhao Meng and Felix Yu and Sai Surya Duvvuri and Inderjit S Dhillon and Cho-Jui Hsieh and Sanjiv Kumar},
booktitle={The Thirteenth International Conference on Learning Representations},
year={2025},
url={https://openreview.net/forum?id=VpWki1v2P8}
}

@inproceedings{fredrikson2015modelinversion,
  title={Model Inversion Attacks that Exploit Confidence Information and Basic Countermeasures},
  author={Fredrikson, Matt and Jha, Somesh and Ristenpart, Thomas},
  booktitle={Proceedings of the 22nd ACM SIGSAC Conference on Computer and Communications Security},
  series={CCS '15},
  pages={1322--1333},
  year={2015},
  publisher={Association for Computing Machinery},
  doi={10.1145/2810103.2813677},
  url={https://doi.org/10.1145/2810103.2813677}
}

@inproceedings{ganju2018property,
  title={Property Inference Attacks on Fully Connected Neural Networks using Permutation Invariant Representations},
  author={Ganju, Karan and Wang, Qi and Yang, Wei and Gunter, Carl A. and Borisov, Nikita},
  booktitle={Proceedings of the 2018 ACM SIGSAC Conference on Computer and Communications Security},
  series={CCS '18},
  pages={619--633},
  year={2018},
  publisher={Association for Computing Machinery},
  doi={10.1145/3243734.3243834},
  url={https://doi.org/10.1145/3243734.3243834}
}

@misc{yousefpour2021opacus,
      title={Opacus: User-Friendly Differential Privacy Library in {PyTorch}}, 
      author={Ashkan Yousefpour and Igor Shilov and Alexandre Sablayrolles and Davide Testuggine and Karthik Prasad and Mani Malek and John Nguyen and Sayan Ghosh and Akash Bharadwaj and Jessica Zhao and Graham Cormode and Ilya Mironov},
      year={2022},
      eprint={2109.12298},
      archivePrefix={arXiv},
      primaryClass={cs.LG},
      url={https://arxiv.org/abs/2109.12298}, 
}

@inproceedings{houlsby2019adapter,
  title={Parameter-Efficient Transfer Learning for {NLP}},
  author={Houlsby, Neil and Giurgiu, Andrei and Jastrzebski, Stanislaw and Morrone, Bruna and De Laroussilhe, Quentin and Gesmundo, Andrea and Attariyan, Mona and Gelly, Sylvain},
  booktitle={Proceedings of the 36th International Conference on Machine Learning},
  pages={2790--2799},
  year={2019},
  volume={97},
  series={Proceedings of Machine Learning Research},
  publisher={PMLR},
  url={https://proceedings.mlr.press/v97/houlsby19a.html}
}

@inproceedings{li2021prefix,
  title={{Prefix-Tuning}: Optimizing Continuous Prompts for Generation},
  author={Li, Xiang Lisa and Liang, Percy},
  booktitle={Proceedings of the 59th Annual Meeting of the Association for Computational Linguistics and the 11th International Joint Conference on Natural Language Processing (Volume 1: Long Papers)},
  pages={4582--4597},
  year={2021},
  publisher={Association for Computational Linguistics},
  doi={10.18653/v1/2021.acl-long.353},
  url={https://aclanthology.org/2021.acl-long.353/}
}

@inproceedings{lester2021prompt,
  title={The Power of Scale for Parameter-Efficient Prompt Tuning},
  author={Lester, Brian and Al-Rfou, Rami and Constant, Noah},
  booktitle={Proceedings of the 2021 Conference on Empirical Methods in Natural Language Processing},
  pages={3045--3059},
  year={2021},
  publisher={Association for Computational Linguistics},
  doi={10.18653/v1/2021.emnlp-main.243},
  url={https://aclanthology.org/2021.emnlp-main.243/}
}

@inproceedings{zaken2021bitfit,
  title={{B}it{F}it: Simple Parameter-efficient Fine-tuning for Transformer-based Masked Language-models},
  author={Ben Zaken, Elad and Goldberg, Yoav and Ravfogel, Shauli},
  booktitle={Proceedings of the 60th Annual Meeting of the Association for Computational Linguistics (Volume 2: Short Papers)},
  pages={1--9},
  year={2022},
  publisher={Association for Computational Linguistics},
  doi={10.18653/v1/2022.acl-short.1},
  url={https://aclanthology.org/2022.acl-short.1/}
}

@inproceedings{
dettmers2023qlora,
title={{QL}o{RA}: Efficient Finetuning of Quantized {LLM}s},
author={Tim Dettmers and Artidoro Pagnoni and Ari Holtzman and Luke Zettlemoyer},
booktitle={Thirty-seventh Conference on Neural Information Processing Systems},
year={2023},
url={https://openreview.net/forum?id=OUIFPHEgJU}
}

@inproceedings{
frantar2023loftq,
  title={{LoftQ}: {LoRA}-Fine-Tuning-aware Quantization for Large Language Models},

author={Yixiao Li and Yifan Yu and Chen Liang and Nikos Karampatziakis and Pengcheng He and Weizhu Chen and Tuo Zhao},
booktitle={The Twelfth International Conference on Learning Representations},
year={2024},
url={https://openreview.net/forum?id=LzPWWPAdY4}
}

@inproceedings{liu2022ia3,
 author = {Liu, Haokun and Tam, Derek and Muqeeth, Mohammed and Mohta, Jay and Huang, Tenghao and Bansal, Mohit and Raffel, Colin},
 booktitle = {Advances in Neural Information Processing Systems},
 editor = {S. Koyejo and S. Mohamed and A. Agarwal and D. Belgrave and K. Cho and A. Oh},
 pages = {1950--1965},
 publisher = {Curran Associates, Inc.},
 title = {Few-Shot Parameter-Efficient Fine-Tuning is Better and Cheaper than In-Context Learning},

 volume = {35},
 year = {2022}
}

@inproceedings{
you2020lamb,
title={Large Batch Optimization for Deep Learning: Training {BERT} in 76 minutes},
author={Yang You and Jing Li and Sashank Reddi and Jonathan Hseu and Sanjiv Kumar and Srinadh Bhojanapalli and Xiaodan Song and James Demmel and Kurt Keutzer and Cho-Jui Hsieh},
booktitle={International Conference on Learning Representations},
year={2020},
url={https://openreview.net/forum?id=Syx4wnEtvH}
}

@article{edelman1998geometry,
  title={The Geometry of Algorithms with Orthogonality Constraints},
  author={Edelman, Alan and Arias, Tom{\'a}s A. and Smith, Steven T.},
  journal={SIAM Journal on Matrix Analysis and Applications},
  volume={20},
  number={2},
  pages={303--353},
  year={1998},
  doi={10.1137/S0895479895290954},
  url={https://doi.org/10.1137/S0895479895290954}
}

@article{mishra2014fixedrank,
  title={Fixed-Rank Matrix Factorizations and {Riemannian} Low-Rank Optimization},
  author={Mishra, Bamdev and Meyer, Gilles and Bonnabel, Silv{\`e}re and Sepulchre, Rodolphe},
  journal={Computational Statistics},
  volume={29},
  pages={591--621},
  year={2014},
  doi={10.1007/s00180-013-0464-z},
  url={https://doi.org/10.1007/s00180-013-0464-z}
}

@article{eckart1936approximation,
  title={The Approximation of One Matrix by Another of Lower Rank},
  author={Eckart, Carl and Young, Gale},
  journal={Psychometrika},
  volume={1},
  number={3},
  pages={211--218},
  year={1936},
  doi={10.1007/BF02288367},
  url={https://doi.org/10.1007/BF02288367}
}

@article{mirsky1960symmetric,
  title={Symmetric Gauge Functions and Unitarily Invariant Norms},
  author={Mirsky, L.},
  journal={The Quarterly Journal of Mathematics},
  volume={11},
  number={1},
  pages={50--59},
  year={1960},
  doi={10.1093/qmath/11.1.50},
  url={https://doi.org/10.1093/qmath/11.1.50}
}

@misc{cobbe2021gsm8k,
      title={Training Verifiers to Solve Math Word Problems}, 
      author={Karl Cobbe and Vineet Kosaraju and Mohammad Bavarian and Mark Chen and Heewoo Jun and Lukasz Kaiser and Matthias Plappert and Jerry Tworek and Jacob Hilton and Reiichiro Nakano and Christopher Hesse and John Schulman},
      year={2021},
      eprint={2110.14168},
      archivePrefix={arXiv},
      primaryClass={cs.LG},
      url={https://arxiv.org/abs/2110.14168}, 
}

@inproceedings{ling2017program,
  title={Program Induction by Rationale Generation: Learning to Solve and Explain Algebraic Word Problems},
  author={Ling, Wang and Yogatama, Dani and Dyer, Chris and Blunsom, Phil},
  booktitle={Proceedings of the 55th Annual Meeting of the Association for Computational Linguistics (Volume 1: Long Papers)},
  pages={158--167},
  year={2017},
  publisher={Association for Computational Linguistics},
  doi={10.18653/v1/P17-1015},
  url={https://aclanthology.org/P17-1015/}
}

@inproceedings{patel2021svamp,
  title={Are {NLP} Models really able to Solve Simple Math Word Problems?},
  author={Patel, Arkil and Bhattamishra, Satwik and Goyal, Navin},
  booktitle={Proceedings of the 2021 Conference of the North American Chapter of the Association for Computational Linguistics: Human Language Technologies},
  pages={2080--2094},
  year={2021},
  address={Online},
  publisher={Association for Computational Linguistics},
  doi={10.18653/v1/2021.naacl-main.168},
  url={https://aclanthology.org/2021.naacl-main.168/}
}

@inproceedings{wang2018glue,
  title={{GLUE}: A Multi-Task Benchmark and Analysis Platform for Natural Language Understanding},
  author={Wang, Alex and Singh, Amanpreet and Michael, Julian and Hill, Felix and Levy, Omer and Bowman, Samuel R.},
  booktitle={Proceedings of the 2018 {EMNLP} Workshop {B}lackbox{NLP}: Analyzing and Interpreting Neural Networks for {NLP}},
  pages={353--355},
  year={2018},
  address={Brussels, Belgium},
  publisher={Association for Computational Linguistics},
  doi={10.18653/v1/W18-5446},
  url={https://aclanthology.org/W18-5446/}
}

@inproceedings{hu2023llm_adapters,
  title={{LLM-Adapters}: An Adapter Family for Parameter-Efficient Fine-Tuning of Large Language Models},
  author={Hu, Zhiqiang and Wang, Lei and Lan, Yihuai and Xu, Wanyu and Lim, Ee-Peng and Bing, Lidong and Xu, Xing and Poria, Soujanya and Lee, Roy},
  booktitle={Proceedings of the 2023 Conference on Empirical Methods in Natural Language Processing},
  pages={5254--5276},
  year={2023},
  address={Singapore},
  publisher={Association for Computational Linguistics},
  doi={10.18653/v1/2023.emnlp-main.319},
  url={https://aclanthology.org/2023.emnlp-main.319/}
}

@misc{llm_adapters_code,
  title={{LLM-Adapters}: Official Implementation},
  author={{AGI-Edgerunners}},
  howpublished={\url{https://github.com/AGI-Edgerunners/LLM-Adapters}},
  year={2023},
}

@inproceedings{shokri2017membership,
  title={Membership Inference Attacks Against Machine Learning Models},
  author={Shokri, Reza and Stronati, Marco and Song, Congzheng and Shmatikov, Vitaly},
  booktitle={2017 IEEE Symposium on Security and Privacy (SP)},
  pages={3--18},
  year={2017},
  publisher={IEEE Computer Society},
  doi={10.1109/SP.2017.41},
  url={https://doi.org/10.1109/SP.2017.41}
}

@inproceedings{carlini2019secret,
  title={The Secret Sharer: Evaluating and Testing Unintended Memorization in Neural Networks},
  author={Carlini, Nicholas and Liu, Chang and Erlingsson, {\'U}lfar and Kos, Jernej and Song, Dawn},
  booktitle={28th USENIX Security Symposium (USENIX Security 19)},
  pages={267--284},
  year={2019},
  publisher={USENIX Association},
  url={https://www.usenix.org/conference/usenixsecurity19/presentation/carlini}
}

@inproceedings{carlini2021extracting,
  title={Extracting Training Data from Large Language Models},
  author={Carlini, Nicholas and Tram{\`e}r, Florian and Wallace, Eric and Jagielski, Matthew and Herbert-Voss, Ariel and Lee, Katherine and Roberts, Adam and Brown, Tom and Song, Dawn and Erlingsson, {\'U}lfar and Oprea, Alina and Raffel, Colin},
  booktitle={30th USENIX Security Symposium (USENIX Security 21)},
  pages={2633--2650},
  year={2021},
  publisher={USENIX Association},
  url={https://www.usenix.org/conference/usenixsecurity21/presentation/carlini-extracting}
}

@inproceedings{bourtoule2021machineunlearning,
  author={Bourtoule, Lucas and Chandrasekaran, Varun and Choquette-Choo, Christopher A. and Jia, Hengrui and Travers, Adelin and Zhang, Baiwu and Lie, David and Papernot, Nicolas},
  title={Machine Unlearning},
  booktitle={2021 IEEE Symposium on Security and Privacy (SP)},
  pages={141--159},
  year={2021},
  doi={10.1109/SP40001.2021.00019}
}

@inproceedings{gopi2021numerical,
  title={Numerical Composition of Differential Privacy},
  author={Gopi, Sivakanth and Lee, Yin Tat and Wutschitz, Lukas},
  booktitle={Advances in Neural Information Processing Systems},
  volume={34},
  pages={11631--11642},
  year={2021},
  publisher={Curran Associates, Inc.},
  url={https://proceedings.neurips.cc/paper_files/paper/2021/file/6097d8f3714205740f30debe1166744e-Paper.pdf}
}

@misc{opacus_privacy_engine,
  title={{Opacus} {PrivacyEngine} {API} Reference},
  author       = {{Opacus Contributors}},
  howpublished = {\url{https://opacus.ai/api/privacy_engine.html}},
  year         = {2026},
}

@inproceedings{hayou2024loraplus,
  title={{L}o{RA}+: Efficient Low Rank Adaptation of Large Models},
  author={Hayou, Soufiane and Ghosh, Nikhil and Yu, Bin},
  booktitle={Proceedings of the 41st International Conference on Machine Learning},
  pages={17783--17806},
  year={2024},
  volume={235},
  series={Proceedings of Machine Learning Research},
  publisher={PMLR},
  url={https://proceedings.mlr.press/v235/hayou24a.html}
}

@inproceedings{koncel-kedziorski-etal-2016-mawps,
    title = "{MAWPS}: A Math Word Problem Repository",
    author = "Koncel-Kedziorski, Rik  and
      Roy, Subhro  and
      Amini, Aida  and
      Kushman, Nate  and
      Hajishirzi, Hannaneh",
    editor = "Knight, Kevin  and
      Nenkova, Ani  and
      Rambow, Owen",
    booktitle = "Proceedings of the 2016 Conference of the North {A}merican Chapter of the Association for Computational Linguistics: Human Language Technologies",
    month = jun,
    year = "2016",
    address = "San Diego, California",
    publisher = "Association for Computational Linguistics",
    url = "https://aclanthology.org/N16-1136/",
    doi = "10.18653/v1/N16-1136",
    pages = "1152--1157"
}

%%%%%%%%%%%%%%%%%%%%%%%%%%%%%%%%%%%%%%%%%%%%%%%%%%%%%%%
% APPENDIX
%%%%%%%%%%%%%%%%%%%%%%%%%%%%%%%%%%%%%%%%%%%%%%%%%%%%%%%

\clearpage
\onecolumn
\appendix

\renewcommand{\projA}{\Pi_{A_{\ell}}}
\renewcommand{\projB}{\Pi_{B_{\ell}}}
\renewcommand{\projperpA}{\Pi_{A_{\ell}}^{\perp}}
\renewcommand{\projperpB}{\Pi_{B_{\ell}}^{\perp}}

\section*{Appendix Roadmap}
\label{app:roadmap}

\begingroup
\small
\setlength{\parindent}{0pt}
\setlength{\parskip}{0.15em}

\newcommand{\AppRoadmapHeader}[2]{%
  \vspace{0.35em}%
  \noindent\hyperref[#2]{\textbf{#1}}\par%
  \vspace{-0.15em}%
}
\newcommand{\AppRoadmapEntry}[3]{%
  \noindent\makebox[3.0em][l]{\hyperref[#3]{\textbf{#1}}}%
  #2\dotfill%
  \hyperref[#3]{\textit{p.~\pageref*{#3}}}\par%
}

\AppRoadmapHeader{Appendix A: Theory and proofs.}{app:theory}
\AppRoadmapEntry{A1}{Quotient geometry / gauge quotient.}{app:quotient_geometry}
\AppRoadmapEntry{A2}{General gauge amplification.}{app:general_gauge_amp}
\AppRoadmapEntry{A3}{Second-order term concentration.}{app:second_order_stats}
\AppRoadmapEntry{A4}{Naive factor-space DP amplification.}{app:naive_noise}
\AppRoadmapEntry{A5}{Effective noise range under rescaling.}{app:eff_noise_range}
\AppRoadmapEntry{A6}{Tangent space projector properties.}{app:tangent}
\AppRoadmapEntry{A7}{Gauge freedom in tangent lifts.}{app:lift_min}
\AppRoadmapEntry{A8}{Projector gauge invariance.}{app:proj}
\AppRoadmapEntry{A9}{Canonical factor lift construction.}{app:lifts}
\AppRoadmapEntry{A10}{Tangent update Frobenius norm.}{app:norm}
\AppRoadmapEntry{A11}{Rank-1 specialization.}{app:norm_rank1}
\AppRoadmapEntry{A12}{Factorized tangent noise equivalence.}{app:lowrank}
\AppRoadmapEntry{A13}{Low-dimensional noise sampler.}{app:noise_sampler}
\AppRoadmapEntry{A14}{Stable projector / basis computation.}{app:conditioning}
\AppRoadmapEntry{A15}{Noise-sampler gauge invariance.}{app:factor_noise_gauge}
\AppRoadmapEntry{A16}{Proof of Thm.~\ref{thm:noise_energy}.}{app:proof_noise_energy}
\AppRoadmapEntry{A17}{Isotropy of projected Gaussian.}{app:isotropy}
\AppRoadmapEntry{A18}{Covariance and intrinsic dimension.}{app:noise_cov}
\AppRoadmapEntry{A19}{Noise concentration bounds.}{app:noise_concentration}
\AppRoadmapEntry{A20}{Retraction / rank-$r$ approximation.}{app:retraction}
\AppRoadmapEntry{A21}{No bilinear DP term in PRISM.}{app:no_second_order}
\AppRoadmapEntry{A22}{Proof of Thm.~\ref{thm:gauge_invariant}.}{app:gauge_invariant}
\AppRoadmapEntry{A23}{Gaussian mechanism on subspace.}{app:subspace_gaussian}
\AppRoadmapEntry{A24}{Procrustes alignment as gauge.}{app:procrustes}
\AppRoadmapEntry{A25}{DP guarantee details.}{app:dp_proof}
\AppRoadmapEntry{A26}{Rank-space noise moments.}{app:noise_moments}
\AppRoadmapEntry{A27}{Adaptive preconditioning analysis.}{app:precond}

\AppRoadmapHeader{Appendix B: Experimental setup.}{app:setup}
\AppRoadmapEntry{B1}{Datasets, hyperparameters, and compute.}{app:exp_details}

\AppRoadmapHeader{Appendix C: Additional diagnostics and analysis.}{app:analysis}
\AppRoadmapEntry{C1}{Issue I: gauge sensitivity diagnostics.}{app:pathology1_extra}
\AppRoadmapEntry{C2}{Issue II: gauge sweep at fixed $Z$.}{app:issue2_gauge_sweep}
\AppRoadmapEntry{C3}{Issue III: preconditioning/noise amplification diagnostics.}{app:path3}

\endgroup

\section{Theory and Proofs}
\label{app:theory}

This appendix provides proofs and additional derivations.
We use $\langle X,Y\rangle=\tr(X^\top Y)$, $\|\cdot\|_F$ for the Frobenius norm, $\vecop(\cdot)$ for vectorization, and $\otimes$ for the Kronecker product.
For a symmetric positive semidefinite matrix $X$, $X\pinv$ denotes the Moore--Penrose pseudoinverse.

\subsection{Quotient geometry: rank-$r$ matrices as a gauge quotient}
\label{app:quotient_geometry}

The non-identifiability $(A_{\ell},B_{\ell})\sim(A_{\ell}R,B_{\ell}R^{-\top})$ can be formalized as a smooth group action.
Let $\tilde{\calM}\triangleq \R^{m\times r}_*\times \R^{n\times r}_*$ be the \emph{total space} of full-column-rank factors and let $\gl(r)$ act on $\tilde{\calM}$ by
\begin{equation}
(A_{\ell},B_{\ell})\cdot R \;=\; (A_{\ell}R,\;B_{\ell}R^{-\top}),\qquad R\in\gl(r).
\label{eq:gauge_action}
\end{equation}
The map $\pi:\tilde{\calM}\to \calM_r$ given by $\pi(A_{\ell},B_{\ell})=A_{\ell}B_{\ell}^\top$ is constant along orbits of this action.
Under mild regularity conditions, the rank-$r$ manifold is the quotient $\calM_r \cong \tilde{\calM}/\gl(r)$, and \emph{intrinsic} quantities on $Z_{\ell}$ are precisely those that are invariant under \eqref{eq:gauge_action}.

\paragraph{Vertical and horizontal spaces.}
Differentiating the action \eqref{eq:gauge_action} at the identity $R=\Id_r$ yields the \emph{vertical space} (tangent to the gauge orbit) at $(A_{\ell},B_{\ell})$:
\begin{equation}
\calV_{(A_{\ell},B_{\ell})} \;=\; \{(A_{\ell}\Omega,\;-B_{\ell}\Omega^\top):\ \Omega\in\R^{r\times r}\}.
\label{eq:vertical}
\end{equation}
A complementary \emph{horizontal space} $\calH_{(A_{\ell},B_{\ell})}$ selects one representative lift of each intrinsic tangent direction.
With the Frobenius metric on $\tilde{\calM}$, the orthogonal complement of \eqref{eq:vertical} is characterized by
\begin{equation}
\calH_{(A_{\ell},B_{\ell})} \;=\; \Bigl\{(\Delta A_{\ell},\Delta B_{\ell}):\ A_{\ell}^\top \Delta A_{\ell} \;=\; (B_{\ell}^\top \Delta B_{\ell})^\top\Bigr\},
\label{eq:horizontal_simple}
\end{equation}
This is exactly the gauge-fixing constraint enforced by our canonical lift \eqref{eq:canonical}.

\begin{lemma}[Orthogonality to the gauge orbit]
\label{lem:vertical_horizontal}
A pair $(\Delta A_{\ell},\Delta B_{\ell})$ satisfies \eqref{eq:horizontal_simple} if and only if it is orthogonal to every vertical direction under the Frobenius inner product on $\tilde{\calM}$:
\[
\langle (\Delta A_{\ell},\Delta B_{\ell}),(A_{\ell}\Omega,-B_{\ell}\Omega^\top)\rangle
=\tr(\Delta A_{\ell}^\top A_{\ell}\Omega)-\tr(\Delta B_{\ell}^\top B_{\ell}\Omega^\top)=0,\qquad \forall\,\Omega.
\]
\end{lemma}

\begin{proof}
The stated inner product equals
$\tr(\Omega A_{\ell}^\top \Delta A_{\ell})-\tr(\Omega (B_{\ell}^\top \Delta B_{\ell})^\top)=\tr(\Omega (A_{\ell}^\top\Delta A_{\ell}-(B_{\ell}^\top\Delta B_{\ell})^\top))$.
Since this vanishes for all $\Omega$ if and only if $A_{\ell}^\top\Delta A_{\ell}=(B_{\ell}^\top\Delta B_{\ell})^\top$, we obtain \eqref{eq:horizontal_simple}.
\end{proof}

\paragraph{Why this matters for DP.}
Clipping and adding noise in factor coordinates implicitly chooses a metric on $\tilde{\calM}$, not on the quotient $\calM_r$.
As a result, the induced intrinsic perturbation can depend on the chosen representative (i.e., the gauge), which is the source of the amplification phenomena in Section~\ref{sec:problem}.
PRISM instead clips and perturbs \emph{horizontal} (gauge-orthogonal) directions and interprets the Gaussian mechanism intrinsically in the quotient geometry.
This is why the effective intrinsic noise $\mathcal{E}_{Z_{\ell}}$ is a fixed, controllable scalar in PRISM (Corollary~\ref{cor:gila_eff_noise}).

\begin{corollary}[Effective intrinsic DP noise of PRISM]
\label{cor:gila_eff_noise}
Let $\mathcal{N}_{Z_{\ell}}^{\text{PRISM}}=\frac{\sigma C}{b}P_{A_{\ell},B_{\ell}}(\Xi_{\ell})$ with $\Xi_{\ell}\sim\mathcal{N}(0,I_{m\times n})$.
Then $\mathcal{E}_{Z_{\ell}}=\frac{\sigma C}{b}\sqrt{r(m+n-r)}$, independent of the factor gauge.
\end{corollary}

\begin{proposition}[Orthogonal tangent projection]
\label{prop:tangentproj}
For $Z_{\ell}=A_{\ell}B_{\ell}^\top\in\mathcal{M}_r$, the linear map $P_{A_{\ell},B_{\ell}}$ in Eq.~\eqref{eq:tangent_proj} is the orthogonal projector onto $T_{Z_{\ell}}\mathcal{M}_r$:
for all $G\in\R^{m\times n}$, $P_{A_{\ell},B_{\ell}}(G)\in T_{Z_{\ell}}\mathcal{M}_r$ and $G-P_{A_{\ell},B_{\ell}}(G)\in (T_{Z_{\ell}}\mathcal{M}_r)^\perp$.
Equivalently, $P_{A_{\ell},B_{\ell}}$ is symmetric and idempotent.
\end{proposition}

\begin{proof}[Proof of Proposition~\ref{prop:tangentproj}]
By Lemma~\ref{lem:tangent_orth}, any $G\in\R^{m\times n}$ decomposes orthogonally as
\[
G = \underbrace{(G-\projperpA G \projperpB)}_{\in T_{Z_{\ell}}\calM_r}
\;+\;
\underbrace{\projperpA G\projperpB}_{\in (T_{Z_{\ell}}\calM_r)^\perp}.
\]
Therefore the orthogonal projector onto $T_{Z_{\ell}}\calM_r$ is
$P_{A_{\ell},B_{\ell}}(G)=G-\projperpA G\projperpB=\projA G + G\projB - \projA G\projB$.
\end{proof}

\subsection{General gauge amplification for arbitrary basis changes}
\label{app:general_gauge_amp}

Corollary~\ref{cor:unbounded} considered the scalar rescaling gauge $(A_{\ell},B_{\ell})\mapsto(cA_{\ell},c^{-1}B_{\ell})$.
Here we record the corresponding expression for a general invertible basis change.

\begin{proposition}[Gauge-dependent noise energy under general $R$]
\label{prop:general_R_amp}
Let $(A'_{\ell},B'_{\ell})=(A_{\ell}R,B_{\ell}R^{-\top})$ for invertible $R\in\gl(r)$ and let $\xi_{A,\ell},\xi_{B,\ell}$ be i.i.d.\ Gaussian as in Proposition~\ref{prop:amp}.
Then
\[
\E\bigl[\|\xi_{A,\ell} {B'_{\ell}}^\top\|_F^2\bigr] = m\tau^2\|B_{\ell}R^{-\top}\|_F^2,\qquad
\E\bigl[\|A'_{\ell}\xi_{B,\ell}^\top\|_F^2\bigr] = n\tau^2\|A_{\ell}R\|_F^2.
\]
Consequently, for any nonzero $Z_{\ell}=A_{\ell}B_{\ell}^\top$, there exist gauge matrices $R$ with arbitrarily large condition number such that the first-order noise energy on $Z_{\ell}$ becomes arbitrarily large, even though $Z_{\ell}$ is unchanged.
\end{proposition}

\begin{proof}
The expressions follow by the same calculation as Proposition~\ref{prop:amp}, with $B_{\ell}$ replaced by $B_{\ell}R^{-\top}$ and $A_{\ell}$ replaced by $A_{\ell}R$.
For the final statement, take an $R$ that scales one singular direction of $A_{\ell}$ (and inversely scales the corresponding direction of $B_{\ell}$) by a large factor; the scalar rescaling case is recovered by $R=c\Id$.
\end{proof}

\subsection{Second-order term magnitude and concentration}
\label{app:second_order_stats}

Proposition~\ref{prop:amp} quantifies the \emph{expected} Frobenius energy of the bilinear term $\mathcal{N}_{Z_{\ell}}^{(2)}=\xi_{A,\ell}\xi_{B,\ell}^\top$.
Here we record a simple high-probability bound that complements \eqref{eq:naive_second} and illustrates why the bilinear term can dominate when the learning rate is not extremely small.

\begin{lemma}[High-probability bound for the bilinear term]
\label{lem:bilinear_tail}
Let $\xi_{A,\ell}\in\R^{m\times r}$ and $\xi_{B,\ell}\in\R^{n\times r}$ have i.i.d.\ $\mathcal{N}(0,\tau^2)$ entries.
Then for any $\delta\in(0,1)$, with probability at least $1-2\delta$,
\begin{equation}
\|\xi_{A,\ell}\xi_{B,\ell}^\top\|_F
\;\le\;
\tau^2\sqrt{
\Bigl(mr + 2\sqrt{mr\log(1/\delta)} + 2\log(1/\delta)\Bigr)
\Bigl(nr + 2\sqrt{nr\log(1/\delta)} + 2\log(1/\delta)\Bigr)
}.
\label{eq:bilinear_tail}
\end{equation}
In particular, the typical scale is $\|\xi_{A,\ell}\xi_{B,\ell}^\top\|_F = \Theta(\tau^2\,r\sqrt{mn})$ up to logarithmic factors.
\end{lemma}

\begin{proof}
Use $\|XY^\top\|_F \le \|X\|_F\|Y\|_F$.
Since $\|\xi_{A,\ell}\|_F^2/\tau^2\sim \chi^2_{mr}$ and $\|\xi_{B,\ell}\|_F^2/\tau^2\sim \chi^2_{nr}$, applying the chi-square tail bound used in Proposition~\ref{prop:noise_tail} to each and taking a union bound yields \eqref{eq:bilinear_tail}.
\end{proof}

\paragraph{Implication for DP-LoRA.}
When both factors are randomized, the intrinsic parameter update contains the extra term $\eta^2\xi_{A,\ell}\xi_{B,\ell}^\top$ (Eq.~\eqref{eq:bilinear_noise_prob}).
Lemma~\ref{lem:bilinear_tail} implies this term has magnitude on the order of $\eta^2\tau^2 r\sqrt{mn}$, which is not controlled by the intrinsic clipping threshold on the first-order tangent update.
PRISM avoids this term altogether (Lemma~\ref{lem:no_second_order}).

\subsection{Noise amplification under naive factor-space DP}
\label{app:naive_noise}

\begin{proof}[Proof of Proposition~\ref{prop:amp}]
Consider the first-order intrinsic perturbation $\mathcal{N}_{Z_{\ell}}^{(1)}=\xi_{A,\ell} B_{\ell}^\top + A_{\ell}\xi_{B,\ell}^\top$.
Expanding the Frobenius norm,
\[
\|\mathcal{N}_{Z_{\ell}}^{(1)}\|_F^2=\|\xi_{A,\ell} B_{\ell}^\top\|_F^2+\|A_{\ell}\xi_{B,\ell}^\top\|_F^2+2\langle \xi_{A,\ell} B_{\ell}^\top, A_{\ell}\xi_{B,\ell}^\top\rangle.
\]
The cross term has zero expectation because $\xi_{A,\ell}$ and $\xi_{B,\ell}$ are independent and centered:
\[
\E\langle \xi_{A,\ell} B_{\ell}^\top, A_{\ell}\xi_{B,\ell}^\top\rangle
= \E\tr\!\bigl((\xi_{A,\ell} B_{\ell}^\top)^\top (A_{\ell}\xi_{B,\ell}^\top)\bigr)
= \E\tr(B_{\ell}\xi_{A,\ell}^\top A_{\ell}\xi_{B,\ell}^\top)
= \E_{\xi_{A,\ell}}\tr\!\bigl(B_{\ell}\xi_{A,\ell}^\top A_{\ell}\,\E_{\xi_{B,\ell}}[\xi_{B,\ell}^\top]\bigr)=0.
\]

For the remaining terms,
\[
\E\|\xi_{A,\ell} B_{\ell}^\top\|_F^2
=\E\tr(B_{\ell}\xi_{A,\ell}^\top \xi_{A,\ell} B_{\ell}^\top)
=\tr\bigl(B_{\ell}\,\E[\xi_{A,\ell}^\top\xi_{A,\ell}]\,B_{\ell}^\top\bigr).
\]
Since $\xi_{A,\ell}$ has i.i.d.\ $\mathcal{N}(0,\tau^2)$ entries, $\E[\xi_{A,\ell}^\top\xi_{A,\ell}]=m\tau^2\,\Id_r$, yielding $m\tau^2\|B_{\ell}\|_F^2$.
Similarly,
\[
\E\|A_{\ell}\xi_{B,\ell}^\top\|_F^2
=\E\tr(\xi_{B,\ell} A_{\ell}^\top A_{\ell} \xi_{B,\ell}^\top)
=\tr\bigl(A_{\ell}^\top A_{\ell}\,\E[\xi_{B,\ell}^\top\xi_{B,\ell}]\bigr)
=n\tau^2\|A_{\ell}\|_F^2,
\]
because $\E[\xi_{B,\ell}^\top\xi_{B,\ell}]=n\tau^2\,\Id_r$.
Summing proves \eqref{eq:naive_energy}.

For the bilinear term, each entry of $\xi_{A,\ell}\xi_{B,\ell}^\top$ is a sum of $r$ independent products of mean-zero Gaussians; its variance is $r\tau^4$.
Summing variances over $mn$ entries yields $\E\|\xi_{A,\ell}\xi_{B,\ell}^\top\|_F^2=mn r\,\tau^4$, proving \eqref{eq:naive_second}.
\end{proof}

\begin{proposition}[One-sided factor noise]
\label{prop:onesided_noise}
Let $\xi_{A,\ell}\in\R^{m\times r}$ and $\xi_{B,\ell}\in\R^{n\times r}$ have i.i.d.\ entries $\mathcal{N}(0,\tau^2)$.
For $Z_{\ell}=A_{\ell}B_{\ell}^\top$, consider the intrinsic perturbations obtained by noising only one factor:
$\mathcal{N}_{Z_{\ell}}^{(A)}\triangleq \xi_{A,\ell} B_{\ell}^\top$ and $\mathcal{N}_{Z_{\ell}}^{(B)}\triangleq A_{\ell}\xi_{B,\ell}^\top$.
Then
\begin{equation}
\label{eq:onesided_energy}
\E\|\mathcal{N}_{Z_{\ell}}^{(A)}\|_F^2 = \tau^2\,m\|B_{\ell}\|_F^2,
\qquad
\E\|\mathcal{N}_{Z_{\ell}}^{(B)}\|_F^2 = \tau^2\,n\|A_{\ell}\|_F^2.
\end{equation}
\end{proposition}

\begin{proof}[Proof of Proposition~\ref{prop:onesided_noise}]
If $A_{\ell}$ is frozen and only $B_{\ell}$ is perturbed by $\xi_{B,\ell}$ with i.i.d.\ $\mathcal{N}(0,\tau^2)$ entries, then the induced intrinsic perturbation is $\mathcal{N}_{Z_{\ell}}=A_{\ell}\xi_{B,\ell}^\top$.
Therefore
\[
\E\|\mathcal{N}_{Z_{\ell}}\|_F^2
=\E\tr(\xi_{B,\ell} A_{\ell}^\top A_{\ell} \xi_{B,\ell}^\top)
=\tr\bigl(A_{\ell}^\top A_{\ell}\,\E[\xi_{B,\ell}^\top\xi_{B,\ell}]\bigr)
=n\tau^2\|A_{\ell}\|_F^2,
\]
which is \eqref{eq:onesided_energy}.
If instead $B_{\ell}$ is frozen and only $A_{\ell}$ is perturbed, the same calculation gives $\E\|\mathcal{N}_{Z_{\ell}}\|_F^2=m\tau^2\|B_{\ell}\|_F^2$.
In either case there is no bilinear term because only one factor is randomized.
\end{proof}

\subsection{Range of effective intrinsic noise under scalar gauge rescaling}
\label{app:eff_noise_range}

Corollary~\ref{cor:unbounded} shows that naive factor-space DP can make the effective intrinsic noise arbitrarily large under the scalar gauge
$(A_{\ell},B_{\ell})\mapsto (cA_{\ell},c^{-1}B_{\ell})$.
For completeness, we record the \emph{entire range} of the first-order effective noise over this one-parameter family.

\begin{proposition}[Range over scalar gauges]
\label{prop:scalar_range}
Let $Z_{\ell}=A_{\ell}B_{\ell}^\top\neq 0$ and consider the scalar gauge family $(A_c,B_c)=(cA_{\ell},c^{-1}B_{\ell})$ for $c>0$.
Let $\xi_{A,\ell},\xi_{B,\ell}$ have i.i.d.\ $\mathcal{N}(0,\tau^2)$ entries as in Proposition~\ref{prop:amp}, and define the first-order perturbation
$\mathcal{N}_{Z,c}^{(1)}\triangleq \xi_{A,\ell} B_c^\top + A_c\xi_{B,\ell}^\top$.
Then
\begin{equation}
\E\bigl[\|\mathcal{N}_{Z,c}^{(1)}\|_F^2\bigr]
=
\tau^2\!\left(\frac{m}{c^2}\|B_{\ell}\|_F^2 + n c^2\|A_{\ell}\|_F^2\right).
\label{eq:scalar_range_energy}
\end{equation}
The minimizing gauge is
\begin{equation}
c^\star \;=\; \left(\frac{m\|B_{\ell}\|_F^2}{n\|A_{\ell}\|_F^2}\right)^{\!1/4},
\label{eq:c_star}
\end{equation}
and the minimum value is
\begin{equation}
\min_{c>0}\ \E\bigl[\|\mathcal{N}_{Z,c}^{(1)}\|_F^2\bigr]
=
2\tau^2\,\sqrt{mn}\,\|A_{\ell}\|_F\,\|B_{\ell}\|_F.
\label{eq:scalar_range_min}
\end{equation}
Moreover, $\sup_{c>0}\E[\|\mathcal{N}_{Z,c}^{(1)}\|_F^2]=\infty$.
\end{proposition}

\begin{proof}
Equation \eqref{eq:scalar_range_energy} follows directly from \eqref{eq:naive_energy} after substituting $A_c=cA_{\ell}$ and $B_c=c^{-1}B_{\ell}$.
The objective in $c$ is strictly convex in $\log c$ and differentiating \eqref{eq:scalar_range_energy} yields
$-2m\|B_{\ell}\|_F^2/c^3 + 2n c\|A_{\ell}\|_F^2=0$, giving \eqref{eq:c_star}.
Substituting $c^\star$ into \eqref{eq:scalar_range_energy} yields \eqref{eq:scalar_range_min}.
The divergence as $c\to 0$ or $c\to\infty$ gives the supremum.
\end{proof}

\paragraph{Interpretation.}
Even if one tunes the gauge once to reduce $\mathcal{E}_{Z_{\ell}}$, the optimizer can still drift to a different implicit scaling of $(A_{\ell},B_{\ell})$ over time.
Therefore, under factor-space DP the effective intrinsic noise is not a fixed function of the intrinsic parameter $Z_{\ell}$ and is not directly controlled by $(\sigma,C,b)$ alone.
PRISM removes this degree of freedom by defining the DP mechanism intrinsically on $T_{Z_{\ell}}\calM_r$ (Corollary~\ref{cor:gila_eff_noise}).

\subsection{Tangent space orthogonal complement and orthogonal projection}
\label{app:tangent}

\begin{lemma}[Tangent space characterization]
\label{lem:tangent_char}
Assume $A_{\ell}\in\R^{m\times r}$ and $B_{\ell}\in\R^{n\times r}$ have full column rank and $Z_{\ell}=A_{\ell}B_{\ell}^\top$.
Then the tangent space of $\calM_r$ at $Z_{\ell}$ is
\begin{equation}
T_{Z_{\ell}}\calM_r
\;=\;
\{\Delta A_{\ell} B_{\ell}^\top + A_{\ell}\Delta B_{\ell}^\top:\ \Delta A_{\ell}\in\R^{m\times r},\ \Delta B_{\ell}\in\R^{n\times r}\},
\label{eq:tangent_char}
\end{equation}
which is Eq.~\eqref{eq:tangent} in the main text.
\end{lemma}

\begin{proof}
Consider the smooth factorization map $\phi:\R^{m\times r}_*\times \R^{n\times r}_*\to\calM_r$ defined by $\phi(A,B)=AB^\top$.
For any perturbations $(\Delta A,\Delta B)$, the directional derivative at $(A_{\ell},B_{\ell})$ is
\[
D\phi_{(A_{\ell},B_{\ell})}[\Delta A,\Delta B] = \Delta A\,B_{\ell}^\top + A_{\ell}\,\Delta B^\top.
\]
Thus every matrix of the form $\Delta A_{\ell}B_{\ell}^\top + A_{\ell}\Delta B_{\ell}^\top$ arises as the derivative of the curve
$t\mapsto (A_{\ell}+t\Delta A_{\ell})(B_{\ell}+t\Delta B_{\ell})^\top$ at $t=0$, and hence lies in $T_{Z_{\ell}}\calM_r$.

Conversely, any tangent vector $\Delta Z\in T_{Z_{\ell}}\calM_r$ is the derivative of a smooth curve $t\mapsto Z(t)\in\calM_r$ with $Z(0)=Z_{\ell}$.
Because $A_{\ell}$ and $B_{\ell}$ are full column rank, rank-$r$ matrices near $Z_{\ell}$ admit factorizations $Z(t)=A(t)B(t)^\top$ with $A(t),B(t)$ smooth in $t$.
Differentiating at $t=0$ yields $\Delta Z = \dot A(0)B_{\ell}^\top + A_{\ell}\dot B(0)^\top$, proving \eqref{eq:tangent_char}.
\end{proof}

\begin{lemma}[Tangent space orthogonal complement]
\label{lem:tangent_orth}
Let $Z_{\ell}=A_{\ell}B_{\ell}^\top$ with $\projA,\projB$ as in \eqref{eq:projectors}.
Then
\begin{equation}
(T_{Z_{\ell}}\calM_r)^\perp
= \{X\in\R^{m\times n} : \projA X = 0 \text{ and } X\projB = 0\}
= \{\projperpA\,Y\,\projperpB:Y\in\R^{m\times n}\}.
\label{eq:orthcomp}
\end{equation}
\end{lemma}

\begin{proof}
Let $X\in(T_{Z_{\ell}}\calM_r)^\perp$.
For all $\Delta A_{\ell},\Delta B_{\ell}$,
\[
0=\langle X, \Delta A_{\ell} B_{\ell}^\top + A_{\ell}\Delta B_{\ell}^\top\rangle
= \langle XB_{\ell}, \Delta A_{\ell}\rangle + \langle X^\top A_{\ell}, \Delta B_{\ell}\rangle.
\]
Hence $XB_{\ell}=0$ and $X^\top A_{\ell}=0$.
Since $\projB$ projects onto $\mathrm{col}(B_{\ell})$, $XB_{\ell}=0$ is equivalent to $X\projB=0$; similarly $X^\top A_{\ell}=0$ is equivalent to $\projA X=0$.
Conversely, if $\projA X=0$ and $X\projB=0$, the inner product above vanishes for all $\Delta A_{\ell},\Delta B_{\ell}$.
Finally, $X\projB=0$ implies $X=X\projperpB$ and $\projA X=0$ implies $X=\projperpA X$, hence $X=\projperpA Y\projperpB$ for $Y=X$.
\end{proof}

\subsection{Gauge freedom in factor lifts and minimum-factor-norm representatives}
\label{app:lift_min}

The representation of an intrinsic tangent update $\Delta Z_{\ell}\in T_{Z_{\ell}}\calM_r$ in factor space is not unique.
This non-uniqueness is the differential analogue of the gauge symmetry $(A_{\ell},B_{\ell})\sim(A_{\ell}R,B_{\ell}R^{-\top})$.
We record the basic ``lift gauge'' degrees of freedom and a canonical choice based on minimum factor norm.

\begin{lemma}[Gauge degrees of freedom in factor lifts]
\label{lem:gauge_lift}
Let $Z_{\ell}=A_{\ell}B_{\ell}^\top$ and fix an intrinsic tangent matrix $\Delta Z_{\ell}\in T_{Z_{\ell}}\calM_r$.
If $(\Delta A_{\ell},\Delta B_{\ell})$ is any pair satisfying $\Delta Z_{\ell}=\Delta A_{\ell} B_{\ell}^\top + A_{\ell}\Delta B_{\ell}^\top$,
then for any $\Omega\in\R^{r\times r}$,
\begin{equation}
(\Delta A_{\ell},\Delta B_{\ell})\ \mapsto\ (\Delta A_{\ell} + A_{\ell}\Omega,\ \Delta B_{\ell} - B_{\ell}\Omega^\top)
\label{eq:gauge_lift}
\end{equation}
produces another valid lift of the \emph{same} intrinsic update $\Delta Z_{\ell}$.
In particular, the induced matrix $\Delta Z_{\ell}$ depends only on the intrinsic point $Z_{\ell}$ and not on the chosen factor lift.
\end{lemma}

\begin{proof}
For any $\Omega$,
\[
(\Delta A_{\ell} + A_{\ell}\Omega)B_{\ell}^\top + A_{\ell}(\Delta B_{\ell} - B_{\ell}\Omega^\top)^\top
= \Delta A_{\ell} B_{\ell}^\top + A_{\ell}\Delta B_{\ell}^\top + A_{\ell}\Omega B_{\ell}^\top - A_{\ell}\Omega B_{\ell}^\top
= \Delta Z_{\ell}.
\]
\end{proof}

Equation \eqref{eq:gauge_lift} shows that there are infinitely many factor pairs that realize the \emph{same} intrinsic tangent update.
This is the differential version of the gauge symmetry \eqref{eq:gauge_action}.
Depending on the optimizer implementation, it can be useful to pick a canonical representative of this equivalence class.
One natural choice is the \emph{minimum-factor-norm} lift, which can be computed by solving a small Sylvester equation of size $r\times r$.

\begin{proposition}[Minimum-norm representative in a gauge class]
\label{prop:min_norm_lift}
Assume $A_{\ell}$ and $B_{\ell}$ have full column rank so that $M=A_{\ell}^\top A_{\ell}\succ 0$ and $N=B_{\ell}^\top B_{\ell}\succ 0$.
Fix any lift $(\Delta A_0,\Delta B_0)$ that realizes a given $\Delta Z_{\ell}$.
Consider the gauge family
\[
\Delta A_{\ell}(\Omega)=\Delta A_0 + A_{\ell}\Omega,\qquad \Delta B_{\ell}(\Omega)=\Delta B_0 - B_{\ell}\Omega^\top.
\]
Then the unique minimizer of
$f(\Omega)\triangleq \|\Delta A_{\ell}(\Omega)\|_F^2+\|\Delta B_{\ell}(\Omega)\|_F^2$
is obtained by the unique solution $\Omega^\star$ to the Sylvester equation
\begin{equation}
M\Omega + \Omega N \;=\; B_{\ell}^\top \Delta B_0 - A_{\ell}^\top \Delta A_0.
\label{eq:sylvester}
\end{equation}
The corresponding $(\Delta A_{\ell}(\Omega^\star),\Delta B_{\ell}(\Omega^\star))$ is the minimum-factor-norm lift of $\Delta Z_{\ell}$.
\end{proposition}

\begin{proof}
Expand the quadratic objective:
\begin{align*}
f(\Omega)
&= \|\Delta A_0 + A_{\ell}\Omega\|_F^2 + \|\Delta B_0 - B_{\ell}\Omega^\top\|_F^2 \\
&= \|\Delta A_0\|_F^2 + \|\Delta B_0\|_F^2
+ 2\,\tr(\Omega^\top A_{\ell}^\top \Delta A_0) - 2\,\tr(\Omega^\top \Delta B_0^\top B_{\ell}) \\
&\quad + \tr(\Omega^\top M \Omega) + \tr(\Omega N \Omega^\top),
\end{align*}
where we used $\|A_{\ell}\Omega\|_F^2=\tr(\Omega^\top M\Omega)$ and $\|B_{\ell}\Omega^\top\|_F^2=\tr(\Omega N \Omega^\top)$.
Taking the derivative and setting it to zero gives the first-order optimality condition
\[
M\Omega + \Omega N = B_{\ell}^\top \Delta B_0 - A_{\ell}^\top \Delta A_0,
\]
which is \eqref{eq:sylvester}.
Since $M$ and $N$ are positive definite, $f$ is strictly convex in $\Omega$ and the Sylvester equation has a unique solution, hence the minimizer is unique.
\end{proof}

\paragraph{Relation to ``horizontal'' conditions.}
The optimality condition \eqref{eq:sylvester} is equivalent to orthogonality of the minimizer to the gauge (kernel) directions of the map $(\Delta A_{\ell},\Delta B_{\ell})\mapsto \Delta A_{\ell} B_{\ell}^\top + A_{\ell}\Delta B_{\ell}^\top$ under the Euclidean metric on factors.
This is the standard minimum-norm property of least-squares solutions and is closely related to horizontal lifts in quotient-geometry treatments \cite{mishra2014fixedrank}.

\subsection{Gauge invariance of subspace projectors}
\label{app:proj}

\begin{lemma}[Projectors are gauge invariant]
\label{lem:proj_gauge}
Let $(A'_{\ell},B'_{\ell})=(A_{\ell}R,B_{\ell}R^{-\top})$ for any invertible $R\in\R^{r\times r}$.
Then $\Pi_{A'_{\ell}}=\Pi_{A_{\ell}}$ and $\Pi_{B'_{\ell}}=\Pi_{B_{\ell}}$.
\end{lemma}

\begin{proof}
We prove the statement for $\Pi_{A_{\ell}}$; the argument for $\Pi_{B_{\ell}}$ is identical.
Let $A'_{\ell}=A_{\ell}R$. Then ${A'_{\ell}}^\top {A'_{\ell}} = R^\top ({A_{\ell}}^\top A_{\ell}) R$.
By Lemma~\ref{lem:pseudoinv_congruence},
\[
(A'_{\ell}{}^{\top} A'_{\ell})^{\pinv}
= \bigl(R^{\top} (A_{\ell}^{\top} A_{\ell}) R\bigr)^{\pinv}
= R^{-1} (A_{\ell}^{\top} A_{\ell})^{\pinv} R^{-\top}.
\]

Therefore
\begin{align*}
\Pi_{A'_{\ell}}
&= A'_{\ell}\bigl({A'_{\ell}}^\top A'_{\ell}\bigr)^{\pinv} {A'_{\ell}}^\top \\
\phantom{\Pi_{A'_{\ell}}}
&= A_{\ell}R \Bigl(R^{-1} (A_{\ell}^\top A_{\ell})^{\pinv} R^{-\top}\Bigr) R^\top A_{\ell}^\top \\
\phantom{\Pi_{A'_{\ell}}}
&= A_{\ell}(A_{\ell}^\top A_{\ell})^{\pinv} A_{\ell}^\top \\
\phantom{\Pi_{A'_{\ell}}}
&= \Pi_{A_{\ell}}.
\end{align*}
\end{proof}

\begin{lemma}[Pseudoinverse under congruence]
\label{lem:pseudoinv_congruence}
Let $X\succeq 0$ and let $R$ be invertible.
Then $(R^\top X R)\pinv = R^{-1} X\pinv R^{-\top}$.
\end{lemma}

\begin{proof}
Let $Y=R^\top X R$ and define $Y^*\triangleq R^{-1}X\pinv R^{-\top}$.
We verify the Moore--Penrose conditions:
$YY^*Y=Y$, $Y^*YY^*=Y^*$, and both $YY^*$ and $Y^*Y$ are symmetric.
All follow from substituting the definition of $Y$ and $Y^*$, using $RR^{-1}=\Id$ and the Moore--Penrose conditions for $X$ and $X\pinv$.
\end{proof}

\subsection{A canonical factor lift of the tangent projection}
\label{app:lifts}

This subsection verifies that the explicit lift defined in Eq.~\eqref{eq:canonical} indeed reproduces the orthogonal tangent projection \eqref{eq:tangent_proj}.
Non-uniqueness of factor lifts (and a minimum-norm choice) is discussed separately in Appendix~\ref{app:lift_min}.

\begin{lemma}[A convenient factor lift of the tangent projection]
\label{lem:lift}
Let $Z_{\ell}=A_{\ell}B_{\ell}^\top$ and let $G_{i,\ell}\triangleq\nabla_{Z_{\ell}}\ell_i$ with factor gradients $g_{A,i,\ell}=G_{i,\ell}B_{\ell}$ and $g_{B,i,\ell}=G_{i,\ell}^\top A_{\ell}$.
Define $\Delta A_{i,\ell},\Delta B_{i,\ell}$ by \eqref{eq:canonical}.
Then the induced matrix update
\begin{equation}
\Delta Z_{i,\ell} \triangleq \Delta A_{i,\ell} B_{\ell}^\top + A_{\ell}\Delta B_{i,\ell}^\top
\end{equation}
equals the tangent projection $P_{A_{\ell},B_{\ell}}(G_{i,\ell})$.
\end{lemma}

\begin{proof}
\textbf{Step 1 (matching the matrix update).}
Using $B_{\ell}N\pinv B_{\ell}^\top=\projB$ and $A_{\ell}M\pinv A_{\ell}^\top=\projA$,
\[
g_{A,i,\ell} N\pinv = G_{i,\ell}(B_{\ell}N\pinv)=G_{i,\ell}\projB,\qquad g_{B,i,\ell} M\pinv = G_{i,\ell}^\top(A_{\ell}M\pinv)=G_{i,\ell}^\top\projA.
\]
Hence
\begin{align*}
\Delta A_{i,\ell} B_{\ell}^\top
&= \left(G_{i,\ell}\projB - \tfrac{1}{2}\projA(G_{i,\ell}\projB)\right)B_{\ell}^\top
= G_{i,\ell}\projB - \tfrac{1}{2}\projA G_{i,\ell}\projB,\\
A_{\ell}\Delta B_{i,\ell}^\top
&= A_{\ell}\left(G_{i,\ell}^\top \projA - \tfrac{1}{2}\projB(G_{i,\ell}^\top \projA)\right)^\top
= \projA G_{i,\ell} - \tfrac{1}{2}\projA G_{i,\ell}\projB.
\end{align*}
Summing yields $\Delta Z_{i,\ell}=\projA G_{i,\ell} + G_{i,\ell}\projB - \projA G_{i,\ell}\projB=P_{A_{\ell},B_{\ell}}(G_{i,\ell})$.
\end{proof}

\subsection{Frobenius norm formula for tangent updates}
\label{app:norm}

\begin{proposition}[Frobenius norm of a factorized update]
\label{prop:norm_formula}
Let $Z_{\ell}=A_{\ell}B_{\ell}^\top$ and let $\Delta Z_{\ell}=\Delta A_{\ell}\,B_{\ell}^\top + A_{\ell}\,\Delta B_{\ell}^\top$ for arbitrary $\Delta A_{\ell}\in\R^{m\times r}$ and $\Delta B_{\ell}\in\R^{n\times r}$.
With $M=A_{\ell}^\top A_{\ell}$ and $N=B_{\ell}^\top B_{\ell}$,
\begin{equation}
\label{eq:norm_general}
\|\Delta Z_{\ell}\|_F^2
=
\tr(\Delta A_{\ell}^\top \Delta A_{\ell}\,N)
+
\tr(\Delta B_{\ell}^\top \Delta B_{\ell}\,M)
+
2\,\tr\!\bigl((A_{\ell}^\top \Delta A_{\ell})(B_{\ell}^\top \Delta B_{\ell})\bigr).
\end{equation}
\end{proposition}

\begin{proof}[Proof of Proposition~\ref{prop:norm_formula}]
Expand $\|\Delta Z_{\ell}\|_F^2=\langle \Delta A_{\ell} B_{\ell}^\top + A_{\ell}\Delta B_{\ell}^\top,\Delta A_{\ell} B_{\ell}^\top + A_{\ell}\Delta B_{\ell}^\top\rangle$ and use cyclicity of trace:
\[
\|\Delta A_{\ell} B_{\ell}^\top\|_F^2=\tr(\Delta A_{\ell}^\top \Delta A_{\ell}\,B_{\ell}^\top B_{\ell})=\tr(\Delta A_{\ell}^\top \Delta A_{\ell}\,N),
\quad
\|A_{\ell}\Delta B_{\ell}^\top\|_F^2=\tr(\Delta B_{\ell}^\top \Delta B_{\ell}\,A_{\ell}^\top A_{\ell})=\tr(\Delta B_{\ell}^\top \Delta B_{\ell}\,M),
\]
and
$\langle \Delta A_{\ell} B_{\ell}^\top, A_{\ell}\Delta B_{\ell}^\top\rangle=\tr((A_{\ell}^\top\Delta A_{\ell})(B_{\ell}^\top\Delta B_{\ell}))$.
Summing yields \eqref{eq:norm_general}.
\end{proof}

\subsection{Specialization to rank-1 per-example gradients}
\label{app:norm_rank1}

\begin{lemma}[Norm for rank-1 gradients]
\label{lem:norm_rank1}
Let $G_{i,\ell}=u v^\top$ and let $\Delta Z_{i,\ell} = P_{A_{\ell},B_{\ell}}(G_{i,\ell})$.
Define $\widehat u=\projA u$ and $\widehat v=\projB v$.
Then
\begin{equation}
\|\Delta Z_{i,\ell}\|_F^2
= \|\widehat u\|_2^2\,\|v\|_2^2 + \|u\|_2^2\,\|\widehat v\|_2^2 - \|\widehat u\|_2^2\,\|\widehat v\|_2^2.
\label{eq:norm_rank1}
\end{equation}
\end{lemma}

\begin{proof}
Use $\Delta Z_{i,\ell}=\widehat u v^\top + u\widehat v^\top - \widehat u\widehat v^\top$ and the identities
$\|ab^\top\|_F^2=\|a\|_2^2\|b\|_2^2$ and $\langle ab^\top,cd^\top\rangle=(a^\top c)(b^\top d)$.
Expanding and simplifying yields \eqref{eq:norm_rank1}.
\end{proof}

\subsection{Factorized tangent noise equals a projected dense Gaussian}
\label{app:lowrank}

For convenience, define the \emph{whitened} factors
\begin{equation}
\label{eq:whiten}
\widehat A_{\ell} \triangleq A_{\ell} M^{-1/2},\qquad
\widehat B_{\ell} \triangleq B_{\ell} N^{-1/2},
\end{equation}
so that $\widehat A_{\ell}^\top \widehat A_{\ell}=\widehat B_{\ell}^\top \widehat B_{\ell}=I_r$ and $\projA=\widehat A_{\ell}\widehat A_{\ell}^\top$, $\projB=\widehat B_{\ell}\widehat B_{\ell}^\top$ when $A_{\ell},B_{\ell}$ have full column rank.

\begin{lemma}[Distributional equivalence of the sampler]
\label{lem:lowrank_equiv}
Let $U\in\R^{m\times r}$ and $V\in\R^{n\times r}$ have i.i.d.\ standard normal entries and define $\widehat A_{\ell},\widehat B_{\ell}$ as in \eqref{eq:whiten}.
Then
\[
(\Id-\projA)\,U\,\widehat B_{\ell}^\top + \widehat A_{\ell}\,V^\top
\ \overset{d}{=}\ P_{A_{\ell},B_{\ell}}(\Xi_{\ell}),
\]
where $\Xi_{\ell}$ has i.i.d.\ $\mathcal{N}(0,1)$ entries.
\end{lemma}

\begin{proof}
We show that $\vecop(\Delta Z_{\mathrm{noise}})$ is a zero-mean Gaussian with covariance equal to \eqref{eq:cov_kronecker}.
Using the identity $\vecop(XY^\top)=(Y\otimes \Id)\vecop(X)$, we have
\begin{align*}
\vecop\bigl((\Id-\projA)U\widehat B_{\ell}^\top\bigr)
&= (\widehat B_{\ell}\otimes (\Id-\projA))\,\vecop(U),\\
\vecop(\widehat A_{\ell} V^\top)
&= (\Id\otimes \widehat A_{\ell})\,\vecop(V).
\end{align*}
Since $\vecop(U)$ and $\vecop(V)$ are independent standard Gaussians, their images under fixed linear maps are independent Gaussians and the covariances add:
\begin{align*}
\mathrm{Cov}\bigl[\vecop(\Delta Z_{\mathrm{noise}})\bigr]
&= (\widehat B_{\ell}\widehat B_{\ell}^\top \otimes (\Id-\projA))\;+\;(\Id\otimes \widehat A_{\ell}\widehat A_{\ell}^\top).
\end{align*}
Substituting $\widehat B_{\ell}\widehat B_{\ell}^\top=\projB$ and $\widehat A_{\ell}\widehat A_{\ell}^\top=\projA$ yields
\[
\mathrm{Cov}\bigl[\vecop(\Delta Z_{\mathrm{noise}})\bigr]
= (\projB\otimes \Id) - (\projB\otimes \projA) + (\Id\otimes \projA),
\]
which matches \eqref{eq:cov_kronecker}.
Therefore $\Delta Z_{\mathrm{noise}} \overset{d}{=} P_{A_{\ell},B_{\ell}}(\Xi_{\ell})$.
\end{proof}

\subsection{Low-dimensional noise sampler}
\label{app:noise_sampler}

The intrinsic PRISM noise for module $\ell$ is $N_{Z_{\ell}}=\tau\,\mathcal{P}_{A_{\ell},B_{\ell}}(\Xi_{\ell})$ with $\Xi_{\ell}\sim\mathcal{N}(0,I_{m\times n})$ and $\tau=\sigma C/b$ (Theorem~\ref{thm:noise_energy}).
Directly sampling the dense matrix $\Xi_{\ell}$ is unnecessary.

Let $U\in\R^{m\times r}$ and $V\in\R^{n\times r}$ have i.i.d.\ $\mathcal{N}(0,1)$ entries, and define
\[
\Xi_{A,\ell}=(I-\projA)\,U\,N^{-1/2},\qquad \Xi_{B,\ell}=V\,M^{-1/2},
\]
where $M=A_{\ell}^\top A_{\ell}$ and $N=B_{\ell}^\top B_{\ell}$.
Then the intrinsic perturbation induced by these factor noises is
\[
\Xi_{A,\ell}B_{\ell}^\top + A_{\ell}\Xi_{B,\ell}^\top,
\]
which matches the low-dimensional sampler in Eq.~\eqref{eq:noise_form} (with $(U,V)$ corresponding to $(\Omega_{A,\ell},\Omega_{B,\ell})$).
By Lemma~\ref{lem:lowrank_equiv} (Appendix~\ref{app:lowrank}), this intrinsic perturbation is distributed exactly as $\mathcal{P}_{A_{\ell},B_{\ell}}(\Xi_{\ell})$.

Computationally, this reduces random number generation from $O(mn)$ to $O((m+n)r)$ per module.
Stable computation of $\projA,\projB$ and the inverse square roots $M^{-1/2},N^{-1/2}$ is discussed in Appendix~\ref{app:conditioning}.
The resulting intrinsic noise distribution is also invariant to the choice of factorization, as formalized in Appendix~\ref{app:factor_noise_gauge}.

\subsection{Stable computation of projectors and orthonormal bases}
\label{app:conditioning}

For numerical stability, it is often preferable to compute the projectors $\projA,\projB$ and orthonormal bases of $\mathrm{col}(A_{\ell})$ and $\mathrm{col}(B_{\ell})$ without forming Gram-matrix inverses explicitly.
We record simple equivalences.

\begin{lemma}[Projector from a thin QR factorization]
\label{lem:qr_projector}
Assume $A_{\ell}\in\R^{m\times r}$ has full column rank and let $A_{\ell}=Q_AR_A$ be a thin QR factorization with $Q_A^\top Q_A=\Id_r$ and $R_A$ invertible.
Then the orthogonal projector onto $\mathrm{col}(A_{\ell})$ satisfies $\projA=Q_AQ_A^\top$.
An analogous statement holds for $B_{\ell}$.
\end{lemma}

\begin{proof}
Since $A_{\ell}^\top A_{\ell} = R_A^\top R_A$, we have
\[
\projA = A_{\ell}(A_{\ell}^\top A_{\ell})^{-1}A_{\ell}^\top
= Q_AR_A(R_A^\top R_A)^{-1}R_A^\top Q_A^\top
= Q_AQ_A^\top,
\]
because $R_A(R_A^\top R_A)^{-1}R_A^\top=\Id_r$ when $R_A$ is invertible.
\end{proof}

\begin{lemma}[Noise sampling is invariant to the choice of orthonormal bases]
\label{lem:basis_invariant_noise}
Let $Q_A,Q'_A\in\R^{m\times r}$ be orthonormal bases of the same subspace $\mathrm{col}(A_{\ell})$ and let $Q_B,Q'_B\in\R^{n\times r}$ be orthonormal bases of $\mathrm{col}(B_{\ell})$.
Define $\projA=Q_AQ_A^\top=Q'_AQ_A'^\top$ and $\projB=Q_BQ_B^\top=Q'_BQ_B'^\top$.
If $U\in\R^{m\times r}$ and $V\in\R^{n\times r}$ have i.i.d.\ $\mathcal{N}(0,1)$ entries, then
\[
(\Id-\projA)\,U\,Q_B^\top + Q_A V^\top
\ \overset{d}{=}\ 
(\Id-\projA)\,U\,Q_B'^\top + Q'_A V^\top.
\]
\end{lemma}

\begin{proof}
Because $Q_A$ and $Q'_A$ are orthonormal bases of the same subspace, there exists an orthogonal matrix $O_A\in\R^{r\times r}$ such that $Q'_A=Q_AO_A$.
Similarly, $Q'_B=Q_BO_B$ for some orthogonal $O_B$.
Then
\[
(\Id-\projA)UQ_B'^\top + Q'_A V^\top
= (\Id-\projA)UO_B^\top Q_B^\top + Q_AO_A V^\top.
\]
Since $U$ and $V$ are i.i.d.\ standard Gaussian matrices, $UO_B^\top \overset{d}{=} U$ and $O_A V \overset{d}{=} V$.
The claim follows.
\end{proof}

\paragraph{Regularization.}
When Gram matrices are nearly singular, one can obtain stable orthonormal bases via QR/SVD and use Lemma~\ref{lem:qr_projector} to form $\projA,\projB$.
If one instead regularizes Gram inverses directly (e.g., spectral truncation or damping), the resulting operators are still deterministic functions of $(A_{\ell},B_{\ell})$ and therefore DP-safe by post-processing, but may not preserve \emph{exact} gauge invariance unless the regularization is itself defined in a gauge-consistent way.

\subsection{Gauge invariance of the factorized noise sampler}
\label{app:factor_noise_gauge}

We recall the factorized sampler used to generate the tangent noise lifts in Eq.~\eqref{eq:dpnoise}:
\begin{equation}
\label{eq:factor_noise}
\Delta A_{\mathrm{noise}}=(I-\projA)\,U\,N^{-1/2},
\qquad
\Delta B_{\mathrm{noise}}=V\,M^{-1/2},
\end{equation}
where $U\in\R^{m\times r}$ and $V\in\R^{n\times r}$ have i.i.d.\ $\mathcal{N}(0,1)$ entries, and $M=A_{\ell}^\top A_{\ell}$, $N=B_{\ell}^\top B_{\ell}$.
The induced intrinsic perturbation is $\Delta Z_{\mathrm{noise}}=\Delta A_{\mathrm{noise}}B_{\ell}^\top + A_{\ell}\Delta B_{\mathrm{noise}}^\top$.
\label{app:sampler_gauge}

The intrinsic mechanism \eqref{eq:dpnoise} is gauge invariant by construction.
Here we make explicit that the \emph{implementation-level} sampler \eqref{eq:factor_noise} inherits the same invariance: regardless of which factorization of $Z_{\ell}$ is used internally, the induced distribution on the intrinsic noise $\Delta Z_{\mathrm{noise}}$ is unchanged.

\begin{proposition}[Sampler invariance under gauge transforms]
\label{prop:sampler_gauge}
Let $(A'_{\ell},B'_{\ell})=(A_{\ell}R,B_{\ell}R^{-\top})$ for some $R\in\gl(r)$.
Construct $\Delta A_{\mathrm{noise}}$ and $\Delta B_{\mathrm{noise}}$ from $(A_{\ell},B_{\ell})$ via \eqref{eq:factor_noise}, and construct $\Delta A'_{\ell,\mathrm{noise}}$ and $\Delta B'_{\ell,\mathrm{noise}}$ from $(A'_{\ell},B'_{\ell})$ via the same formula (with the corresponding projectors and Gram matrices).
Then the induced intrinsic noises
\[
\Delta Z_{\mathrm{noise}}
= \Delta A_{\mathrm{noise}} B_{\ell}^\top + A_{\ell} (\Delta B_{\mathrm{noise}})^\top,
\qquad
\Delta Z'_{\mathrm{noise}}
= \Delta A'_{\ell,\mathrm{noise}} (B'_{\ell})^\top + A'_{\ell} (\Delta B'_{\ell,\mathrm{noise}})^\top.
\]
have the same distribution.
\end{proposition}

\begin{proof}
By Lemma~\ref{lem:lowrank_equiv}, both $\Delta Z_{\mathrm{noise}}$ and $\Delta Z'_{\mathrm{noise}}$ are distributed as $P_{A_{\ell},B_{\ell}}(\Xi_{\ell})$ and $P_{A',B'}(\Xi_{\ell})$, respectively, for a dense standard Gaussian $\Xi_{\ell}$.
By Lemma~\ref{lem:proj_gauge}, the subspace projectors are gauge invariant and thus $P_{A',B'}=P_{A_{\ell},B_{\ell}}$.
Therefore $P_{A',B'}(\Xi_{\ell})$ and $P_{A_{\ell},B_{\ell}}(\Xi_{\ell})$ have identical distributions.
\end{proof}

\paragraph{Contrast with naive factor noise.}
If one instead adds i.i.d.\ Gaussian noise directly to $\Delta A_{\ell}$ and $\Delta B_{\ell}$ without the whitening and projection factors in \eqref{eq:factor_noise}, the induced intrinsic perturbation depends on the chosen gauge through $\|A_{\ell}\|_F$ and $\|B_{\ell}\|_F$ (Proposition~\ref{prop:amp}).
The role of $M^{-1/2}$ and $N^{-1/2}$ in \eqref{eq:factor_noise} is precisely to compensate for this coordinate dependence and yield an isotropic Gaussian in the intrinsic tangent space.

\subsection{Proof of Theorem~\ref{thm:noise_energy}}
\label{app:proof_noise_energy}

\begin{proof}[Proof of Theorem~\ref{thm:noise_energy}]
Let $\Xi_{\ell}\sim\mathcal{N}(0,I_{m\times n})$ be a dense standard Gaussian and write
$N_{Z_{\ell}}=\tau\,\mathcal{P}_{A_{\ell},B_{\ell}}(\Xi_{\ell})$ with $\tau=\sigma C/b$.

\textbf{Gaussianity and support.}
Vectorizing gives $\vecop(\Xi_{\ell})\sim\mathcal{N}(0,I_{mn})$ and
$\vecop(\mathcal{P}_{A_{\ell},B_{\ell}}(\Xi_{\ell}))=\mathbf{P}\,\vecop(\Xi_{\ell})$,
where $\mathbf{P}$ is the matrix representation of the orthogonal projector $\mathcal{P}_{A_{\ell},B_{\ell}}$ under $\vecop(\cdot)$.
Since $\mathbf{P}$ is linear, symmetric, and idempotent, $\mathbf{P}\,\vecop(\Xi_{\ell})$ is Gaussian with covariance $\mathbf{P}$ and is supported on $\mathrm{range}(\mathbf{P})$,
which corresponds to the tangent subspace $T_{Z_{\ell}}\calM_r$ (Eq.~\eqref{eq:tangent_proj}).

\textbf{Isotropy on the tangent space.}
Because the covariance equals the orthogonal projector onto $T_{Z_{\ell}}\calM_r$, the distribution is isotropic within that subspace; a self-contained verification is given in Lemma~\ref{lem:isotropy} (Appendix~\ref{app:isotropy}).

\textbf{Expected energy.}
Using $\|X\|_F^2=\|\vecop(X)\|_2^2$ and $\mathbb{E}\|G\|_2^2=\mathrm{tr}(\mathrm{Cov}[G])$ for a zero-mean Gaussian vector $G$, we obtain
\[
\mathbb{E}\bigl\|\mathcal{P}_{A_{\ell},B_{\ell}}(\Xi_{\ell})\bigr\|_F^2
=\mathbb{E}\bigl\|\mathbf{P}\vecop(\Xi_{\ell})\bigr\|_2^2
=\mathrm{tr}(\mathbf{P})
=\rank(\mathbf{P})
=\dim(T_{Z_{\ell}}\calM_r)
=r(m+n-r),
\]
which is Eq.~\eqref{eq:energy_gila}.
Multiplying by $\tau^2$ yields $\mathbb{E}\|N_{Z_{\ell}}\|_F^2=\tau^2 r(m+n-r)$ and thus Eq.~\eqref{eq:eff_prism}.

\textbf{Gauge invariance.}
Finally, $\mathcal{P}_{A_{\ell},B_{\ell}}$ depends only on $\projA$ and $\projB$ (Eq.~\eqref{eq:tangent_proj}), and these projectors are invariant under the gauge transform $(A_{\ell},B_{\ell})\mapsto(A_{\ell}R,B_{\ell}R^{-\top})$ by Lemma~\ref{lem:proj_gauge} (Appendix~\ref{app:proj}).
\end{proof}

\subsection{Isotropy of the projected Gaussian on the tangent space}
\label{app:isotropy}

\begin{lemma}[Isotropy within $T_{Z_{\ell}}\calM_r$]
\label{lem:isotropy}
Let $\Xi_{\ell}\in\R^{m\times n}$ have i.i.d.\ $\mathcal{N}(0,1)$ entries and let $P_{A_{\ell},B_{\ell}}$ be the orthogonal projector onto $T_{Z_{\ell}}\calM_r$.
For any $U,V\in T_{Z_{\ell}}\calM_r$,
\begin{equation}
\E\bigl[\langle U, P_{A_{\ell},B_{\ell}}(\Xi_{\ell})\rangle \,\langle V, P_{A_{\ell},B_{\ell}}(\Xi_{\ell})\rangle\bigr] \;=\; \langle U,V\rangle.
\label{eq:isotropy}
\end{equation}
Equivalently, $P_{A_{\ell},B_{\ell}}(\Xi_{\ell})$ is an isotropic Gaussian in the tangent space under the Frobenius inner product.
\end{lemma}

\begin{proof}
Because $P_{A_{\ell},B_{\ell}}$ is an orthogonal projector, it is self-adjoint: $\langle U,P_{A_{\ell},B_{\ell}}(X)\rangle=\langle P_{A_{\ell},B_{\ell}}(U),X\rangle$ for all $U,X$.
For $U\in T_{Z_{\ell}}\calM_r$, $P_{A_{\ell},B_{\ell}}(U)=U$.
Therefore $\langle U,P_{A_{\ell},B_{\ell}}(\Xi_{\ell})\rangle=\langle U,\Xi_{\ell}\rangle$ and similarly for $V$.
Since $\Xi_{\ell}$ has i.i.d.\ standard normal entries, $\langle U,\Xi_{\ell}\rangle$ is a centered Gaussian with variance $\|U\|_F^2$, and
\[
\E[\langle U,\Xi_{\ell}\rangle\langle V,\Xi_{\ell}\rangle]=\langle U,V\rangle.
\]
\end{proof}

\subsection{Projected Gaussian covariance and intrinsic dimension}
\label{app:noise_cov}

\begin{lemma}[Covariance of a projected dense Gaussian]
\label{lem:noise_cov}
Let $\Xi_{\ell}\in\R^{m\times n}$ have i.i.d.\ $\mathcal{N}(0,1)$ entries and let $Z_{\ell}=A_{\ell}B_{\ell}^\top\in\mathcal{M}_r$.
Then $P_{A_{\ell},B_{\ell}}(\Xi_{\ell})$ is a centered Gaussian supported on $T_{Z_{\ell}}\mathcal{M}_r$ with vectorized covariance
\begin{equation}
\label{eq:cov_kronecker}
\mathrm{Cov}\!\left[\vecop\!\bigl(P_{A_{\ell},B_{\ell}}(\Xi_{\ell})\bigr)\right]
=
(\Id_n\otimes\projA)+(\projB\otimes\Id_m)-(\projB\otimes\projA).
\end{equation}
If $A_{\ell}$ and $B_{\ell}$ have full column rank, the covariance operator in Eq.~\eqref{eq:cov_kronecker} is an orthogonal projector of rank $r(m+n-r)$, which implies $\E\|P_{A_{\ell},B_{\ell}}(\Xi_{\ell})\|_F^2=r(m+n-r)$ (Eq.~\eqref{eq:energy_gila}).
\end{lemma}

\begin{proof}[Proof of Lemma~\ref{lem:noise_cov}]
Write $P(\Xi_{\ell})=\projA\Xi_{\ell}+\Xi_{\ell}\projB-\projA\Xi_{\ell}\projB$ and apply vectorization:
$\vecop(\projA\Xi_{\ell})=(\Id\otimes \projA)\vecop(\Xi_{\ell})$,
$\vecop(\Xi_{\ell}\projB)=(\projB\otimes \Id)\vecop(\Xi_{\ell})$,
$\vecop(\projA\Xi_{\ell}\projB)=(\projB\otimes \projA)\vecop(\Xi_{\ell})$.
Thus
\[
\vecop(P(\Xi_{\ell}))=
\bigl(\Id\otimes\projA + \projB\otimes\Id - \projB\otimes\projA\bigr)\vecop(\Xi_{\ell}).
\]
Since $\vecop(\Xi_{\ell})\sim \mathcal{N}(0,\Id)$, the covariance is \eqref{eq:cov_kronecker}.
When $A_{\ell},B_{\ell}$ are full column rank, this covariance is an orthogonal projector with rank $r(m+n-r)$ \cite{edelman1998geometry}.
The expected squared norm equals the trace, giving \eqref{eq:energy_gila}.
\end{proof}

\subsection{Concentration of the effective intrinsic noise in PRISM}
\label{app:noise_concentration}

Because $P_{A_{\ell},B_{\ell}}(\Xi_{\ell})$ is an isotropic Gaussian in the tangent space (Lemma~\ref{lem:isotropy}), its Frobenius norm concentrates sharply.
This provides high-probability control beyond the expectation in \eqref{eq:energy_gila}.

\begin{lemma}[Chi-square form]
\label{lem:chisq}
Assume $A_{\ell}$ and $B_{\ell}$ are full column rank and let $d\triangleq \dim(T_{Z_{\ell}}\calM_r)=r(m+n-r)$.
Let $\Xi_{\ell}\in\R^{m\times n}$ have i.i.d.\ $\mathcal{N}(0,1)$ entries and set $G=P_{A_{\ell},B_{\ell}}(\Xi_{\ell})$.
Then $\|G\|_F^2$ has a chi-square distribution with $d$ degrees of freedom:
\begin{equation}
\|G\|_F^2 \ \sim\ \chi^2_d.
\label{eq:chisq}
\end{equation}
\end{lemma}

\begin{proof}
Let $\{E_1,\ldots,E_d\}$ be any orthonormal basis of $T_{Z_{\ell}}\calM_r$ under $\langle\cdot,\cdot\rangle$.
Since $P_{A_{\ell},B_{\ell}}$ is the orthogonal projector onto $T_{Z_{\ell}}\calM_r$, we may write
$G=\sum_{k=1}^d \langle \Xi_{\ell},E_k\rangle E_k$.
By orthonormality and independence of Gaussian linear functionals,
the coefficients $\{\langle \Xi_{\ell},E_k\rangle\}_{k=1}^d$ are i.i.d.\ $\mathcal{N}(0,1)$.
Therefore $\|G\|_F^2=\sum_{k=1}^d \langle \Xi_{\ell},E_k\rangle^2$ is chi-square with $d$ degrees of freedom.
\end{proof}

\begin{proposition}[High-probability bound for PRISM noise]
\label{prop:noise_tail}
Let $\mathcal{N}_{Z_{\ell}}=\frac{\sigma C}{b}P_{A_{\ell},B_{\ell}}(\Xi_{\ell})$ be the intrinsic Gaussian perturbation in DP-PRISM.
Let $d=r(m+n-r)$.
Then for any $\delta\in(0,1)$,
\begin{equation}
\Pr\!\left(
\|\mathcal{N}_{Z_{\ell}}\|_F
\;\le\;
\frac{\sigma C}{b}\sqrt{d + 2\sqrt{d\log(1/\delta)} + 2\log(1/\delta)}
\right)
\ \ge\ 1-\delta.
\label{eq:tail}
\end{equation}
\end{proposition}

\begin{proof}
By Lemma~\ref{lem:chisq}, $\|\mathcal{N}_{Z_{\ell}}\|_F^2 = (\sigma C/b)^2 X$ where $X\sim\chi^2_d$.
A standard chi-square concentration inequality gives
$\Pr(X-d \ge 2\sqrt{d t}+2t)\le e^{-t}$ for all $t\ge 0$.
Setting $t=\log(1/\delta)$ yields \eqref{eq:tail}.
\end{proof}

\paragraph{Remark.}
Proposition~\ref{prop:noise_tail} shows that the realized intrinsic noise magnitude in PRISM concentrates around its mean $\mathcal{E}_{Z_{\ell}}$ with relative fluctuations $O(1/\sqrt{d})$.
This is useful when interpreting the privacy--utility trade-off in large layers, where $d=r(m+n-r)$ is large.

\subsection{Retraction and rank-$r$ approximation}
\label{app:retraction}

We justify Proposition~\ref{prop:retract_error} and record standard facts about truncated SVD retractions.

\begin{lemma}[Eckart--Young--Mirsky theorem]
\label{lem:eym}
% Let $X\in\R^{m\times n}$ have singular values $\sigma_1\ge\cdots\ge \sigma_{\min(m,n)}$.
Let $X\in\R^{m\times n}$ have singular values $s_1\ge\cdots\ge s_{\min(m,n)}$.
Let $X_r$ be the truncated SVD keeping the top $r$ singular values.
Then $X_r$ is a best rank-$r$ approximation in Frobenius norm:
% \[
% X_r \in \arg\min_{\rank(Y)\le r}\|X-Y\|_F,
% \qquad
% \|X-X_r\|_F^2 = \sum_{k>r}\sigma_k^2.
% \]
\[
X_r \in \arg\min_{\rank(Y)\le r}\|X-Y\|_F,
\qquad
\|X-X_r\|_F^2 = \sum_{k>r} s_k^2.
\]
\end{lemma}

\begin{proof}
See \cite{eckart1936approximation,mirsky1960symmetric}.
\end{proof}

\begin{proof}[Proof of Proposition~\ref{prop:retract_error}]
Let $X_\eta = Z-\eta\Delta Z$. Define
\[
Y_\eta=(A-\eta\Delta A)(B-\eta\Delta B)^\top .
\]
Then $\operatorname{rank}(Y_\eta)\le r$, and
\[
Y_\eta
=Z-\eta(\Delta A B^\top+A\Delta B^\top)
+\eta^2\Delta A\Delta B^\top
=X_\eta+\eta^2\Delta A\Delta B^\top .
\]
By Lemma~A.25, $\Retr_r(X_\eta)$ is a best rank-$r$
approximation to $X_\eta$. Since $Y_\eta$ is a rank-$r$
candidate,
\[
\|X_\eta-\Retr_r(X_\eta)\|_F
\le
\|X_\eta-Y_\eta\|_F
=
\eta^2\|\Delta A\Delta B^\top\|_F .
\]
This proves Eq.~\eqref{eq:retract_error1}. Finally,
$\|\Delta A\Delta B^\top\|_F
\le \|\Delta A\|_F\|\Delta B\|_F$, giving the stated
second-order distortion bound.
\end{proof}

\subsection{Absence of bilinear second-order DP noise in PRISM}
\label{app:no_second_order}

\begin{lemma}[PRISM noise is additive in the intrinsic parameter]
\label{lem:no_second_order}
Consider one PRISM iteration for a single LoRA module.
Conditioned on the minibatch and on the Gaussian randomness used in \eqref{eq:dpnoise}, the update takes the intrinsic additive form
\[
Z_{\ell}^+ = \Retr_r\!\left(Z_{\ell} - \eta\left(\bar{\Delta Z_{\ell}} + \tfrac{\sigma C}{b}P_{A_{\ell},B_{\ell}}(\Xi_{\ell})\right)\right),
\]
where $\bar{\Delta Z_{\ell}}$ is the clipped mean tangent update.
In particular, the only randomness in the intrinsic update is the \emph{linear} Gaussian term $P_{A_{\ell},B_{\ell}}(\Xi_{\ell})$; there is no bilinear product of independent noises analogous to $\xi_{A,\ell}\xi_{B,\ell}^\top$ in \eqref{eq:bilinear_noise_prob}.
\end{lemma}

\begin{proof}
This is immediate from the definition of PRISM in \eqref{eq:dpnoise}--\eqref{eq:retraction} and the linearity of $P_{A_{\ell},B_{\ell}}$.
Retraction $\Retr_r$ and any subsequent factorization/gauge alignment are deterministic post-processing steps.
\end{proof}

\subsection{Proof of Theorem~\ref{thm:gauge_invariant}}
\label{app:gauge_invariant}

\begin{proof}[Proof of Theorem~\ref{thm:gauge_invariant}]
Part (i) is Lemma~\ref{lem:proj_gauge}.
Part (ii) follows from Proposition~\ref{prop:tangentproj} since $P_{A_{\ell},B_{\ell}}$ depends only on $(\projA,\projB)$.
For (iii), $P_{A_{\ell},B_{\ell}}(\Xi_{\ell})$ is a measurable function of $(\projA,\projB,\Xi_{\ell})$ and $(\projA,\projB)$ are unchanged under gauge transformations, hence the induced distribution is unchanged.
Retraction, refactorization, and gauge alignment are deterministic maps of the intrinsic quantities, so they preserve gauge invariance.
\end{proof}

\subsection{Gaussian mechanism on a linear subspace}
\label{app:subspace_gaussian}

DP analyses are often stated for outputs in $\R^d$ with full-dimensional Gaussian noise.
PRISM adds Gaussian noise supported on the tangent subspace $T_{Z_{\ell}}\calM_r$.
This is still a standard Gaussian mechanism once the output space is identified with the subspace.

\begin{lemma}[Gaussian mechanism restricted to a subspace]
\label{lem:gauss_subspace}
Let $S\subseteq \R^d$ be a linear subspace with orthogonal projector $\Pi_S$.
Let $f:\mathcal{D}\mapsto S$ be a function with $\ell_2$ sensitivity at most $\Delta$:
$\|f(\mathcal{D})-f(\mathcal{D}')\|_2\le \Delta$ for all adjacent $\mathcal{D},\mathcal{D}'$.
Let $g\sim\mathcal{N}(0,\Id_d)$ and define the mechanism
\[
\mathcal{M}(\mathcal{D}) \triangleq f(\mathcal{D}) + \sigma \Delta\, \Pi_S g.
\]
Then $\mathcal{M}$ is $(\varepsilon,\delta)$-DP for the same $(\varepsilon,\delta)$ guarantee as the standard Gaussian mechanism in dimension $\dim(S)$ (with noise multiplier $\sigma$).
\end{lemma}

\begin{proof}
Let $k=\dim(S)$ and let $U\in\R^{d\times k}$ have orthonormal columns spanning $S$ so that $\Pi_S=UU^\top$.
Write $f(\mathcal{D})=U \alpha(\mathcal{D})$ for some $\alpha(\mathcal{D})\in\R^k$.
Then $\Pi_S g = UU^\top g \overset{d}{=} U h$ where $h\sim\mathcal{N}(0,\Id_k)$.
Therefore $\mathcal{M}(\mathcal{D}) \overset{d}{=} U(\alpha(\mathcal{D}) + \sigma\Delta h)$.
Since $U$ is an isometry on $S$, the DP guarantee for $\alpha(\mathcal{D}) + \sigma\Delta h$ (a standard Gaussian mechanism in $\R^k$) transfers directly to $\mathcal{M}$.
\end{proof}

\subsection{Procrustes alignment is a gauge transform}
\label{app:procrustes}

\begin{lemma}[Orthogonal alignment is gauge preserving]
\label{lem:procrustes_gauge}
Let $Z_{\ell}=A_{\ell}B_{\ell}^\top$ with $A_{\ell},B_{\ell}$ full column rank and let $Q$ be orthogonal.
Then $(A'_{\ell},B'_{\ell})=(A_{\ell}Q,B_{\ell}Q)$ satisfies $A'_{\ell}{B'_{\ell}}^\top=Z_{\ell}$ and leaves $\projA,\projB$ unchanged.
\end{lemma}

\begin{proof}
$A'_{\ell}{B'_{\ell}}^\top=(A_{\ell}Q)(B_{\ell}Q)^\top=A_{\ell}QQ^\top B_{\ell}^\top=A_{\ell}B_{\ell}^\top$.
For the projector,
\[
A'_{\ell}\bigl((A'_{\ell})^\top A'_{\ell}\bigr)^{\pinv}(A'_{\ell})^\top
= A_{\ell}Q\bigl(Q^\top A_{\ell}^\top A_{\ell}Q\bigr)^{\pinv}Q^\top A_{\ell}^\top
= A_{\ell}(A_{\ell}^\top A_{\ell})^{\pinv}A_{\ell}^\top
= \projA,
\]
because \(\bigl(Q^\top X Q\bigr)^{\pinv}=Q^\top X^{\pinv}Q\) for orthogonal \(Q\).

\end{proof}

\subsection{DP guarantee details}
\label{app:dp_proof}

We provide a proof of Theorem~\ref{thm:dp}.

\begin{proof}[Proof of Theorem~\ref{thm:dp}]
Write the per-example intrinsic tangent update (concatenated across all LoRA modules) as
$\Delta \mathbf{Z_{\ell}}_i\in \mathbb{T}$, where $\mathbb{T}$ denotes the direct-sum tangent space equipped with the Frobenius inner product.
PRISM applies intrinsic clipping (Eq.~\eqref{eq:global_clip}) to obtain
$\tilde{\Delta \mathbf{Z_{\ell}}}_i=\alpha_i\,\Delta \mathbf{Z_{\ell}}_i$ with $\|\tilde{\Delta \mathbf{Z_{\ell}}}_i\|_F\le C$.
Hence the $\ell_2$ sensitivity of the minibatch average is bounded by
\[
\left\|\frac{1}{b}\sum_{i=1}^b \tilde{\Delta \mathbf{Z_{\ell}}}_i(\mathcal{D})
-
\frac{1}{b}\sum_{i=1}^b \tilde{\Delta \mathbf{Z_{\ell}}}_i(\mathcal{D}')
\right\|_F
\le \frac{C}{b}
\]
for any adjacent datasets $\mathcal{D},\mathcal{D}'$.

Next, PRISM adds Gaussian noise of standard deviation $\sigma C/b$ in $\mathbb{T}$.
Concretely, each module samples a dense Gaussian matrix and applies the orthogonal tangent projector (Eq.~\eqref{eq:tangent_proj}), so the resulting noise is a Gaussian restricted to a linear subspace.
By Lemma~\ref{lem:gauss_subspace}, the released vector
\[
\widehat{\Delta \mathbf{Z_{\ell}}}
=
\frac{1}{b}\sum_{i=1}^b \tilde{\Delta \mathbf{Z_{\ell}}}_i
+
\frac{\sigma C}{b}\,\mathbf{G},
\qquad \mathbf{G}\sim\mathcal{N}(0,\Pi_{\mathbb{T}}),
\]
is an instance of the Gaussian mechanism with sensitivity $C/b$.

Finally, under Poisson subsampling with rate $q=b/N$, each iteration is a \emph{subsampled} Gaussian mechanism.
The overall $(\varepsilon,\delta)$ guarantee after $T$ steps follows from standard privacy-loss composition for subsampled Gaussian mechanisms, and PRISM uses the PRV accountant implemented in Opacus to compute $\varepsilon$ for a target $\delta$ \cite{gopi2021numerical,yousefpour2021opacus,opacus_privacy_engine}.
All subsequent operations (adaptive post-processing, factorization, alignment, and retraction) are deterministic post-processing and therefore do not weaken DP.
\end{proof}

\subsection{Rank-space moments of isotropic tangent noise}
\label{app:noise_moments}
This section derives the rank-space second moments of the isotropic tangent noise used by PRISM (Eq.~\eqref{eq:noise_form}), which motivates the DP-aware floors in Eq.~\eqref{eq:noise_floor}.
Let $U\sim\mathcal{N}(0,I_{m\times r})$ and $V\sim\mathcal{N}(0,I_{n\times r})$, and define
$\Xi_{A,\ell}=(I-\projA)U\,N^{-1/2}$ and $\Xi_{B,\ell}=V\,M^{-1/2}$ with $M=A_{\ell}^\top A_{\ell}$ and $N=B_{\ell}^\top B_{\ell}$.
Then
\begin{equation}
\mathbb{E}[\Xi_{A,\ell}^\top \Xi_{A,\ell}]
= N^{-1/2}\,\mathbb{E}[U^\top (I-\projA)U]\,N^{-1/2},
\qquad
\mathbb{E}[\Xi_{B,\ell}^\top \Xi_{B,\ell}]
= M^{-1/2}\,\mathbb{E}[V^\top V]\,M^{-1/2}.
\label{eq:rank_moments_start}
\end{equation}
Since $U$ has i.i.d.\ standard normal entries and $(I-\projA)$ is an orthogonal projector of rank $\mathrm{tr}(I-\projA)=m-r$, we have
$\mathbb{E}[U^\top (I-\projA)U]=(m-r)I_r$.
Similarly, $\mathbb{E}[V^\top V]=n I_r$.
Substituting into \eqref{eq:rank_moments_start} yields
\begin{equation}
\mathbb{E}\Big[\frac{\Xi_{A,\ell}^\top \Xi_{A,\ell}}{m}\Big]=\frac{m-r}{m}\,N^{-1},
\qquad
\mathbb{E}\Big[\frac{\Xi_{B,\ell}^\top \Xi_{B,\ell}}{n}\Big]=M^{-1}.
\label{eq:rank_moments}
\end{equation}
Thus the typical eigenvalues of the rank-space noise covariance scale with $M^{-1}$ and $N^{-1}$, explaining why inverse-square-root preconditioning can explode when $M$ or $N$ is ill-conditioned.
PRISM's floors in Eq.~\eqref{eq:noise_floor} are gauge invariant because $\mathrm{tr}(M^{-1})$ and $\mathrm{tr}(N^{-1})$ are invariant under $(A_{\ell},B_{\ell})\mapsto(A_{\ell}R,B_{\ell}R^{-\top})$.

\subsection{Adaptive preconditioning and DP noise amplification}
\label{app:precond}

This subsection complements Section~\ref{sec:theory_opt} and proves Theorem~\ref{thm:precond_bound}.
We also record a simple identity showing how rank-space normalization can ``cancel'' the DP noise scale when the second moment is dominated by noise.

% \begin{proof}[Proof of Theorem~\ref{thm:precond_bound}]
% Since $V\succeq 0$, write its eigendecomposition $V=U\Lambda U^\top$ with $\Lambda=\mathrm{diag}(\lambda_1,\dots,\lambda_r)$ and $\lambda_i\ge 0$.
% Then
% $\mathsf{P}=(V+\lambda I)^{-1/2}=U(\Lambda+\lambda I)^{-1/2}U^\top$ and
% \[
% \|\mathsf{P}\|_2=\max_i (\lambda_i+\lambda)^{-1/2}\le \lambda^{-1/2}.
% \]
% For any $X$, submultiplicativity of the Frobenius norm gives
% $\|X\mathsf{P}\|_F\le \|X\|_F\|\mathsf{P}\|_2\le \lambda^{-1/2}\|X\|_F$.
% Squaring yields Eq.~\eqref{eq:precond_bound}.
% \end{proof}

\begin{proof}[Proof of Theorem~\ref{thm:precond_bound}]
Since $V\succeq 0$, write its eigendecomposition $V=U\Lambda U^\top$ with $\Lambda=\mathrm{diag}(\lambda_1,\dots,\lambda_r)$ and $\lambda_i\ge 0$.
Then
$\mathsf{P}=V+\lambda I = U(\Lambda+\lambda I)U^\top$ and thus
$\mathsf{P}^{-1/2}=U(\Lambda+\lambda I)^{-1/2}U^\top$.
\[
\|\mathsf{P}^{-1/2}\|_2=\max_i (\lambda_i+\lambda)^{-1/2}\le \lambda^{-1/2}.
\]
For any $X$, submultiplicativity of the Frobenius norm gives
$\|X\mathsf{P}^{-1/2}\|_F\le \|X\|_F\|\mathsf{P}^{-1/2}\|_2\le \lambda^{-1/2}\|X\|_F$.
Squaring yields Eq.~\eqref{eq:precond_bound}.
\end{proof}

\begin{proposition}[Noise normalization under naive rank-space preconditioning]
\label{prop:noise_normalization}
Let $G\in\R^{m\times r}$ have i.i.d.\ $\mathcal{N}(0,1)$ entries and define the (uncentered) second moment $V\triangleq \frac{1}{m}G^\top G$.
Then the preconditioned matrix $Q\triangleq G V^{-1/2}$ satisfies
\[
Q^\top Q = m I_r
\qquad\text{and hence}\qquad
\|Q\|_F^2 = m r.
\]
Equivalently, if $\widehat m=\tau G$ for any $\tau>0$ and $V=\frac{1}{m}\widehat m^\top\widehat m$, then $\widehat m\,V^{-1/2}$ has Frobenius norm $\sqrt{mr}$ \emph{independent} of $\tau$.
\end{proposition}

\begin{proof}
By definition,
$Q^\top Q = V^{-1/2}G^\top G V^{-1/2} = V^{-1/2}(mV)V^{-1/2}=mI_r$.
Taking traces gives $\|Q\|_F^2=\tr(Q^\top Q)=mr$.
The final claim follows from $V=\tau^2(\frac{1}{m}G^\top G)$, which implies $\widehat m V^{-1/2}=G(\frac{1}{m}G^\top G)^{-1/2}=Q$.
\end{proof}

Proposition~\ref{prop:noise_normalization} explains the issue in Issue~III:
when an adaptive method forms $V$ directly from a DP-sanitized gradient whose energy is dominated by DP noise, the right-multiplication by $V^{-1/2}$ can make the stochastic component insensitive to the DP noise scale.
PRISM avoids this via (i) DP-aware floors and condition-number clamping (Eq.~\eqref{eq:noise_floor}), which bound $\|V^{-1/2}\|_2$ and prevent ill-conditioned amplification, and (ii) debiasing of the second moment by subtracting the known DP noise covariance in rank space (Appendix~\ref{app:noise_moments}).

\begin{lemma}[Orthogonal gauge equivariance of rank-space preconditioning]
\label{lem:orth_precond}
Let $R\in\mathbb{R}^{r\times r}$ be orthogonal and consider the restricted gauge transform $(A_{\ell},B_{\ell})\mapsto (A_{\ell}R,B_{\ell}R)$.
If $\widehat m_A,\widehat m_B$ transform as $\widehat m_A'=\widehat m_A R$ and $\widehat m_B'=\widehat m_B R$, and the second moments transform as $V_A'=R^\top V_A R$ and $V_B'=R^\top V_B R$, then the preconditioned directions from Eq.~\eqref{eq:precond_dir} satisfy
$U_A'=U_A R$ and $U_B'=U_B R$, and therefore the intrinsic update $U_A B_{\ell}^\top + A_{\ell} U_B^\top$ is invariant.
\end{lemma}

\begin{proof}
For orthogonal $R$, similarity equivariance of matrix functions yields
$(V_A'+\lambda I)^{-1/2} = (R^\top(V_A+\lambda I)R)^{-1/2} = R^\top (V_A+\lambda I)^{-1/2} R$.
Thus
$U_A'=\widehat m_A'(V_A'+\lambda I)^{-1/2}=\widehat m_A R \, R^\top (V_A+\lambda I)^{-1/2} R = U_A R$,
and similarly $U_B'=U_B R$.
Finally, with $A'_{\ell}=A_{\ell}R$ and $B'_{\ell}=B_{\ell}R$ we have
\begin{align*}
U_A' (B'_{\ell})^\top + A'_{\ell} (U_B')^\top
&= (U_A R)(B_{\ell}R)^\top + (A_{\ell}R)(U_B R)^\top \\
&= (U_A R)(R^\top B_{\ell}^\top) + (A_{\ell}R)(R^\top U_B^\top) \\
&= U_A (RR^\top) B_{\ell}^\top + A_{\ell}(RR^\top) U_B^\top \\
&= U_A B_{\ell}^\top + A_{\ell} U_B^\top,
\end{align*}
\end{proof}

\section{Experimental Setup}
\label{app:setup}

\subsection{Experimental Details}
\label{app:exp_details}

This section reports dataset split sizes, hyperparameters, and the hardware/software
environment used in our experiments.

\paragraph{Compute.}
All experiments were run on a single NVIDIA \textbf{A100-PCIE-40GB} GPU.
\paragraph{Datasets and splits.}
\Cref{tab:app_datasets} reports the training sizes we used and evaluation split sizes.
GLUE8 is derived from GLUE~\citep{wang2018glue} (excluding WNLI), converted to an
instruction-format JSON dataset, and sub-sampled with a fixed number of training
examples per task. Math-10K is the LLM-Adapters mixture~\citep{hu2023llm_adapters,llm_adapters_code}
and is evaluated on the standard test splits of its component datasets.

\begin{table}[t]
\caption{Dataset splits used in our experiments.}
\label{tab:app_datasets}
\begin{center}
\begin{small}
\begin{sc}
\begin{tabular}{lrr}
\toprule
\textbf{Dataset / Split} & \textbf{Train} & \textbf{Eval} \\
\midrule
\multicolumn{3}{l}{\textbf{GLUE8}}\\
CoLA   & 1{,}250 & 1{,}043 \\
SST-2  & 1{,}250 & 872 \\
MRPC   & 1{,}250 & 408 \\
STS-B  & 1{,}250 & 1{,}500 \\
QQP    & 1{,}250 & 40{,}430 \\
MNLI (matched / mismatched) & 1{,}250 & 9{,}815 / 9{,}832 \\
QNLI   & 1{,}250 & 5{,}463 \\
RTE    & 1{,}250 & 277 \\
\midrule
\textbf{GLUE8 total} & 10{,}000 & -- \\
\midrule
\multicolumn{3}{l}{\textbf{Math-10K}}\\
Math-10K train (mixture) & 9{,}919 & -- \\
GSM8K test  & -- & 1{,}319 \\
AQuA test   & -- & 254 \\
MAWPS test  & -- & 238 \\
SVAMP test  & -- & 1{,}000 \\
\bottomrule
\end{tabular}
\end{sc}
\end{small}
\end{center}
\vskip -0.1in
\end{table}

\paragraph{Common fine-tuning hyperparameters.}
Unless otherwise stated, all methods share the same backbone, LoRA configuration,
and DP settings in \Cref{tab:app_hparams_common}. For DP runs, the noise multiplier
is calibrated with Opacus \texttt{make\_private\_with\_epsilon} using the default PRV
accountant~\citep{gopi2021numerical,opacus_privacy_engine}.

\begin{table}[t]
\caption{Common hyperparameters shared across methods.}
\label{tab:app_hparams_common}
\begin{center}
\begin{small}
\begin{sc}
\begin{tabular}{lcc}
\toprule
\textbf{Setting} & \textbf{GLUE8} & \textbf{Math-10K} \\
\midrule
Backbone model & \texttt{google/\allowbreak gemma-3-4b-pt}~\citep{gemma3} & same \\
LoRA rank $r$ & 16 & 16 \\
% LoRA $\alpha$ & 16 & 16 \\
LoRA scaling $\alpha_{\mathrm{LoRA}}$ & 16 & 16 \\
LoRA dropout & 0.05 & 0.05 \\
Target modules & \{q\_proj,k\_proj,v\_proj,up\_proj,down\_proj\} & same \\
Update steps & 500 & 300 \\
Effective batch size & 64 & 64 \\
Micro-batch size & 4 & 4 \\
Max sequence length & 384 & 256 \\
Train on inputs & False & True \\
Random seed & 42 & 42 \\
DP budgets & $\varepsilon \in \{3,6\}$, $\delta=10^{-5}$ & same \\
Clipping norm $C$ & 1.0 & 1.0 \\
DP grad-sample backend & \texttt{functorch} (fallback to hooks) & same \\
DP accountant & PRV (Opacus default)~\citep{gopi2021numerical,opacus_privacy_engine} & same \\
\bottomrule
\end{tabular}
\end{sc}
\end{small}
\end{center}
\vskip -0.1in
\end{table}

\begin{table}[t]
\caption{Method-specific hyperparameters.}
\label{tab:app_hparams_methods}
\begin{center}
\begin{small}
\begin{sc}

{\setlength{\tabcolsep}{4.5pt}
\begin{tabular}{l p{0.24\textwidth} c c p{0.28\textwidth}}
\toprule
\textbf{Method} & \textbf{Optimizer} & \textbf{LR (GLUE8)} & \textbf{LR (Math)} & \textbf{Supplementary} \\
\midrule
AdamW & AdamW~\citep{kingma2015adam,loshchilov2019adamw}
& $2{\times}10^{-4}$ & $3{\times}10^{-4}$ & -- \\
FFA & AdamW + freeze $A_{\ell}$~\citep{sun2024improvinglora}
& $2{\times}10^{-4}$ & $3{\times}10^{-4}$ & -- \\
RITE & LoRA-RITE~\citep{yen2025rite}
& $2{\times}10^{-4}$ & $3{\times}10^{-4}$ & -- \\
LoRA+ & AdamW  split LRs~\citep{hayou2024loraplus}
& $2{\times}10^{-4}$ & $3{\times}10^{-4}$ & Ratio $\rho=6.0$. \\
LAMB & LAMB~\citep{you2020lamb}
& $5{\times}10^{-3}$ & $5{\times}10^{-3}$ & -- \\
PRISM & \textsc{PRISM} (ours)
& $2{\times}10^{-4}$ & $3{\times}10^{-4}$ & -- \\
\bottomrule
\end{tabular}}
\end{sc}
\end{small}
\end{center}
\vskip -0.1in
\end{table}

\section{Additional Diagnostics and Analysis}
\label{app:analysis}

\subsection{Additional diagnostics for Issue I }
\label{app:pathology1_extra}

\paragraph{Diagnostic protocol and hyperparameters.}
We evaluate gauge sensitivity by running \emph{ DP training} under multiple equivalent LoRA factorizations of the same intrinsic update $Z_{\ell}=A_{\ell}B_{\ell}^\top$.
For each gauge $c$, we apply the reparameterization $(A_{\ell},B_{\ell})\leftarrow(cA_{\ell},c^{-1}B_{\ell})$, which leaves $Z_{\ell}$ unchanged but alters factor-space norms.
We train for $T=300$ update steps on Math-10K and log (i) clipping fraction $\mathrm{dp\_clip\_frac}$, (ii) mean clipping coefficient $\mathrm{dp\_coef\_mean}$, and (iii) realized intrinsic step magnitude $\|\Delta Z_t\|_F$, where $\Delta Z_t \equiv Z_{t+1}-Z_t$ is computed from the \emph{actual parameter update} (not a formula-level proxy). \textbf{Gauges:} $c\in\{0.25,0.5,1.0,2.0,4.0\}$.

\paragraph{Metrics and theoretical link.}
In baseline factor-space DP-SGD, per-example clipping uses the factor norm
$s_i^{\mathrm{fact}}=\sqrt{\|g_{A,i}\|_F^2+\|g_{B,i}\|_F^2}$ and $\alpha_i=\min\{1,C_{\mathrm{fact}}/s_i^{\mathrm{fact}}\}$ (cf.\ \eqref{eq:dpsgd}).
Under gauge rescaling, the norm transforms as \eqref{eq:gauge_norm}, so $\alpha_i$ and the induced intrinsic update distribution depend on $c$ (Issue-I).
PRISM instead forms intrinsic directions via the tangent construction and projectors (e.g., \eqref{eq:tangent_proj}, \eqref{eq:canonical}), clips using the intrinsic norm \eqref{eq:global_clip}, and adds isotropic tangent noise \eqref{eq:dp_tangent_noise}; these operations are designed to depend on $Z_{\ell}$ rather than on a particular factorization, so gauge dependence should be strongly reduced (up to stochastic variability from DP noise).

\begin{figure}[t]
  \centering
  \includegraphics{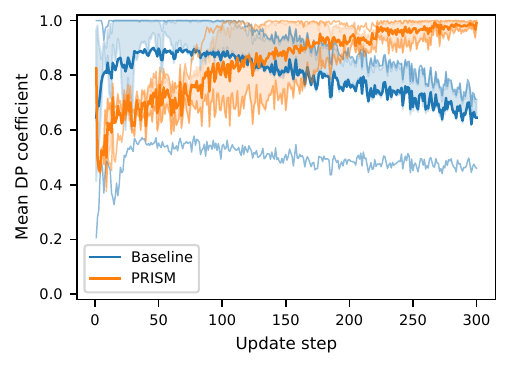}
  \caption{Over-time bands of $\mathrm{dp\_coef\_mean}$ across gauges (mean with IQR band; min/max lines).}
  \label{fig:pathology1_bands_coef}
\end{figure}

\noindent\textbf{Interpretation (\Cref{fig:pathology1_bands_coef}).}
This plot summarizes how the \emph{clipping coefficients} vary across gauges during training.
For baseline, the spread between min/max (and the IQR band) remains wide for most of training, indicating that the same DP configuration produces materially different clipping behavior depending on $(A_{\ell},B_{\ell})$’s gauge.
This matches the mechanism in \eqref{eq:gauge_norm}: changing $c$ reweights $\|g_{A,i}\|_F$ versus $\|g_{B,i}\|_F$, hence changes $s_i^{\mathrm{fact}}$ and pushes different gauges into different clipping regimes (larger/smaller $\alpha_i$).
For PRISM, $\mathrm{dp\_coef\_mean}$ concentrates near $1$ after the transient, and the across-gauge dispersion shrinks, consistent with clipping being controlled by the intrinsic norm $s_i$ in \eqref{eq:global_clip}, which depends on $\Delta Z_{i,\ell}$ rather than on factor scaling.
Importantly, the remaining non-zero dispersion is expected in finite runs because DP noise makes $\alpha_i$ and $\Delta Z_{\ell}$ stochastic, but PRISM’s variability is markedly smaller than baseline’s.

\begin{figure}[t]
  \centering
  \includegraphics{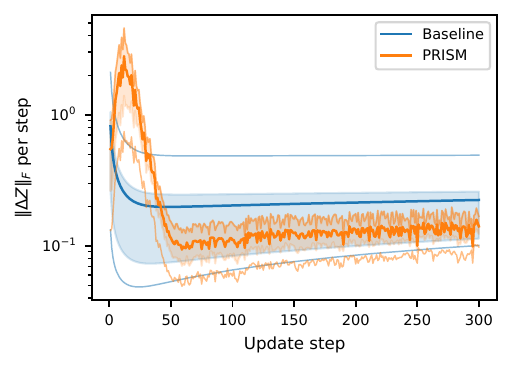}
  \caption{Over-time bands of realized intrinsic step magnitude $\|\Delta Z_t\|_F$ across gauges (mean with IQR band; min/max lines).}
  \label{fig:pathology1_bands_dz}
\end{figure}

\noindent\textbf{Interpretation (\Cref{fig:pathology1_bands_dz}).}
This figure measures the \emph{actual intrinsic update} applied to $Z_{\ell}$ at each step (computed from $Z_{t+1}-Z_t$), thus directly reflecting the DP perturbation that matters in the intrinsic space.
Baseline exhibits a persistent and relatively wide band across gauges: some gauges yield substantially larger $\|\Delta Z_t\|_F$ than others.
This is the operational manifestation of Issue-I: once $\alpha_i$ is gauge-dependent via \eqref{eq:gauge_norm}, the clipped-and-noised factor update implies a gauge-dependent induced update on $Z_{\ell}$, so $\|\Delta Z_t\|_F$ cannot be predicted from $Z_{\ell}$ alone.
PRISM shows a transient early phase (optimizer/moment warm-up plus DP stochasticity) and then stabilizes to a smaller, tighter band; this is consistent with intrinsic clipping \eqref{eq:global_clip} and tangent noise \eqref{eq:dp_tangent_noise} controlling the intrinsic step directly.
The remaining oscillations are natural: even a gauge-invariant \emph{distribution} will yield non-identical single-run trajectories under DP noise, but PRISM suppresses the systematic gauge effect visible in baseline.

\begin{figure}[t]
  \centering
  \includegraphics{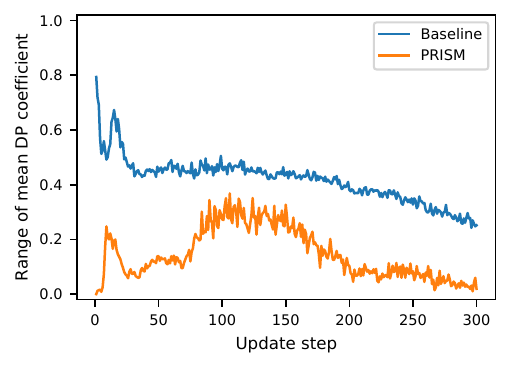}
  \caption{Gauge-sensitivity index for clipping: $\mathrm{range}_c(\mathrm{dp\_coef\_mean})$ over time.}
  \label{fig:pathology1_coef_range}
\end{figure}

\noindent\textbf{Interpretation (\Cref{fig:pathology1_coef_range}).}
We compress the multi-gauge experiment into a single diagnostic:
$\mathrm{range}_c(\mathrm{dp\_coef\_mean})=\max_c \mathrm{dp\_coef\_mean}(c)-\min_c \mathrm{dp\_coef\_mean}(c)$.
A gauge-invariant DP mechanism should make this range small (up to stochastic fluctuations).
Baseline remains high throughout training, directly supporting the theoretical failure mode: factor-space clipping depends on $c$ because of \eqref{eq:gauge_norm}, so the average clipping coefficient changes substantially across equivalent parameterizations.
PRISM yields a much smaller range after the initial transient, consistent with using the intrinsic norm \eqref{eq:global_clip} and tangent construction (\eqref{eq:tangent_proj}, \eqref{eq:canonical}) to decouple DP sensitivity control from gauge.

\begin{figure}[t]
  \centering
  \includegraphics{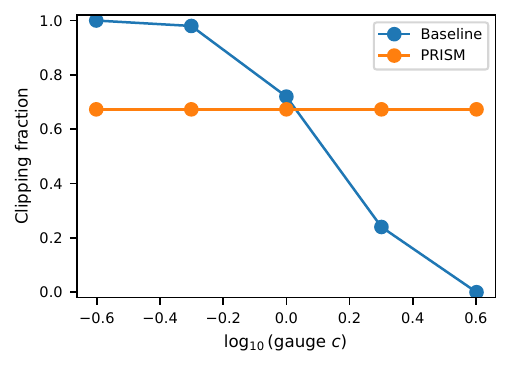}
  \caption{Discrete gauge sweep of $\mathrm{dp\_clip\_frac}$ at step 1.}
  \label{fig:pathology1_sweep1_clip}
\end{figure}

\noindent\textbf{Interpretation (\Cref{fig:pathology1_sweep1_clip}).}
This is the most direct “mechanism check” for Issue-I: the x-axis varies only the gauge $c$, while all intrinsic quantities are initially identical.
At step 1, baseline’s clipping fraction moves from almost fully clipped (small $c$) to almost never clipped (large $c$), i.e., a qualitative regime change triggered purely by reparameterization.
This matches \eqref{eq:gauge_norm}: for small $c$ the $c^{-2}\|g_{A,i}\|_F^2$ term can dominate, inflating norms and forcing $\alpha_i\ll1$; for larger $c$ the norm shrinks and clipping disengages.
PRISM is approximately flat across gauges, aligning with intrinsic clipping \eqref{eq:global_clip}: the clipping decision depends on $\|\Delta Z_{i,\ell}\|_F$ rather than on factor scaling.

\begin{figure}[t]
  \centering
  \includegraphics{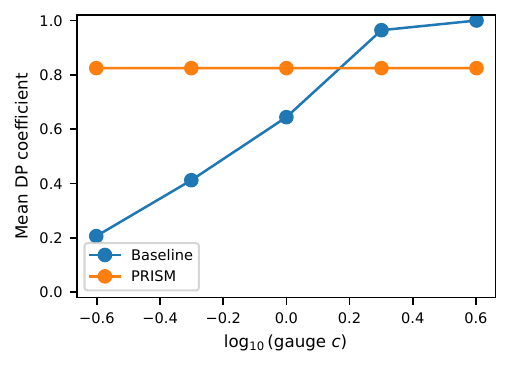}
  \caption{Discrete gauge sweep of $\mathrm{dp\_coef\_mean}$ at step 1.}
  \label{fig:pathology1_sweep1_coef}
\end{figure}

\noindent\textbf{Interpretation (\Cref{fig:pathology1_sweep1_coef}).}
The mean clipping coefficient $\mathrm{dp\_coef\_mean}$ is a smoother counterpart of \Cref{fig:pathology1_sweep1_clip}:
it directly measures the average shrinkage induced by DP clipping, $\alpha_i=\min\{1,C/s_i\}$.
Baseline increases monotonically with $c$ at step 1, showing that the same DP algorithm injects different effective shrinkage (hence different intrinsic update distributions) purely due to gauge, as predicted by \eqref{eq:gauge_norm}.
PRISM remains roughly constant across gauges, consistent with controlling sensitivity in intrinsic space \eqref{eq:global_clip} and therefore avoiding this reparameterization artifact.

\begin{figure}[t]
  \centering
  \includegraphics{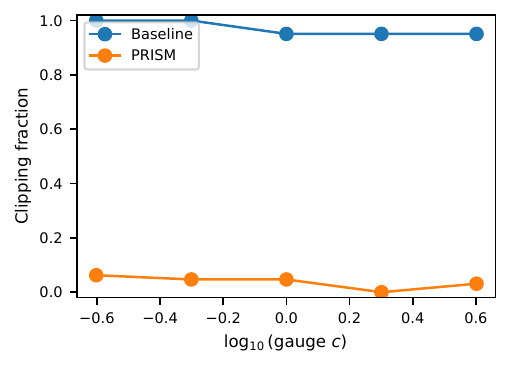}
  \caption{Discrete gauge sweep of $\mathrm{dp\_clip\_frac}$ at step 300.}
  \label{fig:pathology1_sweep300_clip}
\end{figure}

\noindent\textbf{Interpretation (\Cref{fig:pathology1_sweep300_clip}).}
By step 300, baseline’s $\mathrm{dp\_clip\_frac}$ is close to $1$ for all gauges, indicating a \emph{clipping-saturated} regime in factor space: most examples are clipped regardless of $c$.
This supports (but also partially limits) diagnostic interpretability: once clipping saturates, the mechanism becomes less sensitive to further changes in the factor norms, so a flatter curve here does \emph{not} imply gauge invariance.
In contrast, PRISM shows a near-zero clipping fraction at step 300 for all gauges, suggesting that (under the intrinsic threshold $C_{\mathrm{int}}$) optimization has entered a stable region where intrinsic per-example norms mostly lie below the clip bound in \eqref{eq:global_clip}.
Thus, the step-300 sweep is best viewed as confirming that late training can enter a stable/saturated regime, rather than as the primary evidence for Issue-I (which is better captured at step 1 and by \Cref{fig:pathology1_dz_range,fig:pathology1_coef_range}).

\begin{figure}[t]
  \centering
  \includegraphics{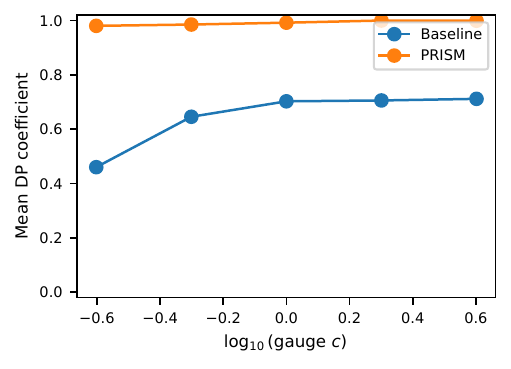}
  \caption{Discrete gauge sweep of $\mathrm{dp\_coef\_mean}$ at step 300.}
  \label{fig:pathology1_sweep300_coef}
\end{figure}

\noindent\textbf{Interpretation (\Cref{fig:pathology1_sweep300_coef}).}
Consistent with \Cref{fig:pathology1_sweep300_clip}, baseline’s $\mathrm{dp\_coef\_mean}$ still varies with gauge at step 300, but the variability is reduced relative to step 1 because clipping is already heavily engaged for all gauges (many $\alpha_i<1$).
This illustrates a subtle but important point: Issue-I is fundamentally about the \emph{mechanism’s dependence on reparameterization} (here seen sharply at step 1), and saturation can mask that dependence by collapsing the algorithm into an always-clipped regime.
PRISM’s $\mathrm{dp\_coef\_mean}$ concentrates near $1$ across gauges at step 300, consistent with intrinsic clipping being mostly inactive and the DP perturbation being governed primarily by tangent noise \eqref{eq:dp_tangent_noise} rather than by gauge-dependent shrinkage.
Together with the intrinsic step sensitivity in \Cref{fig:pathology1_dz_range}, these late-step diagnostics suggest PRISM’s intrinsic control yields a more stable intrinsic update distribution across equivalent factorizations.

\subsection{Additional diagnostics for Issue II }
\label{app:issue2_gauge_sweep}
\paragraph{Gauge sweep protocol.}
We snapshot the LoRA layer with the largest $S_t$ at the end of training and sweep the gauge
$(A_{\ell},B_{\ell})\mapsto(cA_{\ell},c^{-1}B_{\ell})$, which keeps $Z_{\ell}=A_{\ell}B_{\ell}^\top$ fixed.
We evaluate $\log_{10} c \in \mathrm{linspace}(-3,3,61)$
and draw $64$ Monte-Carlo samples per $c$; we plot the median and the 10--90\% band.

\begin{figure}[t]
  \centering
  \includegraphics{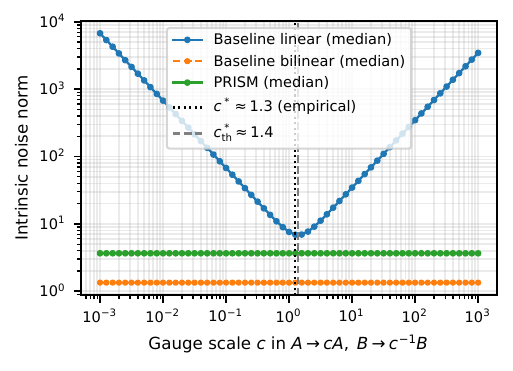}
  \caption{Gauge sweep at fixed $Z_{\ell}$: intrinsic-noise medians with 10--90\% band.}
  \label{fig:issue2_empirical_gauge_sweep_noise}
\end{figure}

\paragraph{What \Cref{fig:issue2_empirical_gauge_sweep_noise} tests.}
Issue~II predicts that factor-space DP can inject a \emph{gauge-dependent} intrinsic noise even when the intrinsic parameter
$Z_{\ell}=A_{\ell}B_{\ell}^\top$ (and thus the model function) is held fixed.
From Eq.~\eqref{eq:naive_energy}, the first-order intrinsic noise satisfies
\[
\E\|\xi_{A,\ell} B_{\ell}^\top + A_{\ell}\xi_{B,\ell}^\top\|_F^2=\tau^2\!\left(m\|B_{\ell}\|_F^2+n\|A_{\ell}\|_F^2\right),
\]
so under $(A_{\ell},B_{\ell})\mapsto(cA_{\ell},c^{-1}B_{\ell})$ the coefficient becomes
$S(c)=m c^{-2}\|B_{\ell}\|_F^2+n c^{2}\|A_{\ell}\|_F^2$, which is minimized at
$c_{\text{th}}=\left(\frac{m\|B_{\ell}\|_F^2}{n\|A_{\ell}\|_F^2}\right)^{1/4}$ and diverges as $c\to 0$ or $c\to\infty$.
PRISM instead samples isotropic tangent noise $P_{A_{\ell},B_{\ell}}(\Xi_{\ell})$, whose distribution depends only on the tangent projector,
and whose energy is controlled by the intrinsic dimension (cf.\ $\mathcal{E}_{Z_{\ell}}=(\sigma C/b)\sqrt{r(m+n-r)}$ in the main text).
Empirically, \Cref{fig:issue2_empirical_gauge_sweep_noise} matches this dichotomy:
the baseline (factor-space DP) curve varies by orders of magnitude across $c$ despite fixed $Z_{\ell}$,
while PRISM remains essentially flat up to Monte-Carlo variability.
The baseline “bilinear” component (from the $\eta\,\xi_{A,\ell}\xi_{B,\ell}^\top$ term in Eq.~\eqref{eq:bilinear_noise_prob})
is comparatively small and gauge-invariant, indicating that the dominant instability here comes from the \emph{linear} term’s gauge dependence.

\begin{figure}[t]
  \centering
  \includegraphics{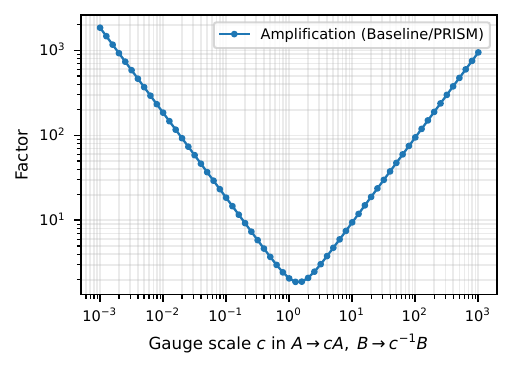}
  \caption{Gauge sweep: amplification factor (median baseline / median PRISM) vs.\ $c$.}
  \label{fig:issue2_gauge_sweep_amplification}
\end{figure}

\paragraph{Amplification factors.}
\Cref{fig:issue2_gauge_sweep_amplification} converts the sweep into an explicit amplification ratio.
Because PRISM’s intrinsic noise is gauge-invariant, the ratio inherits the V-shaped dependence of $S(c)$:
even benign reparameterizations that leave $Z_{\ell}$ unchanged can inflate factor-space intrinsic DP noise by large factors.
This complements the training-time observation in ~\Cref{fig:issue2_amp_vs_step}:
during optimization, the implicit gauge chosen by the optimizer already yields a consistent $>1$ amplification,
and the controlled gauge sweep shows that, in principle, the same model state (same $Z_{\ell}$) admits much larger effective intrinsic noise under factor-space DP.
Together with ~\Cref{fig:issue2_noise_sq_vs_St}, these results empirically validate Issue~II’s core claim:
\emph{factor-space DP produces gauge-dependent, potentially highly amplified intrinsic noise}, while PRISM keeps the intrinsic DP noise scale controlled and gauge-invariant.

\subsection{Additional Issue III diagnostics}\label{app:path3}

\paragraph{Protocols.}
(\textbf{Sigma sweep.})
For \Cref{fig:path3_sigma_precond,fig:path3_sigma_rawnoise,fig:path3_sigma_amp_vs_sigma},
we sweep $\epsilon\in\{1.5,3,6,12\}$ at fixed $(C,\delta)$, run $120$ optimizer steps,
discard the first $10$ steps as burn-in, and report means over the remaining steps.
The x-axis uses the \emph{realized} noise multiplier $\sigma$ returned by the privacy engine.
(\textbf{Step-wise diagnostics.})
For \Cref{fig:path3_precond_max,fig:path3_gram_stress,fig:path3_amp_over_time,fig:path3_scatter_amp_aggr,fig:path3_scatter_amp_gram},
we run $300$ steps at $\epsilon=3$ (hence $\sigma\approx 0.62$ in this setup).

\begin{figure}[t]
  \centering    \includegraphics{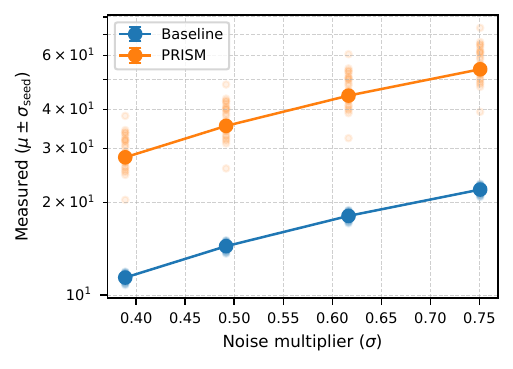}
  \caption{Mean raw intrinsic DP noise $\|\boldsymbol{\xi}_{\mathrm{intr}}\|_F$ vs.\ $\sigma$.}
  \label{fig:path3_sigma_rawnoise}
\end{figure}

\noindent\textbf{Analysis (raw noise scaling).}
For clipped DP-SGD-style noise, the injected intrinsic noise satisfies
$\boldsymbol{\xi}_{\mathrm{intr}}\sim\mathcal{N}\!\bigl(0,(\sigma C/b)^2\mathbf{I}\bigr)$ (up to the intrinsic parameterization),
so $\mathbb{E}\|\boldsymbol{\xi}_{\mathrm{intr}}\|_F$ should grow approximately linearly with $\sigma$ at fixed clipping norm $C$ and batch size $b$.
\Cref{fig:path3_sigma_rawnoise} matches this expectation for both methods, serving as a sanity check that (i) the privacy engine responds correctly to the
$\epsilon$ sweep and (ii) our measurement pipeline is consistent.
Notably, PRISM exhibits a larger \emph{raw} intrinsic noise norm than factor-space DP-AdamW; this is expected because PRISM injects noise directly in the
intrinsic update space, whereas factor-space perturbations are first applied to the LoRA factors and then mapped into the intrinsic update,
which can reduce the resulting $\|\boldsymbol{\xi}_{\mathrm{intr}}\|_F$ via the low-rank geometry.

\begin{figure}[t]
  \centering    \includegraphics{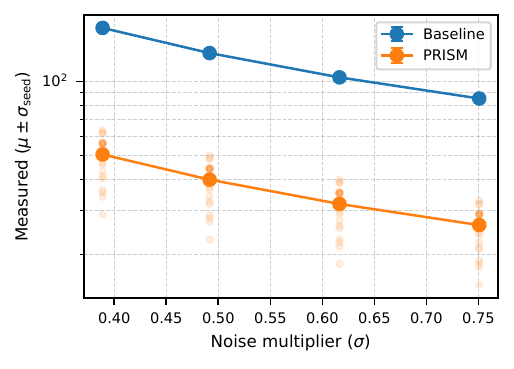}
  \caption{Mean amplification $\|\mathsf{P}^{-1/2}\boldsymbol{\xi}\|_F/\|\boldsymbol{\xi}\|_F$ vs.\ $\sigma$.}
  \label{fig:path3_sigma_amp_vs_sigma}
\end{figure}

\noindent\textbf{Analysis (noise normalization + reduced amplification).}
Define the amplification factor
$a(\sigma)\equiv \mathbb{E}\bigl[\|\mathsf{P}^{-1/2}\boldsymbol{\xi}_{\mathrm{intr}}\|_F/\|\boldsymbol{\xi}_{\mathrm{intr}}\|_F\bigr]$,
which isolates how much the preconditioner scales DP noise in Eq.~\eqref{eq:precond_dp}.
A key prediction of Issue~III is \emph{noise normalization}: when the second-moment estimator is dominated by DP noise,
$\mathbf{V}\propto\sigma^2$ so $(\mathbf{V}+\lambda\mathbf{I})^{-1/2}\propto 1/\sigma$, and the \emph{preconditioned} noise becomes nearly
$\sigma$-invariant (Prop.~\ref{prop:noise_normalization}).
This is exactly what \Cref{fig:path3_sigma_precond} shows; combining it with the near-linear growth of raw noise in \Cref{fig:path3_sigma_rawnoise}
implies $a(\sigma)$ should decrease roughly as $1/\sigma$, which is what \Cref{fig:path3_sigma_amp_vs_sigma} exhibits.
Crucially, PRISM’s amplification is much smaller across $\sigma$ (roughly $18$--$50$ vs.\ $90$--$275$ for DP-AdamW),
supporting the claim that DP-aware floors (Eq.~\eqref{eq:noise_floor}) and the bound in Eq.~\eqref{eq:precond_bound}
control worst-case scaling of DP noise under adaptive preconditioning.

\begin{figure}[t]
  \centering    \includegraphics{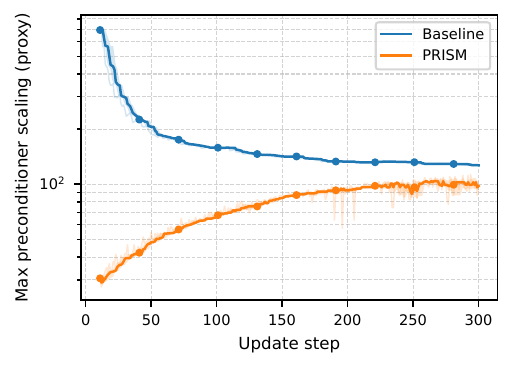}
  \caption{Preconditioner aggressiveness (proxy) over training steps.}
  \label{fig:path3_precond_max}
\end{figure}

\noindent\textbf{Analysis (why amplification can be large).}
\Cref{fig:path3_precond_max} plots a ``max scaling'' proxy for preconditioning strength, roughly corresponding to the largest coordinate-wise scaling
(e.g., $\max_i(\widehat v_i+\epsilon_{\text{adam}})^{-1/2}$ for Adam-like diagonals, and an operator-norm proxy for low-rank preconditioners).
This quantity upper-bounds how much a preconditioner can magnify any input vector, and thus should track amplification of DP noise.
DP-AdamW begins with extremely large aggressiveness (orders of magnitude larger than PRISM) and only gradually decays,
which is consistent with Issue~III: early noisy second-moment estimates can have very small entries/eigenvalues, yielding very large inverse-square-root scaling.
PRISM stays in a much tighter range (tens to $\sim\!100$), consistent with explicitly enforcing a noise-calibrated floor (Eq.~\eqref{eq:noise_floor}),
which prevents the smallest eigenvalues from collapsing and keeps the effective scaling bounded as suggested by Eq.~\eqref{eq:precond_bound}.

\begin{figure}[t]
  \centering
\includegraphics{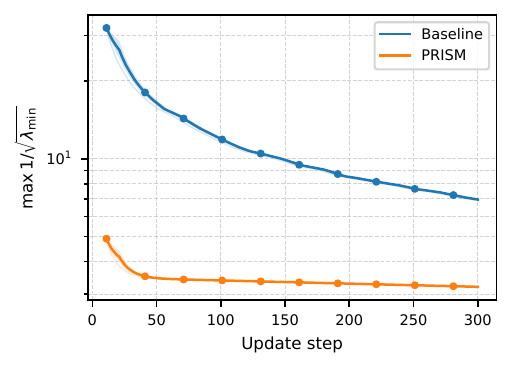}
  \caption{Low-rank numerics stress: $\max\|\mathbf{M}^{-1/2}\|_2$ (Gram proxy).}
  \label{fig:path3_gram_stress}
\end{figure}

\noindent\textbf{Analysis (ill-conditioning in the low-rank core).}
Issue~III also has a \emph{numerical} face: when preconditioning is implemented through low-rank structure, stability is controlled by the smallest eigenvalues
of the relevant Gram/second-moment objects.
Let $\mathbf{M}$ denote the (measured) Gram proxy in the low-rank core; then
$\|\mathbf{M}^{-1/2}\|_2 = 1/\sqrt{\lambda_{\min}(\mathbf{M})}$, so large values indicate severe ill-conditioning and stress both optimization and numerics.
\Cref{fig:path3_gram_stress} shows DP-AdamW exhibits substantially higher stress (large inverse-square-root operator norms) throughout training,
whereas PRISM remains in a low-stress regime.
This supports the theoretical motivation behind DP-aware floors and condition-number control: by preventing near-singular directions in the low-rank core,
PRISM reduces the regimes in which Eq.~\eqref{eq:precond_bound} would otherwise allow very large scaling (small effective $\lambda$).

\begin{figure}[t]
  \centering
\includegraphics{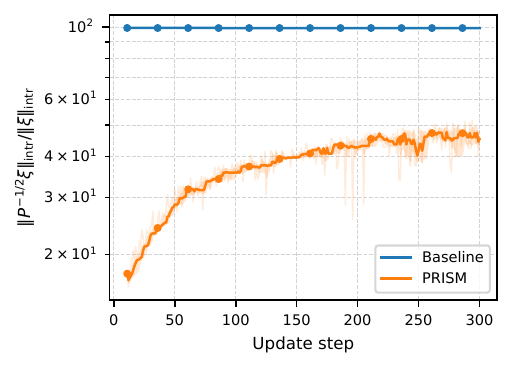}
  \caption{Amplification over training steps: $\|\mathsf{P}^{-1/2}\boldsymbol{\xi}\|/\|\boldsymbol{\xi}\|$.}
  \label{fig:path3_amp_over_time}
\end{figure}

\noindent\textbf{Analysis (direct evidence for Issue~III and PRISM mitigation).}
\Cref{fig:path3_amp_over_time} tracks the realized amplification
$\|\mathsf{P}_t^{-1/2}\boldsymbol{\xi}_{t,\mathrm{intr}}\|_F/\|\boldsymbol{\xi}_{t,\mathrm{intr}}\|_F$ during training.
Under adaptive preconditioning, this factor can be large when (i) the second-moment (or Gram) has tiny eigenvalues and/or (ii) the preconditioner becomes overly aggressive,
precisely the failure mode summarized by Issue~III.
Empirically, DP-AdamW sits near a large constant amplification ($\sim\!10^2$) over the entire run, which explains why its effective update noise
(after preconditioning) can remain large even when the raw intrinsic noise is comparatively small (\Cref{fig:path3_sigma_rawnoise}).
PRISM’s amplification is materially smaller (tens rather than hundreds) and evolves smoothly, consistent with a preconditioner whose smallest eigenvalues are protected by
noise-aware floors (Eq.~\eqref{eq:noise_floor}) and whose worst-case scaling is constrained in the sense of Eq.~\eqref{eq:precond_bound}.
Together with \Cref{fig:path3_precond_max,fig:path3_gram_stress}, this plot provides direct empirical support that PRISM mitigates the amplification aspect of Issue~III.

\begin{figure}[t]
  \centering
  \includegraphics{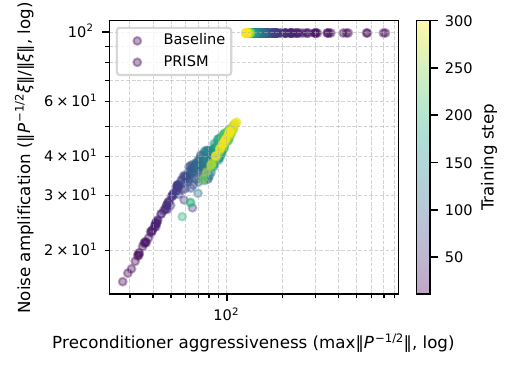}
  \caption{Amplification vs.\ aggressiveness (color = step).}
  \label{fig:path3_scatter_amp_aggr}
\end{figure}

\noindent\textbf{Analysis (mechanism: amplification is controlled by scaling).}
\Cref{fig:path3_scatter_amp_aggr} relates amplification to the max-scaling proxy.
% In general, for any linear preconditioner $\mathbf{A_{\ell}}$,
% $\|\mathbf{A_{\ell}}\boldsymbol{\xi}\|/\|\boldsymbol{\xi}\|$ concentrates between the singular values of $\mathbf{A_{\ell}}$;
In general, for any linear preconditioner $\mathbf{H_{\ell}}$,
$\|\mathbf{H_{\ell}}\boldsymbol{\xi}\|/\|\boldsymbol{\xi}\|$ concentrates between the singular values of $\mathbf{H_{\ell}}$;
Thus a max-scaling (operator-norm) proxy should strongly correlate with realized amplification.
PRISM exhibits this expected monotonic relationship: as the preconditioner becomes more aggressive over training (larger x),
the measured amplification (y) rises accordingly, and the color gradient shows this evolution over steps.
DP-AdamW, in contrast, occupies a regime with much larger aggressiveness yet saturates at a high amplification level,
suggesting the run spends most of its time near a hard constraint (e.g., clipping/conditioning caps) rather than smoothly trading off scaling.
This plot supports the interpretation that PRISM’s improvements are driven by controlling the preconditioner’s effective scaling,
exactly the control knob targeted by Eq.~\eqref{eq:noise_floor} and bounded by Eq.~\eqref{eq:precond_bound}.

\begin{figure}[t]
  \centering
  \includegraphics{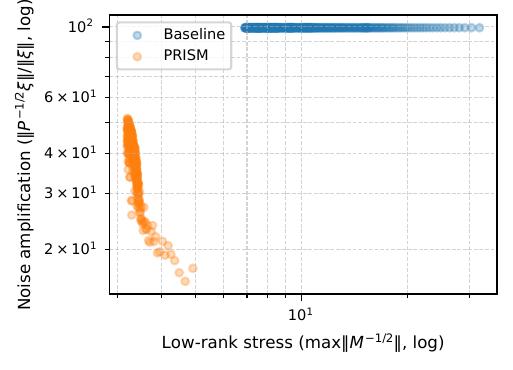}
  \caption{Amplification vs.\ low-rank stress (Gram proxy).}
  \label{fig:path3_scatter_amp_gram}
\end{figure}

\noindent\textbf{Analysis (mechanism: PRISM de-couples amplification from ill-conditioning).}
\Cref{fig:path3_scatter_amp_gram} links amplification to low-rank ill-conditioning via the Gram proxy stress $\|\mathbf{M}^{-1/2}\|_2$.
Without safeguards, increasing stress (smaller $\lambda_{\min}(\mathbf{M})$) would typically increase amplification because inverse-square-root operations
magnify components in near-null directions.
DP-AdamW concentrates in a high-stress regime, consistent with \Cref{fig:path3_gram_stress}, while maintaining a high amplification level.
PRISM stays in a low-stress regime and, importantly, shows that when stress grows, amplification does not explode; rather,
amplification can even decrease as safeguards activate (floors/clamps effectively reduce the preconditioner’s usable gain in ill-conditioned regimes).
This is the intended behavior of Issue~III mitigation: the algorithm should \emph{avoid} coupling worst-case scaling to unstable low-rank directions,
in line with the floor-based control in Eq.~\eqref{eq:noise_floor} and the bound in Eq.~\eqref{eq:precond_bound}.

\end{document}